\documentclass{article}
\usepackage{PRIMEarxiv}

\usepackage[T1]{fontenc}
\usepackage[utf8]{inputenc}
\usepackage{newtxtext}
\usepackage[dvipsnames,table]{xcolor}
\usepackage{graphicx}
\usepackage{booktabs}
\usepackage{multirow}
\usepackage{array}
\usepackage{tabularx}
\usepackage[most]{tcolorbox}
\usepackage{makecell}
\usepackage{pifont}
\usepackage{amsmath,amssymb,amsthm,mathtools}
\usepackage{newtxmath}
\usepackage{bm}
\usepackage{algorithm}
\usepackage{algorithmic}
\usepackage{caption}
\usepackage{subcaption}
\usepackage{placeins}
\usepackage[numbers,sort&compress]{natbib}
\usepackage{xurl}
\usepackage{hyperref}
\usepackage[nameinlink,capitalize,noabbrev]{cleveref}

\definecolor{TableTint}{HTML}{E2F1ED}
\definecolor{TableBest}{HTML}{FCE2C4}
\definecolor{TableSecond}{HTML}{E2E8FA}
\definecolor{TableThird}{HTML}{F0E2F2}
\definecolor{TableFlag}{HTML}{E4E9ED}
\definecolor{IconYes}{HTML}{149B3A}
\definecolor{IconNo}{HTML}{D62828}
\definecolor{PanelBG}{HTML}{F4F8F7}
\definecolor{PanelEdge}{HTML}{CFDFDA}

\newtcolorbox{figpanel}{%
  enhanced, breakable=false, rounded corners,
  colback=PanelBG, colframe=PanelEdge,
  boxrule=0.5pt, arc=5pt, outer arc=5pt,
  left=7pt, right=7pt, top=7pt, bottom=5pt, boxsep=0pt,
}

\newcommand{\cmark}{\textcolor{IconYes}{\ding{51}}}
\newcommand{\xmark}{\textcolor{IconNo}{\ding{55}}}
\newcommand{\bestbase}[1]{\cellcolor{TableBest}#1}
\newcommand{\secondbase}[1]{\cellcolor{TableSecond}#1}
\newcommand{\thirdbase}[1]{\cellcolor{TableThird}#1}
\newcommand{\ourscell}[1]{\cellcolor{TableTint}#1}
\newcommand{\flagcell}[1]{\cellcolor{TableFlag}#1}

\hypersetup{
  hidelinks,
  pdftitle={Beyond Pairwise Graphs in Science: Hypergraph Adaptive Wavelet Operators for Parametric PDEs},
  pdfauthor={Rajat Sarkar, Venkataramana Runkana, Souvik Chakraborty},
  pdfsubject={Scientific machine learning; neural operators; hypergraph wavelets}
}

\newtheorem{definition}{Definition}[section]
\newtheorem{proposition}[definition]{Proposition}
\newtheorem{remark}[definition]{Remark}
\newtheorem{lemma}[definition]{Lemma}
\newtheorem{theorem}[definition]{Theorem}
\providecommand{\shortcite}[1]{\citep{#1}}

\title{Beyond Pairwise Graphs in Science: Hypergraph Adaptive
Wavelet Operators for Parametric PDEs}
\author{
  \textbf{Rajat Sarkar} \\
  \mdseries Department of Applied Mechanics \\
  \mdseries Indian Institute of Technology Delhi \\
  \mdseries New Delhi, 110016, India \\
  \mdseries TCS Research \\
  \mdseries Tata Consultancy Services, India \\
  \mdseries \texttt{amz258208@am.iitd.ac.in} \\
  \mdseries \texttt{rajat.sarkar1@tcs.com} \\
  \And
  \textbf{Venkataramana Runkana} \\
  \mdseries TCS Research \\
  \mdseries Tata Consultancy Services \\
  \mdseries India \\
  \mdseries \texttt{venkat.runkana@tcs.com} \\
  \And
  \textbf{Souvik Chakraborty} \\
  \mdseries Department of Applied Mechanics \\
  \mdseries Yardi School of Artificial Intelligence (ScAI) \\
  \mdseries Indian Institute of Technology Delhi \\
  \mdseries New Delhi, 110016, India \\
  \mdseries \texttt{souvik@am.iitd.ac.in} \\
  \mdseries \texttt{https://www.csccm.in/} \\
}
\date{}

\begin{document}
\maketitle

\begin{abstract}
    Physical systems are often modeled by solution operators that map input
    fields, parameters, geometries, or past states to steady or future physical
    states. Learning these maps is difficult, especially for time-dependent
    systems that must assimilate history and remain stable under autoregressive
    rollout. Many neural operators work best on regular, structured grids,
    while realistic simulations often require unstructured meshes or point
    clouds to resolve complex geometries; in such settings, grid-centric
    representations can lose accuracy. Graph neural operators handle these
    domains through message passing or spectral graph filtering, but pairwise
    edges do not directly capture group-wise couplings among mesh cells, local
    neighborhoods, or conservation volumes. We introduce the Hypergraph Adaptive
    waveLet Operator (HALO), which lifts the domain to a hypergraph and learns
    in its spectral wavelet domain. HALO avoids explicit hypergraph-Laplacian
    eigendecomposition through Chebyshev polynomial wavelet filters, giving
    localized spectral kernels at linear sparse-matrix cost. Its trainable dyadic
    wavelet scales are regularized toward tight-frame coverage, allowing the
    frequency response to adapt to each PDE while encouraging stable
    multi-scale spectral coverage. Across 2D and 3D benchmarks on structured
    and unstructured discretizations, HALO achieves best or near-best accuracy
    among frequency-, transformer-, DeepONet-, state-space-, and graph-based
    baselines and sustains stable multi-step rollouts. The same model scales to
    industrial aerodynamic geometries: on meshes of a few hundred thousand
    points it is on par with, or better than, the strongest fixed-discretization
    transformers, while remaining resolution-equivariant.
\end{abstract}


\section{Introduction}

Scientific machine learning increasingly centers on \emph{operator learning}:
learning maps between infinite-dimensional function spaces that act as the
\emph{solution operators} of PDEs~\cite{lu2021deeponet,kovachki2023neural}.
Trained purely from data and requiring no explicit knowledge of the governing
equations, such operators offer fast surrogates that generalize across problem
instances and have shown strong results in weather
forecasting~\cite{pathak2022fourcastnet}, biomedical surrogate
modeling~\cite{yin2022simulating}, and foundation models for
PDEs~\cite{herde2024poseidon,hao2024dpot}. Yet building operators that stay
accurate on real-world domains remains challenging. Spectral neural operators
built on Fourier~\cite{li2021fourier} and wavelet~\cite{tripura2023wavelet} bases
are highly effective and resolution-equivariant on regular grids, and
geometry-aware variants such as Geo-FNO~\cite{li2022geofno} and
GINO~\cite{li2023gino} extend them beyond rectangular domains, yet their accuracy
still drops on complex, unstructured geometries. The deeper difficulty shared
across these families appears when the discretization itself varies: in many
applications, different samples of one task arrive on different meshes with
different node counts, which operators tied to a fixed grid cannot readily
absorb.

To address this variability, graph-based neural operators replace grid
convolution with message passing or graph kernel integration, processing
unstructured point sets and meshes while preserving permutation
equivariance~\cite{li2020neural,brandstetter2022message,pfaff2021learning}.
However, an ordinary graph encodes the domain through edges that each relate only
two vertices. This pairwise bias is limiting because many physical and numerical interactions
couple a whole group of nodes at once rather than isolated pairs: finite elements
aggregate multiple degrees of freedom~\cite{brenner2008mathematical}, and
finite-volume schemes update cell-wise conserved quantities through neighboring
control volumes~\cite{leveque2002finite}. Collapsing such multi-node relations into
independent edges fragments the interaction, forcing the model to reconstruct
higher-order correlations indirectly through deeper propagation.

Hypergraphs offer a direct representation for set-valued relations: a single
hyperedge can connect an arbitrary set of vertices, encoding a mesh element, a
$k$-ring neighborhood, or a local physical interaction
group~\cite{zhou2006learning,battiston2020networks,bodnar2021weisfeiler}. Both
spatial message-passing and spectral hypergraph learning have improved feature
learning in vision and 3D tasks~\cite{jiang2019dynamic,feng2019hypergraph,nong2021hypergraph},
yet their use as a neural-operator backbone for PDEs remains largely unexplored.
This motivates a natural question: can the same higher-order inductive bias help
scientific operators reason over local groups of nodes, rather than asking
pairwise edges alone to reconstruct their joint effect? The remaining challenge
is to turn this idea into a multi-scale, stable, and computationally practical
neural operator.

To address this challenge, we introduce the \emph{Hypergraph Adaptive waveLet
Operator} (HALO), which encodes group-wise geometry in a hypergraph and realizes
kernel integration through spectral hypergraph wavelets. Rather than
diagonalizing the hypergraph Laplacian, it evaluates each wavelet scale with
Chebyshev polynomial filters~\cite{hammond2011wavelets,defferrard2016convolutional,xu2019graph},
yielding spatially localized receptive fields with sparse-matrix computation.
Its scales are learned from a dyadic initialization and regularized toward
tight-frame coverage, letting the spectral response adapt to each PDE while
discouraging gaps or excessive amplification across scales. The same block
supports both steady-state operators and autoregressive rollouts.

Our main contributions are:
\begin{itemize}
\itemsep2pt
\item \textbf{Higher-order operator.} We propose \textbf{HALO}, a hypergraph
wavelet neural operator that learns higher-order, group-wise interactions beyond
the pairwise couplings of graph operators.
\item \textbf{Efficient adaptive spectrum.} A Chebyshev-accelerated wavelet block
with trainable scales and tight-frame regularization gives localized, multi-scale
kernel integration without eigendecomposition.
\item \textbf{Mesh-flexible, resolution-equivariant.} HALO keeps the
grid-agnostic flexibility of graphs while remaining resolution-equivariant:
it matches the best regular-grid operators on structured domains and surpasses
all baselines on unstructured ones.
\end{itemize}
\section{Problem Formulation and Background}

We consider learning a solution operator
$\mathcal{N}: \mathcal{A} \rightarrow \mathcal{U}$
that maps an input function $a \in \mathcal{A}$ to an output
$u \in \mathcal{U}$, both defined on a bounded physical domain
$D \subset \mathbb{R}^{d}$.
Given $N$ paired observations $\{(a^{(i)}, u^{(i)})\}_{i=1}^{N}$,
the goal is to learn $\mathcal{N}$ from finite samples so that it approximates
the underlying continuous solution map and can be evaluated on compatible
discretizations of $D$.

\subsubsection*{Kernel-Based Neural Operator}
Motivated by the Green's function representation of PDEs~\cite{li2020neural},
a kernel-based neural operator parameterizes $\mathcal{N}$ as $L$ iterative update
layers.
At each layer $\ell$, the hidden representation
$v_\ell : D \rightarrow \mathbb{R}^{d_h}$ is refined as
\begin{equation}
    v_{\ell+1}(x) = \sigma\!\left(
        W_\ell v_\ell(x)
        + \bigl(\mathcal{K}_\phi^{(\ell)}v_\ell\bigr)(x)
    \right),
    \label{eq:neural_op_layer}
\end{equation}
where $W_\ell \in \mathbb{R}^{d_h \times d_h}$ is a pointwise linear map,
$\sigma$ is a nonlinear activation, and
\begin{equation}
    \bigl(\mathcal{K}_\phi^{(\ell)}v\bigr)(x)
    = \int_D \kappa_\phi^{(\ell)}\!\left(x,y,a(x),a(y)\right)v(y)\,dy
    \label{eq:kernel_integration}
\end{equation}
is the \emph{kernel integration operator}~\cite{li2021fourier},
responsible for capturing long-range function dependencies.
Designing an expressive yet computationally tractable realization of
this operator over unstructured domains with higher-order interactions
is the central challenge we address.

\subsection{Hypergraph Wavelet Theory}

A hypergraph $\mathcal{G}=(\mathcal{V},\mathcal{E},W_e)$ generalizes a
graph by allowing each hyperedge $e\in\mathcal{E}$ to connect an arbitrary
subset of vertices. Let $H\in\{0,1\}^{n\times|\mathcal{E}|}$ denote its
incidence matrix, and let $W_e$, $D_v$, and $D_e$ denote the diagonal
hyperedge-weight, vertex-degree, and hyperedge-degree matrices, respectively.
Following Zhou et al.~\shortcite{zhou2006learning}, the normalized hypergraph
Laplacian is
\begin{equation}
    \Delta
    = I-D_v^{-1/2}HW_eD_e^{-1}H^\top D_v^{-1/2}.
    \label{eq:hg_laplacian}
\end{equation}
Since $\Delta$ is symmetric positive semidefinite, it admits the
eigendecomposition $\Delta=U\Lambda U^\top$. For a node signal
$x\in\mathbb{R}^n$, spectral filtering with a band-pass kernel $g$ (whose
Meyer-like form is given later in Eq.~\eqref{eq:meyer_kernel}) is defined as
$U g(\Lambda)U^\top x$.

Following spectral graph wavelet theory
\cite{hammond2011wavelets,nong2021hypergraph}, the wavelet operator at scale
$s$ is $\mathcal{W}_s=U g(s\Lambda)U^\top$, and the wavelet localized at vertex
$i$ is
\begin{equation}
    \psi_{s,i}=\mathcal{W}_s\delta_i=U g(s\Lambda)U^\top\delta_i.
    \label{eq:wavelet_atom}
\end{equation}
Applying $\mathcal{W}_s$ to a node signal $x$ gives the wavelet transform
$\mathcal{W}_s x$, whose $i$-th entry
$\Psi_{s,i}=(\mathcal{W}_s x)_i=\langle\psi_{s,i},x\rangle$ is the wavelet
coefficient at vertex $i$. To avoid explicit eigendecomposition, the spectral
kernel can be approximated by an order-$M$ Chebyshev
expansion~\cite{defferrard2016convolutional}:
\begin{equation}
    \Psi_{s,i}\approx
    \left(
        \frac{1}{2}c_{s,0}x+
        \sum_{m=1}^{M}c_{s,m}T_m(\widetilde{\Delta})x
    \right)_i,
    \widetilde{\Delta}=\frac{2\Delta}{\lambda_{\max}}-I,
    \label{eq:cheby_approx}
\end{equation}
where the coefficients $c_{s,m}$ approximate the scaled kernel
$g(s\lambda)$ on $[0,\lambda_{\max}]$. The expansion can be evaluated
through the Chebyshev recurrence using $M$ sparse matrix--vector products,
with complexity
$\mathcal{O}\!\left(M\,\mathrm{nnz}(\Delta)\right)$. Moreover, because
$T_m(\widetilde{\Delta})$ is a degree-$m$ polynomial in the Laplacian, the
order-$M$ approximation is supported within the $M$-hop neighborhood of
its center. Thus, $M$ controls spatial localization, while $s$ controls
spectral selectivity.

\section{Hypergraph Adaptive Wavelet Operator (HALO)}

HALO realizes the kernel-integration layer of
Eq.~\eqref{eq:neural_op_layer} as a multi-scale hypergraph spectral operator.
Given a field sampled at nodes $\{x_i\}$, it proceeds in three stages
(Figure~\ref{fig:halo_architecture}): raw inputs are lifted into a latent
representation, refined by $L$ \emph{Adaptive Wavelet Blocks} (AWBs) applying
trainable multi-scale filters over a fixed hypergraph Laplacian $\Delta$, and
decoded to the target space. The coordinates $\{x_i\}$ serve a dual purpose:
enriching node features and defining the hypergraph geometry encoded by $\Delta$.

\begin{figure}[H]
    \centering
    \includegraphics[width=\linewidth]{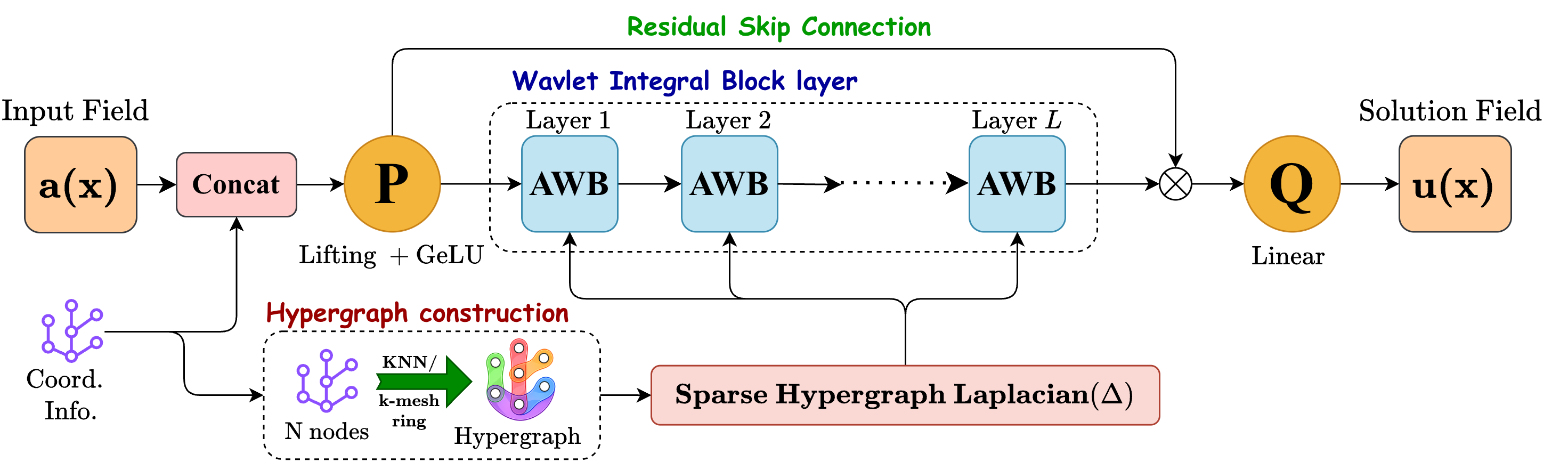}
    \caption{HALO architecture: coordinate-aware input uplift, a stack of $L$
    Adaptive Wavelet Blocks operating on the cached sparse hypergraph
    Laplacian, output projection, and a global residual connection.}
    \label{fig:halo_architecture}
\end{figure}

\subsection{Input Preparation}
HALO reads a single node-wise observation, and what that observation contains is
the only thing that separates the two classes of problem it addresses. A
steady-state task supplies one field over the domain---an input field $a(x)$, or
an initial condition $a(x,t_0)$---from which the operator recovers the target
field in a single application. An unsteady task instead supplies the
$T_{\text{in}}$ most recently observed snapshots
$\{a(x,t_1),\,a(x,t_2),\,\dots,\,a(x,t_{T_{\text{in}}})\}$, from which the
operator predicts the state at the next time level and is then applied
repeatedly to roll out to the target horizon. Let $c$ be the number of physical channels
carried by one snapshot. Both cases then assemble the same object at each node
$i$: the channel-wise stack of the supplied snapshots, ordered from oldest to
most recent,
\begin{equation}
    a(x_i)
    = \bigl[\,a(x_i,t_1);\;\dots;\;a(x_i,t_{T_{\text{in}}})\,\bigr]
      \in \mathbb{R}^{d_a},
    \qquad
    d_a = T_{\text{in}}\,c ,
    \label{eq:input_window}
\end{equation}
the steady case being $T_{\text{in}} = 1$. Concatenating this observation with
the spatial coordinate
$x_i \in \mathbb{R}^{d}$ and with any conditioning channels
$z_i \in \mathbb{R}^{d_{\text{cond}}}$ the task supplies---node or boundary
tags, global physical parameters---forms the node feature vector
$\tilde{x}_i = [a(x_i);\, x_i;\, z_i] \in \mathbb{R}^{d_{\text{in}}}$, with
$d_{\text{in}} = T_{\text{in}}\,c + d + d_{\text{cond}}$ and
$d_{\text{cond}} = 0$ when the task provides none.
A single linear projection with nonlinearity then \emph{uplifts} this
compact observation ($d_{\text{in}} \ll d_h$) to a richer $d_h$-dimensional
latent representation,
$\mathbf{h}^{(0)}_i = \mathrm{GELU}(\tilde{x}_i\,W_{\text{in}} + b_{\text{in}})
\in \mathbb{R}^{d_h}$, with $W_{\text{in}} \in \mathbb{R}^{d_{\text{in}} \times d_h}$
and $b_{\text{in}} \in \mathbb{R}^{d_h}$.

\subsection{Hypergraph Construction}

We build the hypergraph once from the node coordinates $\{x_i\}$ and reuse it
throughout training: its sparse Laplacian $\Delta$, with $\lambda_{\max}$
obtained by the power method, is cached and shared across all AWBs. The
construction follows the data, mapping every node to exactly one anchoring
hyperedge that feeds the same downstream operator.

\begin{figure}[htbp]
    \centering
    \begin{tabular}{@{}cc@{}}
        \includegraphics[width=0.35\linewidth]{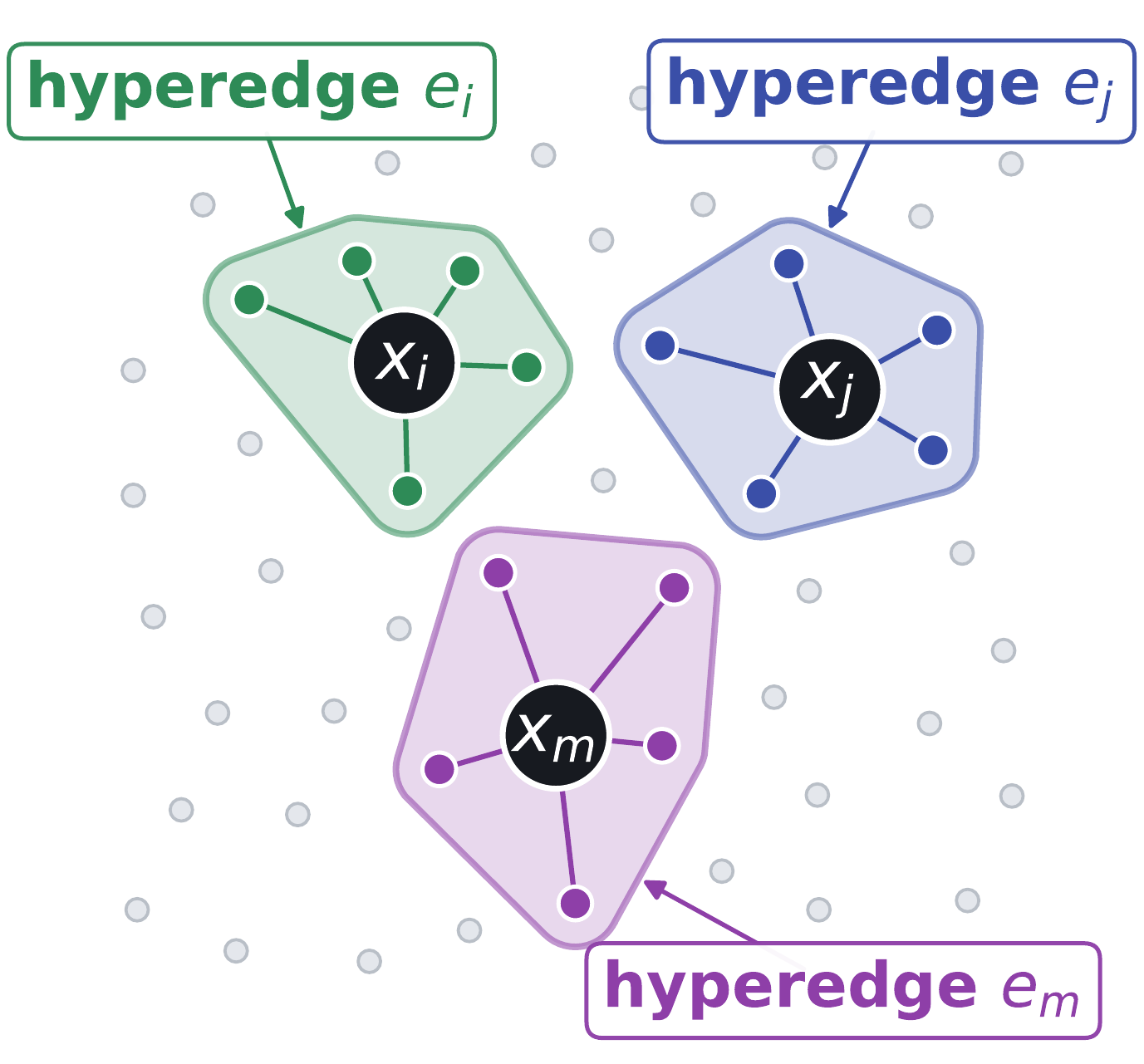} &
        \includegraphics[width=0.35\linewidth]{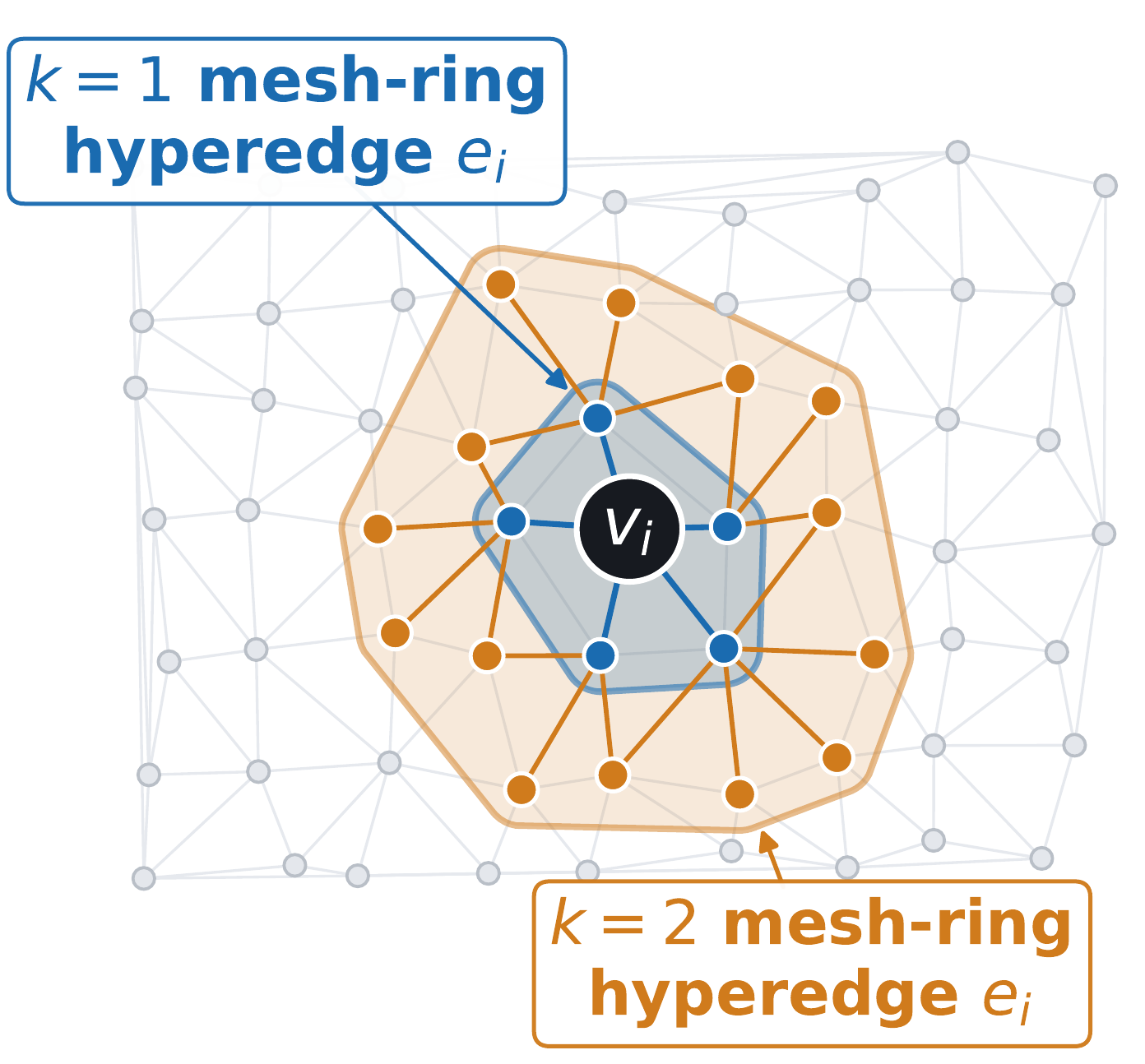} \\[2pt]
        (a) $k$-NN point-cloud hyperedges & (b) mesh-ring hyperedges
    \end{tabular}
    \caption{Dataset-aware hypergraph construction. Point clouds use Euclidean
    $k$-nearest-neighbor groups, whereas meshes use connected $k$-ring patches
    that respect the supplied topology.}
    \label{fig:knn_hypergraph}
\end{figure}

\subsubsection*{Local Hyperedges}
Each hyperedge is built around an \emph{anchor} node $i$ by collecting its
neighbors, $e_i=\{i\}\cup\mathcal{N}_k(i)$, with the neighbor rule set by the
spatial information the dataset provides. For a point cloud we take the $k$
nearest neighbors of $i$ by Euclidean distance
(Figure~\ref{fig:knn_hypergraph}(a)); when mesh connectivity is available we take
the $k$-ring, the nodes within $k$ hops of $i$ on the mesh graph (two nodes
joined when they share a cell or a listed edge), forming a connected patch
(Figure~\ref{fig:knn_hypergraph}(b)) that never links nodes which are spatially
close yet geodesically distant. In both cases the incidence matrix stays
binary, and each hyperedge is weighted by how tightly its members cluster around
the anchor $x_e:=x_i$: with $\sigma_e$ the mean anchor-to-member distance,
$(W_e)_{e_ie_i}$ averages the Gaussian proximity
$\exp(-\|x_v-x_e\|^2/\sigma_e^2)$ over $v\in e_i$, so compact groups weigh near
one and diffuse ones less. This yields $n$ overlapping hyperedges that encode
local higher-order structure in the topology.

\subsection{Adaptive Wavelet Block}

The AWB, illustrated in Figure~\ref{fig:awb}, is the
core computational module of HALO. It maps the node feature matrix
$\mathbf{h}^{(\ell)} \in \mathbb{R}^{n \times d_h}$ to
$\mathbf{h}^{(\ell+1)} \in \mathbb{R}^{n \times d_h}$ using the fixed sparse
Laplacian $\Delta$. The following components specify its
Chebyshev hypergraph-wavelet realization, beginning with the spectral kernel
and the scales at which it is applied.

\begin{figure}[htbp]
    \centering
    \includegraphics[width=\linewidth]{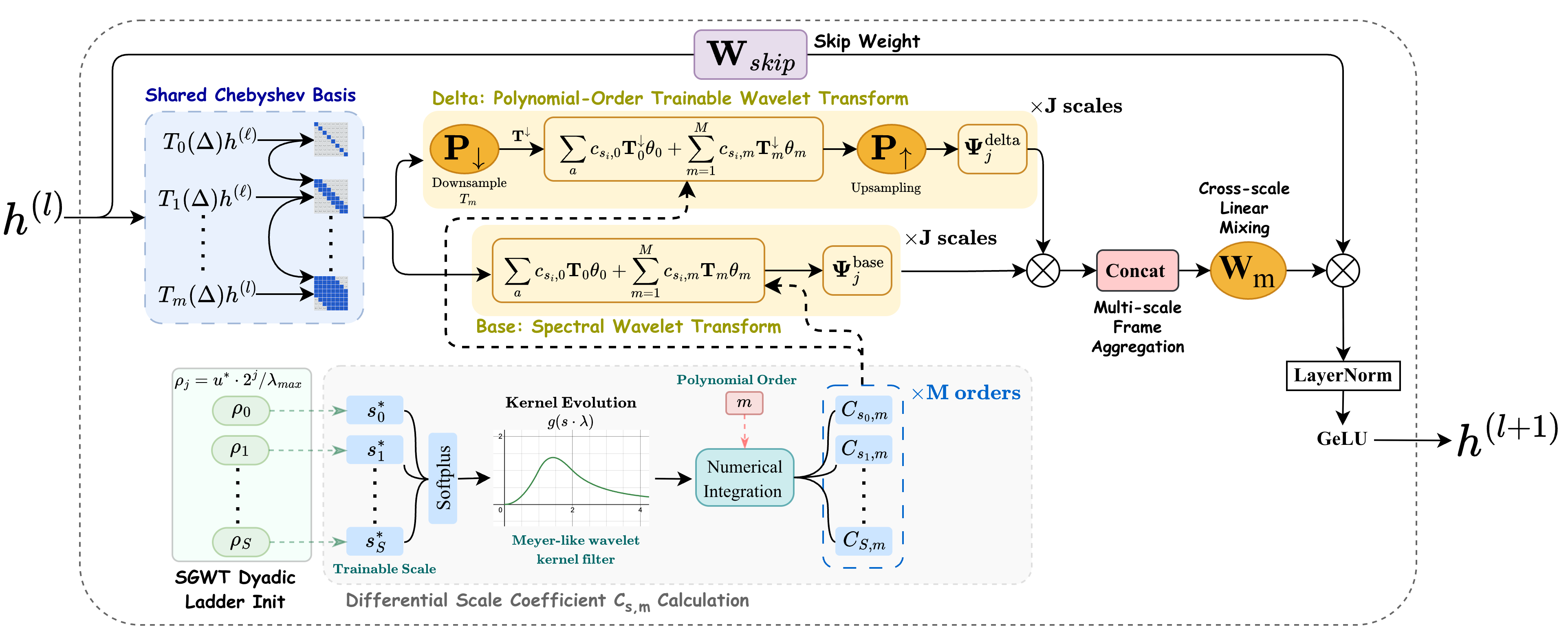}%
    \caption{Adaptive Wavelet Block. A shared Chebyshev basis supports every
    learned wavelet scale; analytic Base and low-rank Delta responses are mixed
    across scales and combined with the skip path before normalization and
    activation.}
    \label{fig:awb}
\end{figure}

\subsubsection*{Dyadic Scale Selection}
The spectral kernel used in HALO is a Meyer-like band-pass
function~\cite{hammond2011wavelets}:
\begin{equation}
    g(x) = \begin{cases}
        x^2 & 0 \leq x < 1, \\
        -5 + 11x - 6x^2 + x^3 & 1 \leq x \leq 2, \\
        4\,/\,x^2 & x > 2,
    \end{cases}
    \label{eq:meyer_kernel}
\end{equation}
which is zero at the origin (vanishing at DC), smooth, and decays for
large $x$.
Setting $g'(x^*) = 0$ in $[1,2]$ gives the unique maximizer
$x^* = 2 - 1/\sqrt{3} \approx 1.423$, the spectral location of peak wavelet
response.
We initialize the $J$ wavelet scales following the
\emph{SGWT dyadic ladder}~\cite{hammond2011wavelets},
$s_j^{\mathrm{init}} = x^* 2^{j}/\lambda_{\max}$ for $j = 0,1,\ldots,J-1$,
so that the $j$-th wavelet $\psi_{s_j, i}$ peaks at eigenvalue
$\lambda_j^* = x^*/s_j = \lambda_{\max}/2^j$, tiling the spectrum in
logarithmically spaced bands.

\subsubsection*{Trainable Scale Parameterization}
While the dyadic initialization provides a principled spectral tiling,
different PDE operators concentrate energy at different frequency
scales.
We therefore make all scales \emph{fully trainable} by
reparameterizing each as $s_j = \mathrm{softplus}(\rho_j) = \log(1 + e^{\rho_j})$,
with $\rho_j = \mathrm{softplus}^{-1}(s_j^{\mathrm{init}})$ as the
initial value.
The softplus map guarantees $s_j > 0$ for all $\rho_j \in \mathbb{R}$,
eliminating the need for constrained optimization.
The parameters $\{\rho_j\}$ are optimized jointly with all network
weights, letting each wavelet level track the dominant frequency
content of the target operator; the resulting per-benchmark scales are
reported in the Appendix.

\subsubsection*{Adaptive Chebyshev Coefficients}
At layer $\ell$, each learned scale $s_j^{(\ell)}$ determines an
order-$M$ Chebyshev approximation of the scaled spectral kernel
$g(s_j^{(\ell)}\lambda)$. The corresponding coefficients are
\begin{equation}
\begin{aligned}
c_{s_j^{(\ell)},m}
&=
\frac{2}{\pi}
\int_{0}^{\pi}
\cos(m\theta)\,
g\!\left(
s_j^{(\ell)}\lambda(\theta)
\right)d\theta,
\; m=0,\ldots,M,\\
\lambda(\theta)
&=
\frac{\lambda_{\max}}{2}
\left(\cos\theta+1\right).
\end{aligned}
\label{eq:adaptive_cheby_coeff}
\end{equation}
These coefficients translate the learned scale into the scalar weights
applied to the shared Chebyshev basis, thereby adapting the spectral
response without recomputing the polynomial features.

\subsubsection*{Differentiable Quadrature}
We evaluate Eq.~\eqref{eq:adaptive_cheby_coeff} using a $Q$-point
Chebyshev--Gauss quadrature rule. Let
\begin{equation}
    \theta_q=\frac{\left(q-\tfrac{1}{2}\right)\pi}{Q},
    \qquad q=1,\ldots,Q.
    \label{eq:quadrature_nodes}
\end{equation}
The coefficients are then approximated as
\begin{equation}
    c_{s_j^{(\ell)},m}
    \approx
    \frac{2}{Q}
    \sum_{q=1}^{Q}
    \cos(m\theta_q)\,
    g\!\left(
        s_j^{(\ell)}
        \frac{\lambda_{\max}}{2}
        [\cos\theta_q+1]
    \right).
    \label{eq:quadrature_coeff}
\end{equation}
The nodes $\{\theta_q\}_{q=1}^{Q}$ and cosine factors
$\{\cos(m\theta_q)\}_{m,q}$ are precomputed and shared across layers
and scales. Only the evaluations of
$g(s_j^{(\ell)}\lambda)$ vary with the learned scales, making the
coefficient computation inexpensive and differentiable with respect to
$\rho_j^{(\ell)}$.
Appendix~\ref{sec:approx_results} gives the supporting quadrature-exactness
result and the Chebyshev truncation-error guarantee.

\subsubsection*{Shared Basis}
At layer $\ell$, the shared tensors represent
\begin{equation*}
    \mathbf{T}_m
    =T_m(\widetilde{\Delta})\,\mathbf{h}^{(\ell)},
    \qquad m=0,\ldots,M.
\end{equation*}
Thus, $\mathbf{T}_m$ applies the degree-$m$ Chebyshev polynomial to the hidden
features. The tensors are computed \emph{once} for all $J$ scales via the
three-term recurrence:
\begin{equation}
    \mathbf{T}_0 = \mathbf{h}^{(\ell)},\quad
    \mathbf{T}_1 = \widetilde{\Delta}\,\mathbf{h}^{(\ell)},\quad
    \mathbf{T}_m = 2\widetilde{\Delta}\,\mathbf{T}_{m-1} - \mathbf{T}_{m-2},
    \label{eq:cheby_recurrence}
\end{equation}
where each step is a single sparse product $\widetilde{\Delta}\,\mathbf{T}_{m-1}$.
Crucially, these $M{+}1$ tensors are \emph{shared} across all $J$ scales; only
the scalar coefficients $c_{s_j^{(\ell)},m}$ differ per scale.

\begin{remark}[Per-layer cost of the shared Chebyshev AWB]
\label{rem:cost}
Sharing the polynomial basis reduces the AWB cost to $M$ sparse products,
$\mathcal{O}(M\,\mathrm{nnz}(\widetilde{\Delta})\,d_h)$, which is
\emph{independent of the number of scales $J$} (versus $JM$ for naive filtering).
The memory for the shared basis is $\mathcal{O}(Mnd_h)$. Because our $k$-NN and
$k$-ring constructions yield $\mathrm{nnz}(\widetilde{\Delta})=\mathcal{O}(nk^2)$,
the layer scales linearly with the node count $n$.
\end{remark}

In the implementation, the repeated products in
Eq.~\eqref{eq:cheby_recurrence} use a width-aware Sliced-ELLPACK backend
(SELL-C-1): the rescaled Laplacian is converted once, cached, and applied by a
custom Triton SpMM kernel, with CSR retained for narrow operands and CPU
execution. This layout matches the near-uniform row lengths induced by our
$k$-NN and $k$-ring hypergraphs, enabling coalesced structure reads with little
padding; symmetry of $\widetilde{\Delta}$ also lets the backward pass reuse the
same stored operator. On the representative Airfoil product, SELL-C-1 is
$2.0\times$ faster than cuSPARSE CSR, and the backend reduces the aggregate
one-epoch time across the benchmark suite by $1.38\times$ without changing the
model outputs to the reported precision. Appendix~\ref{app:sell_backend}
documents the layout and controlled timing study. On the two 3D geometry
benchmarks, where every sample carries its own mesh and those meshes are the
largest in the suite, $\widetilde{\Delta}$ is never assembled but applied
through its incidence factors, which keeps the per-geometry operator set
resident at the same result (Appendix~\ref{app:factored_laplacian}).

\begin{proposition}[$M$-hop localization of the AWB]
\label{prop:local}
Let $d_{\mathcal{G}}(i,v)$ denote the hop distance on the induced
adjacency graph, where two vertices are adjacent if they share a
hyperedge. For any scale $s_j^{(\ell)}>0$, let $\widehat{\mathcal{W}}_{s_j^{(\ell)}}^{(M)}$
denote the order-$M$ Chebyshev approximation of the wavelet operator
$\mathcal{W}_{s_j^{(\ell)}}$,
\begin{equation}
\widehat{\mathcal{W}}_{s_j^{(\ell)}}^{(M)}
=
\frac{1}{2}c_{s_j^{(\ell)},0}I
+
\sum_{m=1}^{M}
c_{s_j^{(\ell)},m}
T_m(\widetilde{\Delta}).
\label{eq:truncated_wavelet_operator}
\end{equation}
The wavelet localized at vertex $i$ satisfies
\begin{equation}
\left(
\widehat{\mathcal{W}}_{s_j^{(\ell)}}^{(M)}
\delta_i
\right)_v
=0
\qquad
\text{whenever}\qquad
d_{\mathcal{G}}(i,v)>M.
\label{eq:m_hop_localization}
\end{equation}
Thus, the order-$M$ wavelet operator is supported within the $M$-hop
neighborhood of its center, independently of the learned scale
$s_j^{(\ell)}$ and the coefficients
$\{c_{s_j^{(\ell)},m}\}_{m=0}^{M}$.
\end{proposition}
\begin{proof}
The off-diagonal part of the normalized Laplacian arises solely from its second
term, whose $(v,i)$ entry, for $v\neq i$, is the negative of a nonnegative sum,
\begin{equation}
   \Delta_{vi}
   =-\frac{1}{\sqrt{d_v d_i}}\sum_{e} H_{ve}\,
     \frac{(W_e)_{ee}}{(D_e)_{ee}}\,H_{ie}.
   \label{eq:app_p1_offdiag}
\end{equation}
Since every summand is nonnegative, no cancellation is possible, so this entry
is nonzero exactly when some hyperedge contains both $i$ and $v$; that is,
\begin{equation}
   \Delta_{vi}\neq 0 \iff d_{\mathcal{G}}(i,v)=1 ,
   \label{eq:app_p1_adj}
\end{equation}
and the affine rescaling
$\widetilde{\Delta}=(2/\lambda_{\max})\Delta-I$ inherits this off-diagonal
pattern. Hence one factor of $\widetilde{\Delta}$ propagates a signal by at
most one hop. We claim, by induction on $m$, that
\begin{equation}
   (\widetilde{\Delta}^{\,m})_{vi}=0
   \qquad\text{whenever}\qquad d_{\mathcal{G}}(i,v)>m .
   \label{eq:app_p1_power}
\end{equation}
The base case $\widetilde{\Delta}^{0}=I$ is immediate. Assuming the claim for
$m-1$ and expanding the matrix power,
\begin{equation}
   (\widetilde{\Delta}^{\,m})_{vi}
   =\sum_{w}(\widetilde{\Delta}^{\,m-1})_{vw}\,
     (\widetilde{\Delta})_{wi},
   \label{eq:app_p1_expand}
\end{equation}
a summand is nonzero only if $d_{\mathcal{G}}(v,w)\le m-1$ and
$d_{\mathcal{G}}(w,i)\le 1$, whence $d_{\mathcal{G}}(i,v)\le m$ by the
triangle inequality, proving Eq.~\eqref{eq:app_p1_power}. Finally, $T_m$ has
degree $m$, so by Eq.~\eqref{eq:app_p1_power} the term
$T_m(\widetilde{\Delta})$ is supported within $m$ hops of $i$. The wavelet in
Eq.~\eqref{eq:truncated_wavelet_operator} superposes only terms with $m\le M$,
and the coefficients $c_{s,m}$ reweight them without enlarging their support.
Therefore,
$\bigl(\widehat{\mathcal{W}}_{s}^{(M)}\delta_i\bigr)_v=0$ whenever
$d_{\mathcal{G}}(i,v)>M$.
\end{proof}

\subsubsection*{Base and Delta Wavelet Transforms}
As shown in Figure~\ref{fig:awb}, each scale $s_j^{(\ell)}$
produces two complementary outputs that together constitute the
per-scale wavelet response.

\noindent\textit{Base: Spectral Wavelet Transform.}
The shared Chebyshev basis is directly combined with the
scale-specific coefficients, realizing the fixed spectral wavelet
response of Eq.~\eqref{eq:cheby_approx} at scale $s_j^{(\ell)}$:
\begin{equation}
    \mathbf{\Psi}^{\mathrm{base},(\ell)}_{j}
    = \frac{1}{2}\,c_{s_j^{(\ell)},0}\,\mathbf{T}_0
    + \sum_{m=1}^{M} c_{s_j^{(\ell)},m}\,\mathbf{T}_m
    \in \mathbb{R}^{n \times d_h}.
    \label{eq:base_path}
\end{equation}

\noindent\textit{Delta: Polynomial-Order Trainable Wavelet Transform.}
To inject additional learnable spectral interactions at reduced cost,
the Delta path passes the Chebyshev basis through a
\emph{down-projection} $\mathbf{P}_{\downarrow}^{(\ell)}
\in \mathbb{R}^{d_h \times d_c}$, where $d_c \ll d_h$, applies per-order
trainable kernels $\Theta_{j,m}^{(\ell)} \in
\mathbb{R}^{d_c \times d_c}$ initialized at zero, and restores the
feature dimension through an independent \emph{up-projection}
$\mathbf{P}_{\uparrow}^{(\ell)} \in \mathbb{R}^{d_c \times d_h}$:
\begin{equation}
\begin{split}
    \mathbf{\Psi}^{\mathrm{delta},(\ell)}_{j}
    &= \tfrac{1}{2}\,c_{s_j^{(\ell)},0}\,
       \bigl(\mathbf{T}_0\mathbf{P}_{\downarrow}^{(\ell)}\bigr)
       \Theta_{j,0}^{(\ell)}\mathbf{P}_{\uparrow}^{(\ell)}\\
    &\quad + \sum_{m=1}^{M} c_{s_j^{(\ell)},m}\,
       \bigl(\mathbf{T}_m\mathbf{P}_{\downarrow}^{(\ell)}\bigr)
       \Theta_{j,m}^{(\ell)}\mathbf{P}_{\uparrow}^{(\ell)}.
\end{split}
    \label{eq:delta_path}
\end{equation}
Zero initialization of $\{\Theta_{j,m}^{(\ell)}\}$ ensures the Delta
transform is silent at startup, letting the Base transform dominate
early training while the Delta path gradually learns corrective
higher-order spectral interactions.
The combined per-scale wavelet output is
$\mathbf{\Psi}_j^{(\ell)} = \mathbf{\Psi}^{\mathrm{base},(\ell)}_{j} +
\mathbf{\Psi}^{\mathrm{delta},(\ell)}_{j}
\in \mathbb{R}^{n \times d_h}$.

\subsubsection*{Cross-Scale Mixing and Skip Connection}
The $J$ per-scale outputs are concatenated along the feature dimension and
passed through a learned mixing layer
$\mathbf{W}_{\mathrm{mix}}^{(\ell)} \in \mathbb{R}^{Jd_h \times d_h}$, giving
the discrete approximation
$\left.\mathcal{K}_{\phi}^{(\ell)}v_\ell\right|_{\mathcal{V}}
\approx \mathbf{V}^{(\ell)}$ on the sampled nodes:
\begin{equation}
    \mathbf{V}^{(\ell)}
    = \bigl[
        \mathbf{\Psi}_{0}^{(\ell)} \;\|\; \cdots \;\|\; \mathbf{\Psi}_{J-1}^{(\ell)}
      \bigr]\mathbf{W}_{\mathrm{mix}}^{(\ell)}
    \in \mathbb{R}^{n \times d_h}.
    \label{eq:scale_mix}
\end{equation}
A skip connection $\mathbf{h}^{(\ell)}\,W_{\text{skip}}$ (with
$W_{\text{skip}} \in \mathbb{R}^{d_h \times d_h}$) preserves low-frequency
information bypassing the wavelet branches, and the AWB update applies Layer
Normalization ($\mathrm{LN}$), which stabilizes training during long
autoregressive rollouts:
\begin{equation}
    \mathbf{h}^{(\ell+1)}
    = \mathrm{GELU}\!\left(
        \mathrm{LN}\!\left(
            \mathbf{V}^{(\ell)} + \mathbf{h}^{(\ell)}\,W_{\text{skip}}
        \right)
      \right).
    \label{eq:wib_output}
\end{equation}

\subsection{Output Projection}

After $L$ AWB layers, the final hidden state $\mathbf{h}^{(L)}$ is summed
with the preserved input uplift $\mathbf{h}^{(0)}$ via a \emph{global input
residual} connection---a direct path across all $L$ AWB layers---and a
linear projection decodes it to the target output dimension:
\begin{equation}
    \hat{u}_i
    = \left(\mathbf{h}^{(L)}_i + \mathbf{h}^{(0)}_i\right) W_{\text{out}},
    \label{eq:output_proj}
\end{equation}
with $W_{\text{out}} \in \mathbb{R}^{d_h \times d_{\text{out}}}$ and
\emph{no} output nonlinearity, appropriate for regression tasks.
The same decoder serves both regimes: for steady-state problems
Eq.~\eqref{eq:output_proj} yields the target field in a single forward pass,
whereas for time-dependent problems HALO is applied autoregressively over a
window of fixed width. Let
$X^{(r)} = \bigl[\,u^{(r-T_{\text{in}})};\dots;u^{(r-1)}\,\bigr]$ denote the
history consumed at rollout step $r$, so that $X^{(1)}$ is the observation
$a$ of Eq.~\eqref{eq:input_window}. The network predicts the next snapshot
$\hat{u}^{(r)}$ from $[X^{(r)};\,x_i;\,z_i]$, and the window then advances by
discarding its oldest snapshot and appending that prediction,
\begin{equation}
    X^{(r+1)}
    = \bigl[\,u^{(r-T_{\text{in}}+1)};\dots;u^{(r-1)};\,\hat{u}^{(r)}\,\bigr],
    \label{eq:window_shift}
\end{equation}
so the window width---and with it $d_{\text{in}}$ and $W_{\text{in}}$---is
unchanged across the rollout. The window is seeded with observed snapshots and
turns progressively synthetic: after $r$ steps it holds
$\max(0,\,T_{\text{in}}-r)$ observed and $\min(r,\,T_{\text{in}})$ predicted
snapshots, so beyond $r = T_{\text{in}}$ the model forecasts entirely from its
own output. Training must therefore expose the network to its own error
distribution, which is the role of the pushforward noise described in the
Appendix.

\subsection{Tight-Frame Regularization}

The trainable scales are optimized jointly with the prediction objective. To
prevent them from collapsing onto a narrow part of the spectrum, we regularize
the coverage of the resulting wavelet bank.

\begin{definition}[Hypergraph wavelet frame coverage and defect]
\label{def:frame}
For scales $\mathbf{s}=\{s_j\}_{j=0}^{J-1}$, the \emph{total spectral
coverage} of the bank is
$G_{\mathbf{s}}(\lambda):=\sum_{j=0}^{J-1} g(s_j\lambda)^2$.
The scales form a \emph{tight frame} on the nonzero spectrum
$\sigma_+(\Delta)=\{\lambda_k:\lambda_k>0\}$ if $G_{\mathbf{s}}(\lambda)=C$
for all $\lambda\in\sigma_+(\Delta)$ and some $C>0$, and the
\emph{tight-frame defect} is
\begin{equation}
    \mathcal{D}_{\mathrm{TF}}(\mathbf{s})
    := \mathrm{Var}_{\lambda\in\sigma_+(\Delta)}\!\big[G_{\mathbf{s}}(\lambda)\big].
    \label{eq:tf_defect}
\end{equation}
\end{definition}

The restriction to $\sigma_+(\Delta)$ reflects that the band-pass kernel
vanishes at $\lambda=0$ (i.e.\ $g(0)=0$), so the DC mode is excluded by
construction. A vanishing defect is the condition under which the wavelet bank
preserves signal energy and admits a stable inverse on the non-DC subspace.

\begin{proposition}[Tight frame $\Rightarrow$ energy conservation and stable inversion]
\label{prop:frame}
Let $\mathcal{W}_{s_j}$ be the wavelet analysis operator of scale $s_j$
(Eq.~\eqref{eq:wavelet_atom}; so that, at layer $\ell$,
$\mathbf{\Psi}^{\mathrm{base},(\ell)}_j=
\mathcal{W}_{s_j^{(\ell)}}\mathbf{h}^{(\ell)}$ in
Eq.~\eqref{eq:base_path}), and let $P_+$ be the orthogonal projector onto the
eigenvectors of $\Delta$ with nonzero eigenvalue. If the scales $\mathbf{s}$
form a tight frame with constant $C$ (Definition~\ref{def:frame}), then every
signal $x\in\mathbb{R}^{n}$ satisfies the Parseval identity
\[
   \sum_{j=0}^{J-1}\bigl\|\mathcal{W}_{s_j}x\bigr\|_2^{2}=C\,\|P_+x\|_2^{2},
\]
the frame operator collapses to
$S:=\sum_{j}\mathcal{W}_{s_j}^{\top}\mathcal{W}_{s_j}=C\,P_+$, and the non-DC
component of $x$ is recovered by the stable inversion
$P_+x=C^{-1}\sum_{j}\mathcal{W}_{s_j}^{\top}\bigl(\mathcal{W}_{s_j}x\bigr)$.
\end{proposition}
\begin{proof}
Because the kernel $g$ is real-valued and $\Delta=U\Lambda U^{\top}$ is
symmetric, each analysis operator is self-adjoint, giving
\begin{equation}
   \mathcal{W}_{s_j}^{\top}\mathcal{W}_{s_j}
   =\mathcal{W}_{s_j}^{2}
   =U\,g(s_j\Lambda)^{2}\,U^{\top}.
   \label{eq:app_p2_self}
\end{equation}
Summing over the $J$ scales, the frame operator is
\begin{equation}
   S=\sum_{j=0}^{J-1}\mathcal{W}_{s_j}^{\top}\mathcal{W}_{s_j}
    =U\Bigl(\sum_{j=0}^{J-1} g(s_j\Lambda)^{2}\Bigr)U^{\top},
   \label{eq:app_p2_S}
\end{equation}
so $S$ shares the eigenbasis $U$ and acts on the $k$-th mode by multiplication
with the total spectral coverage
\begin{equation}
   G_{\mathbf{s}}(\lambda_k)=\sum_{j=0}^{J-1} g(s_j\lambda_k)^{2}.
   \label{eq:app_p2_cov}
\end{equation}
For every $\lambda_k>0$ the tight-frame hypothesis
(Definition~\ref{def:frame}) gives $G_{\mathbf{s}}(\lambda_k)=C$, while for
$\lambda_k=0$ the assumption $g(0)=0$ gives $G_{\mathbf{s}}(0)=0$. Hence $S$
multiplies by $C$ on the span of the nonzero-eigenvalue eigenvectors and by $0$
on the null space of $\Delta$, i.e.
\begin{equation}
   S=C\,P_+ .
   \label{eq:app_p2_CP}
\end{equation}
The energy identity now follows from Eq.~\eqref{eq:app_p2_CP}: for any
$x\in\mathbb{R}^{n}$, using $P_+=P_+^{\top}=P_+^{2}$,
\begin{equation}
\begin{split}
   \sum_{j=0}^{J-1}\|\mathcal{W}_{s_j}x\|_2^{2}
   &=x^{\top}Sx=C\,x^{\top}P_+x\\
   &=C\,\|P_+x\|_2^{2}.
\end{split}
   \label{eq:app_p2_parseval}
\end{equation}
Finally, applying the identity $C^{-1}S=P_+$ to $x$ yields the stable
reconstruction
\begin{equation}
   P_+x=C^{-1}\sum_{j=0}^{J-1}\mathcal{W}_{s_j}^{\top}
   \bigl(\mathcal{W}_{s_j}x\bigr),
   \label{eq:app_p2_recon}
\end{equation}
whose analysis--synthesis gain is the constant $C^{-1}$ on every nonzero mode,
independent of the eigenvalue.
\end{proof}

\subsubsection*{Eigendecomposition-Free Evaluation}
Although Definition~\ref{def:frame} states the frame condition on the exact
spectrum $\sigma_+(\Delta)$, we never form it. The coverage
$G_{\mathbf{s}}(\lambda)=\sum_j g(s_j\lambda)^2$ is a smooth function of the
scales alone, so it can be sampled at any frequency without eigenvectors, and the
spectral band is fixed a priori since $\sigma(\Delta)\subseteq[0,1]$ for every
hypergraph (Appendix~\ref{sec:theory}, Lemma~\ref{lem:spectral_bound}). We
therefore flatten
$G_{\mathbf{s}}$ across this band by minimizing a cheap surrogate defect,
\begin{equation}
    \widehat{\mathcal{D}}_{\mathrm{TF}}(\mathbf{s})
    := \mathrm{Var}_{r=1,\dots,R}\!\big[G_{\mathbf{s}}(\lambda_r)\big],
    \label{eq:tf_defect_surrogate}
\end{equation}
evaluated on a fixed frequency grid $\{\lambda_r\}$ over the band.
The only spectral quantity required is $\lambda_{\max}$, estimated once by sparse
power iteration~\cite{golub2013matrix} at $\mathcal{O}(\mathrm{nnz}(\Delta))$ cost
($\lambda_{\max}\approx0.997$ on our hypergraphs); driving
$\widehat{\mathcal{D}}_{\mathrm{TF}}$ to zero thus flattens $G_{\mathbf{s}}$ with
no eigendecomposition at any stage.

Because trainable scales can drift from their dyadic initialization and
break this property, we penalize the per-layer surrogate defect
$\widehat{\mathcal{D}}_{\mathrm{TF}}(\mathbf{s}^{(\ell)})$, averaged over all $L$
AWB layers to jointly regularize the scales in each block:
\begin{equation}
    \mathcal{L}_{\mathrm{TF}}
    = \frac{1}{L}\sum_{\ell=0}^{L-1}
      \widehat{\mathcal{D}}_{\mathrm{TF}}\!\bigl(\mathbf{s}^{(\ell)}\bigr).
    \label{eq:tight_frame_loss}
\end{equation}
The total training objective is
\begin{equation}
    \mathcal{L}
    = \mathcal{L}_{\mathrm{data}}
    + \alpha_{\mathrm{grad}}\,\mathcal{L}_{\mathrm{grad}}
    + \beta_{\mathrm{TF}}\,\mathcal{L}_{\mathrm{TF}},
    \label{eq:total_loss}
\end{equation}
where $\mathcal{L}_{\mathrm{data}}$ is the data-fidelity term,
$\mathcal{L}_{\mathrm{grad}}$ is a periodic finite-difference gradient
loss that penalizes spatial smoothness errors, and
$\alpha_{\mathrm{grad}}, \beta_{\mathrm{TF}} \geq 0$ are hyperparameters.
On steady tasks $\mathcal{L}_{\mathrm{data}}$ is the relative $L_2$ error of the
single forward pass. On time-dependent tasks it is accumulated over the rollout
of Eq.~\eqref{eq:window_shift} under a per-step weight,
\begin{equation}
    \mathcal{L}_{\mathrm{data}}
    = \sum_{r=1}^{T_{\mathrm{roll}}} w_r\,
      \frac{\lVert \hat{u}^{(r)} - u^{(r)} \rVert_2}
           {\lVert u^{(r)} \rVert_2},
    \qquad
    w_r = 1 + \frac{r-1}{T_{\mathrm{roll}}-1}\bigl(w_{\mathrm{end}}-1\bigr),
    \label{eq:rollout_loss}
\end{equation}
which ramps linearly along the rollout from $w_1 = 1$ at the first step to
$w_{T_{\mathrm{roll}}} = w_{\mathrm{end}}$ at the last, and reduces to uniform
weighting at $w_{\mathrm{end}} = 1$. Late steps predict from a window that is
largely self-generated and carry compounded error; left unweighted they are
dominated by the easy first step, and the optimizer trades long-horizon accuracy
for one-step accuracy. The weights are constants inside the forward loss
graph---not a curriculum annealed over epochs, and not a modification of the
gradients---so they propagate as ordinary coefficients. The Appendix lists the
per-benchmark $w_{\mathrm{end}}$ and gives the complete training and inference
procedure as pseudocode, from hypergraph construction through the AWB stack to
the autoregressive rollout.

\subsection{Permutation Equivariance}

Because HALO builds the hypergraph directly from the node coordinates and
processes nodes through shared, row-wise operations, its prediction cannot
depend on how the nodes happen to be ordered---a prerequisite for a
discretization-consistent operator~\cite{kovachki2023neural}.

\begin{proposition}[Permutation equivariance]
\label{prop:equiv}
Let $P\in\mathbb{R}^{n\times n}$ be any permutation matrix. If the nodes are
relabeled so that $\mathbf{h}^{(0)}\mapsto P\mathbf{h}^{(0)}$ and
$\Delta\mapsto P\Delta P^{\top}$ (equivalently
$\widetilde{\Delta}\mapsto P\widetilde{\Delta}P^{\top}$), then the HALO output
transforms equivariantly, $\hat{u}\mapsto P\hat{u}$.
\end{proposition}
\begin{proof}
Since $P$ is a permutation matrix, $P^{\top}P=I$, so for every $k\ge0$
\begin{equation}
   (P\widetilde{\Delta}P^{\top})^{k}
   =P\,\widetilde{\Delta}^{k}\,P^{\top},
   \label{eq:app_p3_power}
\end{equation}
and therefore, for every Chebyshev polynomial,
\begin{equation}
   T_m(P\widetilde{\Delta}P^{\top})
   =P\,T_m(\widetilde{\Delta})\,P^{\top}.
   \label{eq:app_p3_cheby}
\end{equation}
Applying the shared basis to the relabeled features and using
Eq.~\eqref{eq:app_p3_cheby},
\begin{equation}
\begin{split}
   T_m(P\widetilde{\Delta}P^{\top})
   \bigl(P\mathbf{h}^{(\ell)}\bigr)
   &=P\,T_m(\widetilde{\Delta})\,P^{\top}P\mathbf{h}^{(\ell)}\\
   &=P\,T_m(\widetilde{\Delta})\mathbf{h}^{(\ell)},
\end{split}
   \label{eq:app_p3_apply}
\end{equation}
so both the Base and Delta wavelet responses
(Eqs.~\eqref{eq:base_path}--\eqref{eq:delta_path}) transform as
$\Psi\mapsto P\Psi$. The remaining components of the AWB update---the input
uplift, the cross-scale mixing $\mathbf{W}_{\mathrm{mix}}$
(Eq.~\eqref{eq:scale_mix}), the skip map $W_{\text{skip}}$, the GELU
activation, and Layer Normalization (Eq.~\eqref{eq:wib_output})---all act
identically and independently on each node, i.e.\ row-wise, and therefore
commute with the row permutation $P$:
\begin{equation}
   \mathbf{h}^{(\ell)}\mapsto P\mathbf{h}^{(\ell)}
   \quad\Longrightarrow\quad
   \mathbf{h}^{(\ell+1)}\mapsto P\mathbf{h}^{(\ell+1)}.
   \label{eq:app_p3_layer}
\end{equation}
By induction over the $L$ AWB layers and the row-wise output projection
(Eq.~\eqref{eq:output_proj}),
\begin{equation}
   \hat{u}\mapsto P\hat{u}.
   \label{eq:app_p3_out}
\end{equation}
Finally, the $k$-NN/$k$-ring builder depends only on the node coordinates, so
relabeling the nodes produces exactly the relabeled
$\Delta\mapsto P\Delta P^{\top}$ assumed above.
\end{proof}
\section{Experiments}
\label{sec:experiments}

We evaluate HALO on a suite of eight PDE benchmarks spanning regular grids (Darcy, Navier--Stokes, Allen--Cahn~\cite{li2021fourier,tripura2023wavelet}), unstructured meshes (Airfoil, Cylinder~\cite{pfaff2021learning}), 3D volumes (Supernova~\cite{ohana2024well}), and 3D aerodynamic geometries (ShapeNet-Car~\cite{wu2024transolver}, Aircraft~\cite{luo2025transolverpp}). All baselines are \textit{reproduced} in our setup under identical train/test splits and resolutions, and we report the mean relative $L_2$ error. Every HALO result is the mean over three independent training runs; its accompanying standard deviation is reported directly in Tables~\ref{tab:benchmark} and~\ref{tab:large_geom_results}. HALO is trained with the composite objective in Eq.~\eqref{eq:total_loss}; time-dependent problems additionally use autoregressive rollout training with a step-weighted loss and pushforward noise~\cite{brandstetter2022message}. All experiments use two NVIDIA~A40 48\,GB GPUs. Full dataset specifications, dataset-specific optimization and architecture settings, and representative held-out qualitative predictions are collected in the Appendix. Table~\ref{tab:benchmark} summarizes the primary results.

\subsection{Results}
Tables~\ref{tab:benchmark} and~\ref{tab:large_geom_results} show that HALO is the only
operator ranked in the top two on the field error of all eight benchmarks, and it
is best or tied-best on seven, trailing only on Navier--Stokes. It retains
the mesh flexibility of graph operators while narrowing the accuracy gap with
fixed-discretization frequency and state-space models.

\begin{table}[htbp]
\centering
{\fontsize{8.5}{10.2}\selectfont
\renewcommand{\arraystretch}{1.20}
\setlength{\tabcolsep}{4.75pt}
\begin{tabular*}{\linewidth}{@{\extracolsep{\fill}}llccccccc}
\toprule
\multicolumn{3}{c}{\textbf{Model Specification}} &
\multicolumn{3}{c}{\textbf{2D Regular Grid}} &
\multicolumn{2}{c}{\textbf{2D Unstructured Grid}} &
\multicolumn{1}{c}{\textbf{3D Volume}} \\
\cmidrule(lr){1-3}\cmidrule(lr){4-6}\cmidrule(lr){7-8}\cmidrule(lr){9-9}
\textbf{Family} & \textbf{Operator} &
{\fontsize{8}{9.5}\selectfont\bfseries Res.-Eq.} &
\textbf{Darcy} &
{\fontsize{8}{9.5}\selectfont\bfseries Navier--Stokes} &
{\fontsize{8}{9.5}\selectfont\bfseries Allen--Cahn} &
\textbf{Airfoil} &
\textbf{Cylinder} &
\textbf{Supernova} \\
\midrule
\multicolumn{3}{@{}l}{\textit{Problem regime}} &
Steady-state & Time-dep. & Steady-state &
Time-dep. & Time-dep. & Time-dep. \\
\midrule
\multirow{4}{*}{\textsc{Spectral}}
 & FNO~\cite{li2021fourier} & \cmark & 1.08 & \secondbase{13.67} & 2.39 & --     & --     & \thirdbase{55.76} \\
 & Geo-FNO~\cite{li2022geofno} & \cmark & 1.08 & \secondbase{13.67} & 2.39 & 2.80 & 59.06 & \thirdbase{55.76} \\
 & WNO~\cite{tripura2023wavelet} & \cmark & 0.84 & 38.49 & 6.43 & --     & --     & \secondbase{55.31} \\
 & LSM~\cite{wu2023latent} & \xmark & \thirdbase{0.65} & 22.25 & 4.43 & 1.38 & 59.55$^\ast$ & \bestbase{52.64} \\
\midrule
\multirow{2}{*}{\textsc{Transformer}}
 & Transolver~\cite{wu2024transolver} & \xmark & \secondbase{0.57} & 17.20 & \secondbase{1.73} & \bestbase{\textbf{1.14}} & \bestbase{3.28}$^\ast$ & 69.06 \\
 & GNOT~\cite{hao2023gnot} & \xmark & 1.05 & \thirdbase{13.80} & 64.53 & \thirdbase{1.21} & 53.14$^\ast$ & 73.35 \\
\midrule
\textsc{DeepONet} & POD-DeepONet~\cite{lu2021deeponet} & \xmark & 4.17 & 25.85 & 63.85 & 74.78 & --     & 80.58 \\
\midrule
\textsc{State-space} & LaMO~\cite{tiwari2025latent} & \xmark & \bestbase{\textbf{0.41}} & \bestbase{\textbf{3.68}} & \bestbase{0.74} & \secondbase{1.20} & \secondbase{3.80}$^\ast$ & 78.44 \\
\midrule
\multirow{2}{*}{\textsc{Graph}}
 & GNO~\cite{li2020neural} & \cmark & 3.46 & 45.35 & \thirdbase{2.15} & 2.55 & \thirdbase{7.95} & 65.24 \\
 & \ourscell{\textbf{HALO (Ours)}} & \ourscell{\cmark} & \ourscell{\textbf{0.41}} & \ourscell{12.14} & \ourscell{\textbf{0.46}} & \ourscell{\textbf{1.14}} & \ourscell{\textbf{2.50}} & \ourscell{\textbf{40.04}} \\
 & \ourscell{} & \ourscell{} & \ourscell{$\pm 0.0115$} & \ourscell{$\pm 0.2000$} & \ourscell{$\pm 0.0058$} & \ourscell{$\pm 0.0135$} & \ourscell{$\pm 0.0289$} & \ourscell{$\pm 0.4000$} \\
\bottomrule
\end{tabular*}

\vspace{3pt}
\begin{minipage}{\linewidth}
\small
\textit{Legend.}\enspace
\colorbox{TableBest}{\strut\hspace{0.9em}} strongest baseline;\enspace
\colorbox{TableSecond}{\strut\hspace{0.9em}} second baseline;\enspace
\colorbox{TableThird}{\strut\hspace{0.9em}} third baseline;\enspace
\colorbox{TableTint}{\strut\hspace{0.9em}} HALO.
\par\smallskip
\cmark/\xmark: resolution-equivariant/not resolution-equivariant;
``--'': unsupported; $^\ast$: dynamic node-set wrapper. Bold marks the best
overall result, including ties.
\end{minipage}
}
\caption{Mean relative $L_2$ error across the six field-prediction benchmarks
($\times 10^{-2}$); lower is better. For HALO, the second line gives the
standard deviation over three independent training runs. Baseline ranks are
computed independently for each dataset.}
\label{tab:benchmark}
\end{table}

\begin{table}[htbp]
\centering
{\fontsize{8.5}{10.2}\selectfont
\renewcommand{\arraystretch}{1.20}
\setlength{\tabcolsep}{3.0mm}
\begin{tabular}{lccccccccc}
\toprule
\multicolumn{2}{c}{\textbf{Model Specification}} &
\multicolumn{4}{c}{\textbf{ShapeNet-Car}} &
\multicolumn{4}{c}{\textbf{Aircraft}} \\
\cmidrule(lr){1-2}\cmidrule(lr){3-6}\cmidrule(lr){7-10}
\textbf{Model} & {\fontsize{8.5}{10}\selectfont\bfseries Res.-Eq.} &
\textbf{Vol.} $\downarrow$ & \textbf{Surf.} $\downarrow$ &
$C_D$ $\downarrow$ & $\rho_D$ $\uparrow$ &
\textbf{Field} $\downarrow$ & $C_l$ $\downarrow$ & $R_l^2$ $\uparrow$ &
\textbf{Surf.} $\downarrow$ \\
\midrule
PointNet~\cite{qi2017pointnet}        & \xmark & 4.94 & 11.04 & 2.98 & 95.83 & 15.20 & 9.50 & 98.20 & 16.90 \\
Graph U-Net~\cite{gao2019graphunets}  & \xmark & 4.71 & 11.02 & 2.26 & 97.25 & --    & 6.30 & 95.30 & 16.10 \\
MGN~\cite{pfaff2021learning}          & \xmark & 3.54 & \thirdbase{7.81} & 1.68 & 98.40  & 10.20 & 3.80 & \thirdbase{99.30} & 11.30 \\
\midrule
GNO~\cite{li2020neural}               & \cmark & 3.83 &  8.15 & 1.72 & 98.34 & 11.40 & \thirdbase{3.10} & 99.10 & 12.90 \\
Geo-FNO~\cite{li2022geofno}           & \cmark &16.70 & 23.78 & 6.64 & 82.80 & --    & --   & --    & --    \\
Galerkin~\cite{cao2021galerkin}       & \xmark & 3.39 &  8.78 & 1.79 & 97.64 & 10.10 & 6.90 & 87.90 & 11.80 \\
GNOT~\cite{hao2023gnot}               & \xmark & 3.29 &  7.98 & 1.78 & 98.33 & \thirdbase{8.11} & 3.30 & 99.10 & \thirdbase{9.30} \\
GINO~\cite{li2023gino}                & \cmark & 3.86 &  8.10 & 1.84 & 98.26 & 11.90 & 4.70 & 98.30 & 13.30 \\
3D-GeoCA~\cite{deng2024geoca}         & \xmark & \thirdbase{3.19} & \secondbase{7.79} & \thirdbase{1.59} & \thirdbase{98.42} & -- & \secondbase{2.20} & \thirdbase{99.30} & 9.70 \\
Transolver~\cite{wu2024transolver}    & \xmark & \secondbase{2.17} & 7.86 & \secondbase{1.36} & \bestbase{\textbf{99.08}} & \secondbase{5.23} & 3.70 & \secondbase{99.42} & \secondbase{9.20} \\
Transolver++~\cite{luo2025transolverpp} & \xmark & \bestbase{\textbf{2.15}} & \bestbase{7.72} & \bestbase{\textbf{1.33}} & \secondbase{99.00} & \bestbase{4.98} & \bestbase{1.40} & \bestbase{99.95} & \bestbase{6.40} \\
\midrule
\ourscell{\textbf{HALO (Ours)}} & \ourscell{\cmark} & \ourscell{\textbf{2.15}} & \ourscell{\textbf{7.70}} & \ourscell{1.39} & \ourscell{98.92} & \ourscell{\textbf{3.36}} & \ourscell{\textbf{1.38}} & \ourscell{\textbf{99.97}} & \ourscell{\textbf{4.93}} \\
\ourscell{} & \ourscell{} & \ourscell{$\pm 0.04$} & \ourscell{$\pm 0.16$} & \ourscell{$\pm 0.02$} & \ourscell{$\pm 0.02$} & \ourscell{$\pm 0.04$} & \ourscell{$\pm 0.02$} & \ourscell{$\pm 0.02$} & \ourscell{$\pm 0.16$} \\
\bottomrule
\end{tabular}
}
\caption{Aerodynamic prediction on two large 3D unstructured geometries
($\times10^{-2}$). HALO entries are mean $\pm$ standard deviation over three
independent runs. Cell colors follow the baseline-rank and HALO legend in
Table~\ref{tab:benchmark}; bold marks the best overall result, including ties.
``--'' denotes an unsupported setting. MGN: MeshGraphNet.}
\label{tab:large_geom_results}
\end{table}

\subsubsection*{Run-to-Run Stability}
Every HALO entry is the mean of three independent runs, with its standard
deviation shown directly beneath the mean in Tables~\ref{tab:benchmark}
and~\ref{tab:large_geom_results}. Steady-state and short-horizon mesh tasks vary
only in the third decimal of the unscaled relative error, whereas the longer
Navier--Stokes and Supernova rollouts show larger first-decimal variation in the
$\times10^{-2}$ reporting scale. Where HALO leads outright, this spread remains
well below the gap to the runner-up; for the reported Darcy--LaMO and
Airfoil--Transolver ties, the spread is comparable to the method difference, so
we make no finer ranking claim.

The benefit is clearest on unstructured meshes. Fixed-discretization models
remain competitive on Airfoil, whose samples share a mesh, but require a dynamic
node-set adapter for the variable-mesh Cylinder task. HALO instead constructs
its hypergraph directly from each sample and remains permutation-equivariant,
reducing Cylinder error by approximately $23.8\%$ relative to the strongest
adapted baseline and by $68.6\%$ relative to GNO without assuming a common
discretization; the corresponding native-mesh rollouts appear in the Appendix.

\subsubsection*{Large-Geometry Aerodynamics}
To assess scalability, we evaluate HALO on ShapeNet-Car~\cite{wu2024transolver}
($32$k-point meshes) and the larger Aircraft dataset~\cite{luo2025transolverpp},
whose CFD meshes reach roughly $330$k points per case.
Table~\ref{tab:large_geom_results} reports the complete comparison, including
the design-level drag and lift metrics as well as the underlying field errors.

On ShapeNet-Car, HALO obtains the lowest surface error ($7.70$) and ties the
best volume error ($2.15$). Transolver and Transolver++ remain slightly ahead on
$C_D$ and $\rho_D$, but among resolution-equivariant operators HALO is strongest
on all four metrics.

\begin{figure}[htbp]
\centering
\includegraphics[width=0.66\linewidth]{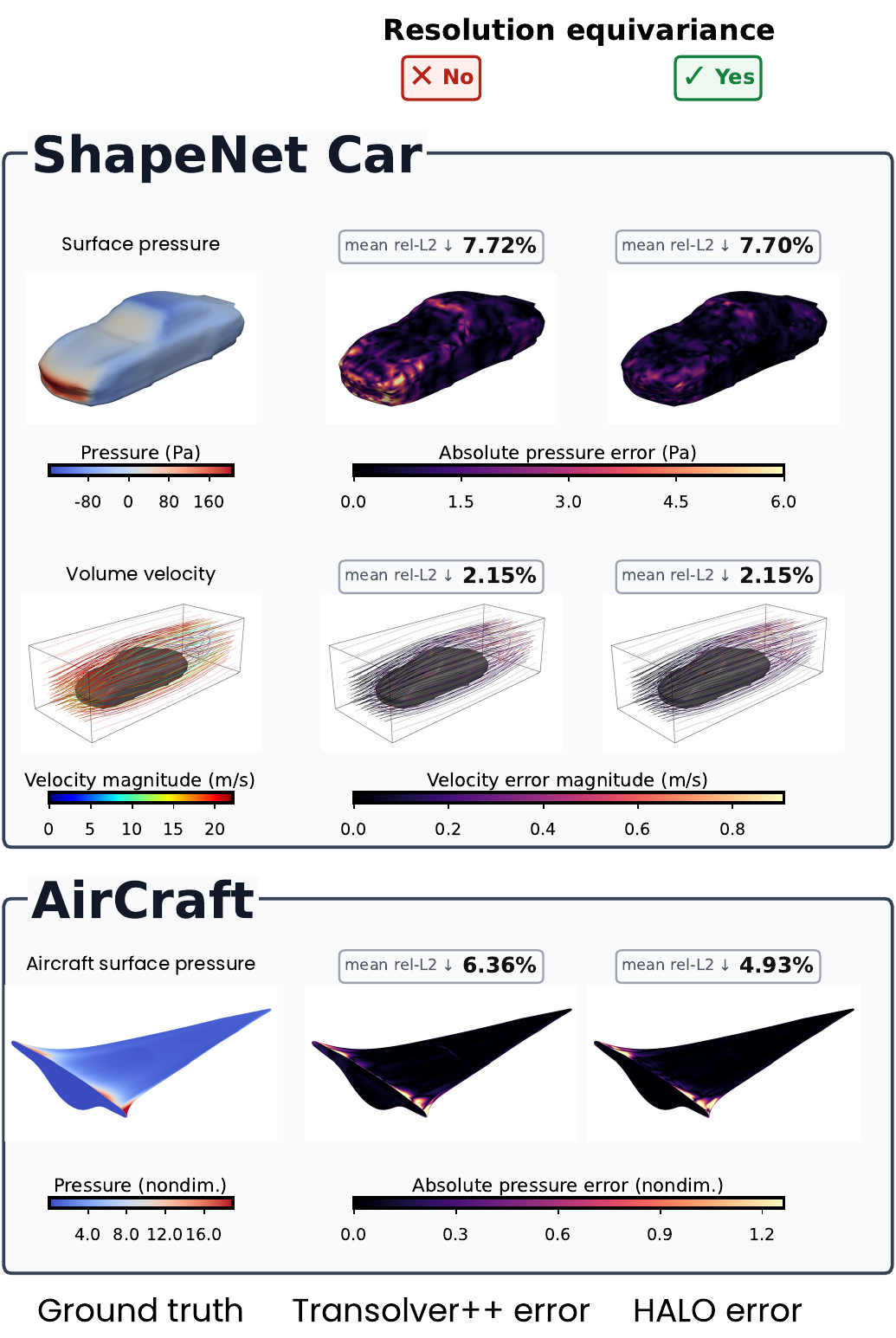}
\caption{Large-geometry prediction. Columns show the reference and
absolute errors of Transolver++ and HALO. ShapeNet-Car includes surface pressure
and volume velocity; Aircraft shows surface pressure. Labels report mean
relative-$L_2$ error, and $\times$/$\checkmark$ marks resolution equivariance.}
\label{fig:large_geom}
\end{figure}

The distinction becomes clearer on Aircraft: HALO leads every metric, reducing
Transolver++ field error from $4.98$ to $3.36$ and surface-pressure error from
$6.40$ to $4.93$, improvements of $32.5\%$ and $23.0\%$, respectively. This
widening advantage on a mesh roughly ten times larger supports the linear node
scaling of the sparse hypergraph operator (Remark~\ref{rem:cost}). Because each
sample here carries its own mesh, and these are the largest meshes in the suite,
both benchmarks keep the Laplacian in factored incidence form rather than
assembling it, which is what makes holding one operator per geometry feasible
(Appendix~\ref{app:factored_laplacian}).
Figure~\ref{fig:large_geom} shows that the quantitative gains coincide with
faithful pressure and velocity fields rather than only improved aggregate
coefficients. Dataset specifications and additional held-out predictions are
provided in the Appendix.

The two datasets also separate local field reconstruction from design-level
fidelity: surface and volume errors assess the predicted CFD state, whereas
$C_D$, $\rho_D$, $C_l$, and $R_l^2$ assess the induced aerodynamic statistics.
Agreement across both views is therefore more informative than a field-error
improvement alone.
\begin{figure}[htbp]
\centering
\includegraphics[width=\linewidth]{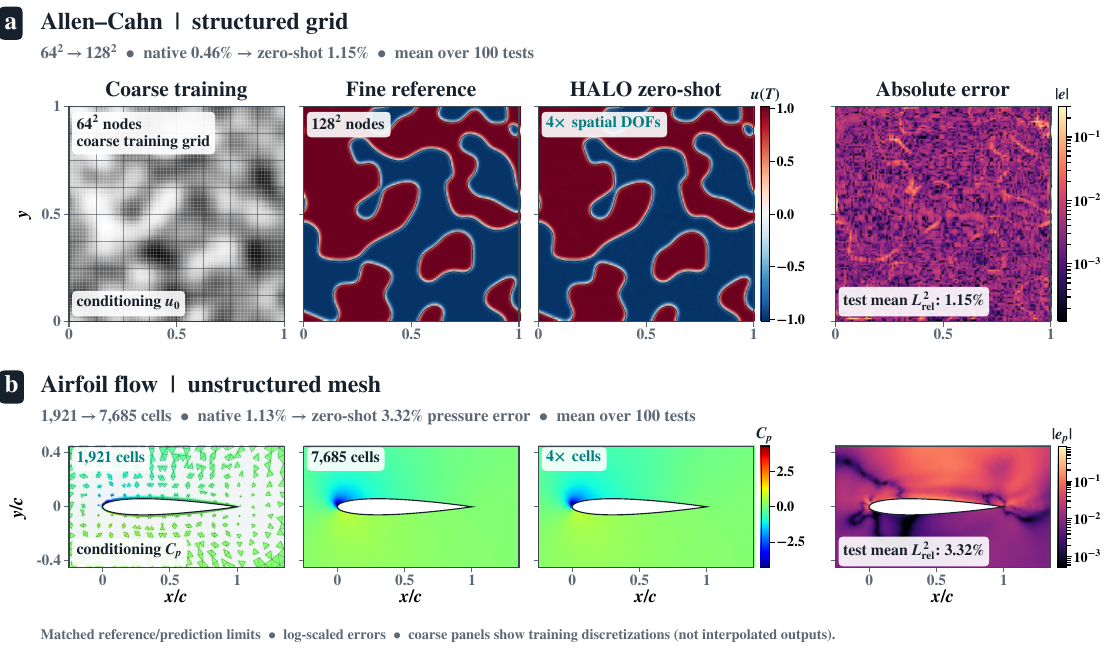}
\vspace{5pt}

{\fontsize{9}{10.8}\selectfont
\setlength{\fboxsep}{5pt}
\noindent\colorbox{TableSecond}{%
\begin{minipage}{\dimexpr\linewidth-2\fboxsep\relax}
\textbf{Structured-grid transfer}\hfill
\textit{mean relative $L_2$ error; lower is better}\par\smallskip
\renewcommand{\arraystretch}{1.08}
\begin{tabular*}{\linewidth}{@{\extracolsep{\fill}}lccc@{}}
\textbf{Dataset} & \textbf{Base $\rightarrow$ target} & \textbf{Base} & \textbf{Zero-shot} \\
Allen--Cahn          & $64^2 \rightarrow 128^2$ & 0.0046 & \textbf{0.0115} \\
Darcy                & $85^2 \rightarrow 421^2$ & 0.0041 & \textbf{0.0082} \\
NS $(\nu{=}10^{-4})$ & $64^2 \rightarrow 256^2$ & 0.0442 & \textbf{0.0556} \\
Supernova            & $32^3 \rightarrow 64^3$  & 0.4004 & \textbf{0.5849}
\end{tabular*}
\end{minipage}}

\vspace{4pt}
\noindent\colorbox{TableTint}{%
\begin{minipage}{\dimexpr\linewidth-2\fboxsep\relax}
\textbf{Unstructured-mesh transfer}\hfill
\textit{training uses ${\sim}25\%$ of target cells}\par\smallskip
\renewcommand{\arraystretch}{1.08}
\begin{tabular*}{\linewidth}{@{\extracolsep{\fill}}lccc@{}}
\textbf{Dataset} & \textbf{Base $\rightarrow$ target} & \textbf{Base} & \textbf{Zero-shot} \\
Airfoil  & $1{,}921 \rightarrow 7{,}685$ & 0.0113 & \textbf{0.0332} \\
Cylinder & $884 \rightarrow 3{,}512$     & 0.0664 & \textbf{0.0902}
\end{tabular*}
\end{minipage}}
}
\caption{Zero-shot super-resolution on structured grids and unstructured
meshes. The field montage contrasts the coarse training discretization, fine
reference, HALO prediction, and absolute error for representative Allen--Cahn
and Airfoil tests. The benchmark ledger reports all six base- and target-resolution
errors. Reference and prediction share color
limits; error maps use logarithmic scales.}
\label{fig:superresolution_collage}
\end{figure}

\subsubsection*{Zero-Shot Transfer}
Every model is trained once on a coarse base discretization and evaluated on a
finer target without fine-tuning (Figure~\ref{fig:superresolution_collage}).
HALO attains the
lowest target-grid error on three of the four structured benchmarks, with FNO
ahead only on Navier--Stokes; the per-baseline breakdown is in the Appendix. On
unstructured meshes it sees just
$25\%$ of the mesh cells in training, yet on the full mesh stays within
$13.5\%$ (Cylinder) and $30.2\%$ (Airfoil) of a GNO trained directly at full
resolution, so the hypergraph structure survives aggressive coarsening.
Relative to the strongest zero-shot baseline at each structured target
resolution, HALO lowers error by $52.9\%$ on Allen--Cahn, $24.8\%$ on Darcy,
and $13.3\%$ on Supernova. Navier--Stokes is the sole exception, where FNO is
$15.6\%$ better. Thus HALO's advantage is not merely retention of its
coarse-grid accuracy; it remains the most accurate transferred operator on
three of four structured targets.

These results point to a core mechanism: resolution equivariance alone does not
guarantee transfer, since the coarse latent must still preserve the higher-order
correlations needed to reconstruct the fine-grid solution. HALO does this by
encoding group-wise geometry into its hyperedges, which sidesteps expensive
fine-resolution retraining at the linear inference cost of
Remark~\ref{rem:cost}.

\subsection{Ablation Studies}
\label{sec:ablation}

The benchmark comparisons above establish end-to-end accuracy across
heterogeneous PDEs. We next isolate which architectural choices account for
that behavior through controlled ablations of (i) higher-order hypergraph
connectivity, (ii) hyperedge neighborhood size, (iii) the learned dyadic scale
bank, and (iv) the low-rank Delta path. Each entry reports a mean and standard
deviation, with errors as relative-$L_2$ percentages; lower is better.
Percentage differences are relative to the control configuration named in each
subsection. Unless stated otherwise, the data split, optimizer, training
schedule, and remaining HALO components are fixed. Each study varies one factor,
so the results provide directional evidence rather than a formal significance
test.

\subsubsection{Effect of Higher-Order Connectivity}

The first ablation tests whether the gain comes from the higher-order incidence
structure itself rather than simply from enlarging a local neighborhood. The
pairwise control uses exactly the same local support as HALO---periodic $k{=}8$
for Allen--Cahn and non-periodic $k{=}24$ for Darcy---but replaces the
hypergraph spectral operator based on Eq.~\eqref{eq:hg_laplacian} with its
pairwise counterpart. Thus, the comparison holds the geometric reach fixed
while changing how local nodes are grouped and normalized spectrally.

Table~\ref{tab:ablation_connectivity} shows that the higher-order bias is not
uniformly beneficial across PDEs. On the periodic Allen--Cahn field, the
pairwise control wins consistently, lowering the error by $10.2\%$, indicating
that the local pairwise support already captures the relevant interface
dynamics at this radius. On Darcy, however, the hypergraph wins consistently,
with the pairwise error $21.5\%$ higher. This is consistent with overlapping
multi-node groups providing a useful inductive bias for aggregating
heterogeneous local permeability structure before predicting the pressure
field. Because the support is matched, the contrast cannot be attributed to a
larger direct neighborhood; it compares the full higher-order spectral geometry
with a pairwise spectral backbone. In particular, it does not claim that
hyperedges dominate pairwise interactions on every smooth, regular-grid
problem.

\begin{table}[H]
\centering
{\small
\setlength{\tabcolsep}{4mm}
\begin{tabular}{@{}lcc@{}}
\toprule
\textbf{Graph type} & \textbf{Allen--Cahn} & \textbf{Darcy} \\
\midrule
Hypergraph & $0.4614\%{\pm}0.0058\%$ & $0.4125\%{\pm}0.0115\%$ \\
Pairwise   & $0.4142\%{\pm}0.0179\%$ & $0.5011\%{\pm}0.0170\%$ \\
\bottomrule
\end{tabular}
}
\caption{Relative-$L_2$ error of the hypergraph operator versus a pairwise
graph with matched local support, on Allen--Cahn ($k{=}8$) and Darcy
($k{=}24$). Entries are mean $\pm$ standard deviation in percent.}
\label{tab:ablation_connectivity}
\end{table}

\subsubsection{Effect of Hyperedge Neighborhood Size}

We next vary only the number of neighbors per hyperedge on Darcy. The parameter
count is exactly the same in both rows, so this is a connectivity-only
comparison rather than a capacity ablation.

Increasing the neighborhood from $k{=}16$ to $k{=}24$ lowers the error
consistently, with the $k{=}16$ error $16.0\%$ higher
(Table~\ref{tab:ablation_receptive_field}). This changes the local groups
entering the incidence matrix and hence the spectrum seen by the shared
Chebyshev basis; it does not add learned weights. It also does not change the
formal $M$-hop locality set by the Chebyshev order
(Proposition~\ref{prop:local}); instead, it changes the connectivity, sparsity,
and geometric coverage within each hop. The result suggests that Darcy benefits
from jointly resolving a broader local coefficient neighborhood. This accuracy
gain has a structural cost: for the $k$-NN construction, the sparse operator has
$\mathcal{O}(nk^2)$ nonzeros (Remark~\ref{rem:cost}), so a larger $k$ increases
the cost of the sparse spectral products even though the trainable parameter
count is unchanged.

\begin{table}[H]
\centering
{\small
\setlength{\tabcolsep}{4mm}
\begin{tabular}{@{}lc@{}}
\toprule
\textbf{Neighborhood size $k$} & \textbf{Relative $L_2$ error} \\
\midrule
$k{=}24$ & $0.4125\%{\pm}0.0115\%$ \\
$k{=}16$ & $0.4786\%{\pm}0.0109\%$ \\
\bottomrule
\end{tabular}
}
\caption{Effect of hyperedge neighborhood size on Darcy at fixed capacity
($1{,}797{,}023$ parameters), so only connectivity changes. Entries are mean
$\pm$ standard deviation.}
\label{tab:ablation_receptive_field}
\end{table}

\subsubsection{Effect of Trainable Wavelet Scales}

The third study isolates the scale bank in the Allen--Cahn configuration. The
control uses three trainable dyadic scales. Freezing those scales retains the
same three spectral bands and nearly the same number of parameters, whereas the
single-scale model removes the additional scale-dependent branches.

The fixed-scale comparison (Table~\ref{tab:ablation_scales}) is the cleanest
test of spectral adaptation: it removes only 12 parameters, so its $2.8\%$
higher error is not explained by capacity. The learned scales can move the
peaks of the Meyer-like filters in Eq.~\eqref{eq:meyer_kernel} to frequencies
emphasized by the phase-field operator while the tight-frame loss in
Eq.~\eqref{eq:tight_frame_loss} keeps aggregate coverage from collapsing. The
effect is modest relative to run-to-run variation, but the aggregate result
favors adaptation. The single-scale model is worse throughout ($14.0\%$ higher
error), showing that multiple bands and their learned cross-scale combination
carry useful information. Since that variant also reduces the parameter count
substantially, it is evidence for the combined multiscale design rather than a
capacity-matched estimate of the effect of $J$ alone.

\begin{table}[H]
\centering
{\small
\setlength{\tabcolsep}{4mm}
\begin{tabular}{@{}lcc@{}}
\toprule
\textbf{Scale setting} & \textbf{Parameters} & \textbf{Relative $L_2$ error} \\
\midrule
Trainable scales & $575{,}629$ & $0.4614\%{\pm}0.0058\%$ \\
Fixed scales     & $575{,}617$ & $0.4741\%{\pm}0.0195\%$ \\
Single scale     & $280{,}709$ & $0.5258\%{\pm}0.0023\%$ \\
\bottomrule
\end{tabular}
}
\caption{Learned multiscale wavelet filtering on Allen--Cahn, comparing the
trainable-scale control with frozen and single-scale banks. Entries are mean
$\pm$ standard deviation.}
\label{tab:ablation_scales}
\end{table}

\subsubsection{Effect of the Delta Path}

Finally, we disable the Delta path of Eq.~\eqref{eq:delta_path} while retaining
the Base wavelet transform, shared Chebyshev recurrence, cross-scale mixing,
and the remaining training protocol. This removes the per-(scale, order)
low-rank correction that is zero-initialized at the start of optimization.

Table~\ref{tab:ablation_delta} shows a consistent loss when the Delta path is
removed ($62.3\%$ higher error). Technically, the Base path remains a
scale-adapted realization of the fixed wavelet response, whereas the Delta path
can learn scale- and polynomial-order-specific corrections through the
down-projection, kernels $\Theta_{j,m}$, and up-projection in
Eq.~\eqref{eq:delta_path}. The ablation therefore supports the utility of that
corrective mechanism, especially after its zero initialization lets the stable
Base response dominate early training. It also removes 311,296 trainable
parameters, however, so the experiment should be interpreted as the
contribution of the Delta path together with its associated capacity; a
capacity-matched no-Delta control would be needed to separate those effects
completely.

\begin{table}[H]
\centering
{\small
\setlength{\tabcolsep}{4mm}
\begin{tabular}{@{}lcc@{}}
\toprule
\textbf{Model} & \textbf{Parameters} & \textbf{Relative $L_2$ error} \\
\midrule
w/ Delta path  & $575{,}629$ & $0.4614\%{\pm}0.0058\%$ \\
w/o Delta path & $264{,}333$ & $0.7488\%{\pm}0.0028\%$ \\
\bottomrule
\end{tabular}
}
\caption{Delta-path ablation on Allen--Cahn over two seeds; entries are mean
$\pm$ standard deviation.}
\label{tab:ablation_delta}
\end{table}

\subsubsection*{Summary}
The four studies identify complementary rather than interchangeable sources of
performance. Higher-order grouping is most helpful on the heterogeneous Darcy
map, whereas a matched pairwise representation is sufficient for the tested
periodic Allen--Cahn setting. At fixed parameter count, Darcy benefits from a
larger hyperedge support; at nearly fixed parameter count, Allen--Cahn benefits
from adapting a multi-band wavelet bank. The Delta path then supplies a learned
low-rank correction beyond the analytic wavelet response. Together, these
controls support the design rationale of HALO while making clear that its
higher-order and multiscale biases are task dependent rather than universal.

\section{Conclusion}

We introduced HALO, a neural operator that represents a PDE domain as a
hypergraph and learns kernels in its spectral wavelet domain. This construction
captures set-valued couplings directly rather than recovering them indirectly
through pairwise graph interactions. Trainable wavelet scales, regularized
toward a tight frame, provide adaptive multi-scale filtering, while the
Chebyshev approximation avoids eigendecomposition and retains a per-layer cost
that is linear in the number of nodes. Across all eight benchmarks, HALO ranks
first or among the strongest methods from every operator family considered. Its
largest gains occur on unstructured problems and the largest geometries, without
sacrificing resolution equivariance.

\subsubsection*{Limitations and Outlook}
HALO remains accurate over the rollout horizons represented during training,
but errors accumulate when inference extends substantially beyond that window;
reliable long-range forecasting therefore remains an open problem. Scale poses
a different unanswered question. The relative gain over the baselines was
largest on the largest geometry evaluated, but our experiments have not yet
reached million-point industrial meshes or coupled multiphysics systems. These
settings are the natural next test of whether the observed advantage persists.

\FloatBarrier

\clearpage
\appendix
\renewcommand{\theequation}{\thesection.\arabic{equation}}
\renewcommand{\thetable}{\thesection.\arabic{table}}
\renewcommand{\thefigure}{\thesection.\arabic{figure}}
\renewcommand{\thealgorithm}{\thesection.\arabic{algorithm}}
\numberwithin{equation}{section}
\numberwithin{table}{section}
\numberwithin{figure}{section}
\numberwithin{algorithm}{section}

\section*{Technical Appendix}
\section{Notation}
\label{sec:notation}

Table~\ref{tab:notation} collects the symbols used throughout the paper.

\begin{table}[H]
\centering
{\fontsize{9}{10}\selectfont
\setlength{\tabcolsep}{1mm}
\begin{tabularx}{\textwidth}{@{}l@{\ }>{\raggedright\arraybackslash}X@{\quad}l@{\ }>{\raggedright\arraybackslash}X@{}}
\toprule
\textbf{Symbol} & \textbf{Meaning} & \textbf{Symbol} & \textbf{Meaning} \\
\midrule
$\mathcal{N}:\mathcal{A}\!\to\!\mathcal{U}$ & Solution operator to be learned & $T_m,\,\mathbf{T}_m$ & Chebyshev polynomial and basis tensor \\
$a^{(i)},\,u^{(i)}$ & Input / output functions of sample $i$ & $c_{s,m}$ & Chebyshev coefficients of $g(s\cdot)$ \\
$D\subset\mathbb{R}^{d}$ & Bounded physical domain, dimension $d$ & $Q,\,\theta_q$ & Quadrature-node count and angles \\
$N$ & Number of training pairs & $x^{*}\!\approx\!1.423$ & Peak location of the kernel $g$ \\
$\mathcal{K}_{\phi}^{(\ell)},\,\kappa_{\phi}$ & Kernel-integration operator and kernel & $s_j^{\mathrm{init}},\,\lambda_j^{*}$ & Dyadic-ladder initial scale, peak eigenvalue \\
$v_{\ell},\,W_{\ell},\,\sigma$ & Hidden state, pointwise map, nonlinearity & $\rho_j$ & Softplus pre-activation for scale $s_j$ \\
$\mathcal{G}=(\mathcal{V},\mathcal{E},W_e)$ & Weighted hypergraph & $G_{\mathbf{s}}(\lambda)$ & Total spectral coverage of the bank \\
$n=|\mathcal{V}|$ & Number of nodes (vertices) & $\sigma_+(\Delta),\,C$ & Nonzero spectrum, tight-frame constant \\
$e\in\mathcal{E}$ & Hyperedge & $S$ & Frame operator $\sum_j\mathcal{W}_{s_j}^{\top}\mathcal{W}_{s_j}$ \\
$H$ & Incidence matrix ($H_{ve}\!=\!1$ iff $v\!\in\!e$) & $\mathcal{D}_{\mathrm{TF}},\,\widehat{\mathcal{D}}_{\mathrm{TF}}$ & Tight-frame defect and surrogate ($R\!=\!64$) \\
\cmidrule(lr){3-4}
$W_e,\,D_v,\,D_e$ & Hyperedge-weight and degree matrices & $\mathbf{h}^{(\ell)}\!\in\!\mathbb{R}^{n\times d_h}$ & Node feature matrix at layer $\ell$ \\
$d_v$ & Vertex degree & $L,\,\ell$ & Number of AWB layers, layer index \\
$\Delta,\,\tilde{\Delta}$ & Normalized, rescaled Laplacian & $d_h$ & Latent (hidden) width \\
$\tilde{H}$ & Normalized incidence matrix & $d_{\mathrm{in}},\,d_{\mathrm{out}}$ & Input / output feature dimensions \\
$\mathrm{nnz}(\tilde{\Delta})$ & Number of nonzeros of $\tilde{\Delta}$ & $d_a,\,c,\,c_f$ & Observation, snapshot, and per-group channels \\
$U,\,\Lambda$ & Eigenvectors and eigenvalues of $\Delta$ & $T_{\mathrm{in}},\,T_{\mathrm{roll}}$ & History window length; rollout horizon \\
$\lambda_k,\,\lambda_{\max}$ & Eigenvalue, spectral upper bound & $\tilde{x}_i,\,z_i$ & Node feature $[a;x;z]$; conditioning channels \\
$\hat{x}=U^{\top}x$ & Hypergraph Fourier transform of $x$ & $W_{\mathrm{in}},\,b_{\mathrm{in}},\,\mathbf{h}^{(0)}$ & Uplift weight, bias, and uplifted input \\
$P_+$ & Projector onto nonzero-$\lambda$ eigenspace & $\Psi^{\mathrm{base}}_{j},\,\Psi^{\mathrm{delta}}_{j},\,\Psi_j$ & Base, Delta, combined per-scale response \\
$d_{\mathcal{G}}(i,v)$ & Hop distance on the hypergraph & $\mathbf{P}_{\downarrow},\,\mathbf{P}_{\uparrow}$ & Down- / up-projections (Delta path) \\
$\delta_i,\,x_i$ & One-hot indicator, coordinate of node $i$ & $d_c$ & Compressed channel width ($d_c\!\ll\!d_h$) \\
$\mathcal{N}_k(i),\,k$ & $k$-NN set and neighborhood size & $\Theta_{j,m}^{(\ell)}$ & Trainable per-order Delta kernels \\
$e_i,\,x_e,\,\sigma_e$ & Per-node hyperedge, anchor, Gaussian width & $\mathbf{V}^{(\ell)}$ & Cross-scale mixed AWB output \\
\cmidrule(lr){1-2}
$g$ & Meyer-like band-pass spectral kernel & $\mathbf{W}_{\mathrm{mix}}^{(\ell)}$ & Cross-scale linear mixing matrix \\
$\mathcal{W}_s=U g(s\Lambda)U^{\top}$ & Wavelet analysis operator at scale $s$ & $W_{\mathrm{skip}},\,W_{\mathrm{out}}$ & Skip-connection and output projections \\
$\psi_{s,i}$ & Wavelet atom at scale $s$, node $i$ & $\hat{u}$ & Model prediction \\
$\hat{\psi}^{(M)}_{s,i}$ & Order-$M$ Chebyshev approximation & $\mathrm{GELU},\,\mathrm{LN}$ & Activation, Layer Normalization \\
$C_{s,i}$ & Wavelet coefficient of a signal & $P$ & Permutation matrix \\
$s,\,s_j$ & Wavelet scale; $j$-th scale & $\mathcal{L}_{\mathrm{data}},\,\mathcal{L}_{\mathrm{grad}}$ & Data ($L_2$) and gradient-penalty losses \\
$J,\,M$ & Number of dyadic scales; Chebyshev order & $\mathcal{L}_{\mathrm{TF}},\,\mathcal{L}$ & Tight-frame and total training loss \\
 &  & $\alpha_{\mathrm{grad}},\,\beta_{\mathrm{TF}}$ & Loss-term weights \\
 &  & $X^{(r)},\,w_r,\,w_{\mathrm{end}}$ & History window; step weight and endpoint \\
\bottomrule
\end{tabularx}
}
\caption{Notation used throughout the paper and this appendix.}
\label{tab:notation}
\end{table}

\section{Theoretical Insights}
\label{sec:theory}

\subsection{Supporting Approximation-Theoretic Results}
\label{sec:approx_results}

The following standard results underpin the spectral rescaling, the tight-frame
surrogate, and the Chebyshev coefficient computation and truncation used in the
main text; we state them with proofs and references for completeness.

\begin{lemma}[\citealp{zhou2006learning}]
\label{lem:spectral_bound}
For every hypergraph, the normalized Laplacian of Eq.~\eqref{eq:hg_laplacian}
satisfies $\sigma(\Delta)\subseteq[0,1]$.
\end{lemma}

\noindent\textbf{\textit{Proof.}}
Define the normalized incidence matrix by
$\tilde{H}=D_v^{-1/2}HW_e^{1/2}D_e^{-1/2}$.
The affinity term of
Eq.~\eqref{eq:hg_laplacian} factors as the Gram matrix
$\tilde{H}\tilde{H}^{\top}$, giving $\Delta = I - \tilde{H}\tilde{H}^{\top}$.
Since $\tilde{H}\tilde{H}^{\top}\succeq 0$, we have $\Delta\preceq I$; combined
with the positive semi-definiteness of $\Delta$ from the symmetric
normalization, this pins the spectrum to $\sigma(\Delta)\subseteq[0,1]$. On a
$2$-uniform hypergraph (an ordinary graph) $\Delta$ reduces to one-half the
normalized graph Laplacian, mapping the familiar bound $[0,2]$ onto $[0,1]$.
\qed

\subsubsection*{Eigendecomposition-Free Surrogate Evaluation}
Because the coverage $G_{\mathbf{s}}(\lambda)=\sum_j g(s_j\lambda)^2$ is a
closed-form, smooth function of the scales, it can be sampled at any frequency,
and by Lemma~\ref{lem:spectral_bound} the nonzero spectrum lies in
$\sigma_+(\Delta)\subset(0,\lambda_{\max}]$ with $\lambda_{\max}\le 1$; flattening
$G_{\mathbf{s}}$ across this fixed band therefore forces it flat on the embedded
spectrum. We accordingly evaluate the surrogate defect
$\widehat{\mathcal{D}}_{\mathrm{TF}}$ of Eq.~\eqref{eq:tf_defect_surrogate} on
$R=64$ points $\{\lambda_r\}$ spaced uniformly on
$[\,10^{-3}\lambda_{\max},\,\lambda_{\max}]$, excluding the DC mode. The lower
cutoff leaves any near-null eigenvalues below it unconstrained; this is
deliberate, since those modes are the ones the band-pass kernel is designed to
suppress, and $G_{\mathbf{s}}$ is continuous, so flatness on the sampled band
controls the coverage up to the cutoff regardless.

\subsubsection*{Estimating $\lambda_{\max}$}
The only spectral quantity the surrogate requires is $\lambda_{\max}$, which also
rescales the operator in Eq.~\eqref{eq:cheby_approx}. We obtain it once by the
power method~\citep{mises1929praktische,golub2013matrix}: from a random unit
$v^{(0)}$, iterate
\begin{equation}
   v^{(t+1)}=\frac{\Delta v^{(t)}}{\lVert\Delta v^{(t)}\rVert_2},
   \qquad
   \widehat{\lambda}^{(t+1)}=\bigl(v^{(t+1)}\bigr)^{\!\top}\!\Delta\,v^{(t+1)},
   \label{eq:app_power_iter}
\end{equation}
until $\lvert\widehat{\lambda}^{(t+1)}-\widehat{\lambda}^{(t)}\rvert$ falls below
tolerance. Each step is one sparse product, so the whole estimate costs
$\mathcal{O}(\mathrm{nnz}(\Delta))$ per iteration and $\mathcal{O}(n)$ memory.
Two properties of $\Delta$ make this the right tool here. It is symmetric
positive semi-definite (Lemma~\ref{lem:spectral_bound}), so the eigenvalue of
largest modulus \emph{is} $\lambda_{\max}$ and no spectral shift is needed to
separate it from a negative extreme; and reading the estimate off the Rayleigh
quotient rather than $\lVert\Delta v^{(t)}\rVert_2$ makes its error quadratic,
not linear, in the eigenvector error. The iteration recovers the dominant
eigenpair only. It never forms the remaining $n-1$ eigenvectors, which is the
sense in which no stage of $\widehat{\mathcal{D}}_{\mathrm{TF}}$ requires an
eigendecomposition. Finally, Rayleigh quotients of a symmetric matrix never
exceed its largest eigenvalue, so Eq.~\eqref{eq:app_power_iter} approaches
$\lambda_{\max}$ from below; the certified upper bound needed to keep the
Chebyshev argument inside $[-1,1]$ is supplied instead by
Lemma~\ref{lem:spectral_bound}, $\lambda_{\max}\le 1$. The measured values in
Table~\ref{tab:learned_scales} ($0.987$ to $0.998$) sit within $1.3\%$ of that
ceiling, so the rescaling leaves almost none of the Chebyshev domain unused.

\begin{lemma}[\citealp{mason2002chebyshev}]
\label{lem:quad}
Let $\theta_q=\pi(q+\tfrac12)/Q$, $q=0,\dots,Q-1$. For any $f(\cos\theta)$ with
$f$ a polynomial of degree $\le 2Q-1$,
\begin{equation}
   \int_0^\pi f(\cos\theta)\,d\theta
   =\frac{\pi}{Q}\sum_{q=0}^{Q-1}f(\cos\theta_q).
   \label{eq:lem_quad}
\end{equation}
Consequently $\frac{2}{\pi}\int_0^\pi \cos(m\theta)\,h(\cos\theta)\,d\theta$ is
computed exactly by $\frac{2}{Q}\sum_q\cos(m\theta_q)\,h(\cos\theta_q)$ whenever
$m+\deg h\le 2Q-1$.
\end{lemma}

\noindent\textbf{\textit{Proof.}}
Substituting $x=\cos\theta$, with $dx=-\sin\theta\,d\theta$ and
$\sin\theta=\sqrt{1-x^2}$ on $[0,\pi]$, maps the integral to
\begin{equation}
   \int_0^\pi f(\cos\theta)\,d\theta
   =\int_{-1}^{1}\frac{f(x)}{\sqrt{1-x^2}}\,dx.
   \label{eq:lem_quad_sub}
\end{equation}
The right-hand side is the Gauss--Chebyshev quadrature of the first kind, whose
$Q$ nodes $x_q=\cos\theta_q$ carry equal weights $\pi/Q$ and integrate every
polynomial of degree $\le 2Q-1$ exactly~\citep{mason2002chebyshev}, which is
Eq.~\eqref{eq:lem_quad}. The second claim follows by taking $f(x)=T_m(x)h(x)$,
of degree $m+\deg h$. \qed

\medskip
\noindent\textit{Remark (applicability to HALO).}
Lemma~\ref{lem:quad} requires a \emph{global} polynomial in the Chebyshev
variable $x=\cos\theta\in[-1,1]$. To keep that variable distinct from the
argument of the spectral kernel, write $\xi:=s\lambda$ for the latter throughout
this subsection, so that the peak of Eq.~\eqref{eq:meyer_kernel} sits at
$\xi^{*}=2-1/\sqrt{3}\approx1.423$. In $\xi$ the kernel is only piecewise
smooth: polynomial on $[0,2]$, rational ($4/\xi^{2}$) beyond it, with breakpoints
at $\xi=1,2$ where $g\in C^{1}$ but not $C^{2}$. Every scale in
Table~\ref{tab:learned_scales} satisfies $s>1$, so the $\xi=1$ breakpoint maps to
an interior point of the rescaled band. The mapped integrand is therefore
piecewise, not globally, polynomial in $x$, and the coefficient quadrature
converges at an algebraic, not exponential, rate in $Q$, consistent with the main
text.

\begin{theorem}[\citealp{hammond2011wavelets}]
\label{thm:cheb-err}
Let $\mathcal{W}_s=U\,g(s\Lambda)\,U^{\top}$, let $\psi_{s,i}=\mathcal{W}_s\delta_i$,
and let $\hat{\psi}^{(M)}_{s,i}=q_M(\Delta)\delta_i$ be its order-$M$ Chebyshev
approximation, where $q_M$ is the degree-$M$ Chebyshev partial sum of
$\lambda\mapsto g(s\lambda)$ on $[0,\lambda_{\max}]$. Then
\begin{equation}
   \bigl\|\psi_{s,i}-\hat{\psi}^{(M)}_{s,i}\bigr\|_2
   \;\le\; \bigl(1+\Lambda_M\bigr)\,E_M\bigl(g(s\cdot)\bigr),
   \label{eq:thm_bound}
\end{equation}
where $E_M(\cdot)$ is the best degree-$M$ uniform approximation error and
$\Lambda_M=\mathcal{O}(\log M)$ is the Lebesgue constant of the degree-$M$
Chebyshev projection. If
$g(s\cdot)\in C^{r}$ on $[0,\lambda_{\max}]$ then $E_M=\mathcal{O}(M^{-r})$; if
$g(s\cdot)$ is analytic in a Bernstein ellipse then $E_M=\mathcal{O}(\rho^{-M})$
for some $\rho>1$.
\end{theorem}

\noindent\textbf{\textit{Proof.}}
Write the approximation as $\hat{\psi}^{(M)}_{s,i}=q_M(\Delta)\delta_i$, so the
error is governed by the operator $g(s\Delta)-q_M(\Delta)$. Throughout the
proof, let
$\|h\|_{\infty}:=\sup_{\lambda\in[0,\lambda_{\max}]}\lvert h(\lambda)\rvert$
denote the uniform norm over the spectral band. Since
$\Delta=U\Lambda U^{\top}$ is symmetric, this operator is symmetric and its
spectral norm equals its largest eigenvalue magnitude,
\begin{equation}
\begin{split}
   \bigl\|g(s\Delta)-q_M(\Delta)\bigr\|_2
   &=\max_k\bigl|g(s\lambda_k)-q_M(\lambda_k)\bigr|\\
   &\le\bigl\|g(s\cdot)-q_M\bigr\|_{\infty}.
\end{split}
   \label{eq:thm_specnorm}
\end{equation}
Therefore, using $\|\delta_i\|_2=1$,
\begin{equation}
\begin{split}
   \bigl\|\psi_{s,i}-\hat{\psi}^{(M)}_{s,i}\bigr\|_2
   &=\bigl\|(g(s\Delta)-q_M(\Delta))\delta_i\bigr\|_2\\
   &\le\bigl\|g(s\cdot)-q_M\bigr\|_{\infty}.
\end{split}
   \label{eq:thm_bound1}
\end{equation}
The degree-$M$ truncated Chebyshev series is near-best in the uniform norm,
\begin{equation}
   \bigl\|g(s\cdot)-q_M\bigr\|_{\infty}
   \le \bigl(1+\Lambda_M\bigr)\,E_M\bigl(g(s\cdot)\bigr),
   \label{eq:thm_nearbest}
\end{equation}
where $E_M$ is the best degree-$M$ uniform error and $\Lambda_M$ is the Lebesgue
constant of the Chebyshev projection, which grows only logarithmically,
$\Lambda_M\sim(4/\pi^2)\log M$~\citep{mason2002chebyshev}. At the orders used
here ($M\le 8$, Tables~\ref{tab:app_cfg_structured}
and~\ref{tab:app_cfg_unstructured}) this prefactor is below $3.2$.
Combining Eqs.~\eqref{eq:thm_bound1}--\eqref{eq:thm_nearbest} gives
Eq.~\eqref{eq:thm_bound}. The rates $E_M=\mathcal{O}(M^{-r})$ for
$g(s\cdot)\in C^{r}$ and $E_M=\mathcal{O}(\rho^{-M})$ for analytic $g(s\cdot)$
are the classical Jackson and Bernstein estimates~\citep{mason2002chebyshev}.
\qed

\medskip
\noindent\textit{Remark (which rate applies).}
The Meyer-like kernel is piecewise smooth and not analytic across $\xi=1,2$, so
the honest guarantee is the algebraic branch $E_M=\mathcal{O}(M^{-r})$ with small
$r$; the coefficients are moreover computed by quadrature
$($Lemma~\ref{lem:quad}$)$, which is only near-best. Exponential or ``spectral''
convergence should therefore not be claimed for this kernel.

\begin{lemma}[\citealp{hammond2011wavelets}]
\label{prop:dc}
If $g(0)=0$, then the wavelet analysis operator
$\mathcal{W}_s=U\,g(s\Lambda)\,U^{\top}$ annihilates the null space of $\Delta$.
In particular the (unnormalized) constant mode $\phi_0:=D_v^{1/2}\mathbf{1}$
satisfies $\Delta\phi_0=0$ and $\mathcal{W}_s\phi_0=0$, and
$\phi_0^{\top}\mathcal{W}_s x=0$ for every signal $x$.
\end{lemma}

\noindent\textbf{\textit{Proof.}}
Let $\Pi:=D_v^{-1/2}HW_eD_e^{-1}H^{\top}D_v^{-1/2}$ denote the affinity operator,
so that $\Delta=I-\Pi$. The unweighted edge-degree identity
$H^{\top}\mathbf{1}=\mathrm{diag}(D_e)$ gives
\begin{equation}
   D_e^{-1}H^{\top}\mathbf{1}=\mathbf{1}_{|\mathcal{E}|}.
   \label{eq:dc_edge}
\end{equation}
Left-multiplying Eq.~\eqref{eq:dc_edge} by $HW_e$ and reading off entry $v$,
\begin{equation}
\begin{split}
   \bigl(HW_eD_e^{-1}H^{\top}\mathbf{1}\bigr)_v
   &=\sum_e (W_e)_{ee}H_{ve}=d_v,\\
   \text{so}\quad HW_eD_e^{-1}H^{\top}\mathbf{1}
   &=D_v\mathbf{1},
\end{split}
   \label{eq:dc_vertex}
\end{equation}
since $d_v$ is by definition the $v$-th diagonal entry of $D_v$. Applying $\Pi$
to the constant mode $\phi_0=D_v^{1/2}\mathbf{1}$ and using
Eq.~\eqref{eq:dc_vertex},
\begin{equation}
\begin{split}
   \Pi\phi_0
   &=D_v^{-1/2}\bigl(HW_eD_e^{-1}H^{\top}\mathbf{1}\bigr)\\
   &=D_v^{-1/2}D_v\mathbf{1}
   =D_v^{1/2}\mathbf{1}
   =\phi_0,
\end{split}
   \label{eq:dc_fixed}
\end{equation}
so $\Delta\phi_0=(I-\Pi)\phi_0=0$; hence $\phi_0$ lies in the null space of
$\Delta$. Expanding the analysis operator in the eigenbasis,
\begin{equation}
   \mathcal{W}_s=\sum_{k}g(s\lambda_k)\,u_k u_k^{\top},
   \label{eq:dc_spectral}
\end{equation}
every term with $\lambda_k=0$ carries the factor $g(0)=0$, so $\mathcal{W}_s$
annihilates the null space; in particular $\mathcal{W}_s\phi_0=0$. Finally, by
symmetry of $\mathcal{W}_s$,
\begin{equation}
   \phi_0^{\top}\mathcal{W}_s x=(\mathcal{W}_s\phi_0)^{\top}x=0,
   \label{eq:dc_final}
\end{equation}
for every signal $x\in\mathbb{R}^{n}$. \qed

\section{Benchmark Dataset Details}
\label{sec:data}

Our evaluation comprises eight PDE families (ten settings when the three
Navier--Stokes viscosities are counted separately). Table~\ref{tab:app_data_2d}
summarizes the five 2D benchmarks, while Table~\ref{tab:app_data_3d} groups the
three 3D benchmarks by structured or unstructured discretization and identifies
whether their prediction fields occupy a volume, a surface, or both. Both tables
report the learning task, discretization, and split of each dataset. Tensor
shapes are ordered as (time, spatial nodes/grid, fields), and ``--'' denotes no
temporal axis. The regular-grid Allen--Cahn, Darcy, and Navier--Stokes tasks
follow the Fourier and wavelet neural-operator benchmarks
\citep{li2021fourier,tripura2023wavelet}; Airfoil and Cylinder use the
MeshGraphNet flow benchmarks~\citep{pfaff2021learning}; Supernova is drawn from
the Well collection~\citep{ohana2024well}; ShapeNet-Car follows the
aerodynamic design benchmark of~\citet{wu2024transolver}; and Aircraft follows
the large-scale aerodynamic benchmark of~\citet{luo2025transolverpp}.

\subsubsection*{Allen--Cahn}
The Allen--Cahn equation is a reaction--diffusion PDE modelling phase separation
in alloys with multiple components. In two spatial dimensions it evolves a phase
field $u$ from an initial condition to a final state. We take the problem
setting unchanged from~\citet{tripura2023wavelet}: \emph{periodic boundary
conditions} on $\partial(0,3)^2$, and an initial condition drawn from the
Gaussian random field specified there. Because the dynamics are chaotic, small
perturbations of $u_0$ diverge significantly, so the learned operator
$u_0(x,y)\mapsto u(x,y,t{=}20\,\mathrm{s})$ is a stringent test. All benchmark
results reported in Table~\ref{tab:benchmark} are trained and compared at a
resolution of $64\times64$, with a $128\times128$ discretization held out for the
zero-shot super-resolution study. We use $1000$ training and $100$ test samples.

\subsubsection*{Darcy Flow}
Darcy flow is the steady state of a second-order linear elliptic PDE on the unit
square that models pressure in a porous medium. The input permeability
(diffusion coefficient) $a$ determines the pressure $u$. We take the problem
setting unchanged from~\citet{li2021fourier}: a homogeneous \emph{Dirichlet
boundary condition} on $\partial(0,1)^2$, a fixed uniform forcing, and a
two-phase piecewise-constant permeability sampled as specified there, so that
different samples correspond to different medium structures. The operator to be
learned is $a\mapsto u$. Reference solutions are
obtained with a second-order finite-difference scheme on a $421\times421$ grid
and \emph{subsampled to $85\times85$} for training, with the native
$421\times421$ resolution retained for zero-shot super-resolution. We use $1000$
training and $200$ test samples.

\subsubsection*{Navier--Stokes}
We use the 2D incompressible Navier--Stokes equation in vorticity form on the
unit torus. We take the problem setting unchanged from~\citet{li2021fourier}: a
spatially \emph{periodic} domain, initial vorticity sampled from the Gaussian
random field specified there, and the fixed forcing used in that benchmark. The
one condition we vary is the viscosity $\nu$, which sets how chaotic the
dynamics are and therefore how difficult the forecast is. The main benchmark in
Table~\ref{tab:benchmark} uses $\nu=10^{-5}$, where the dynamics are strongly
chaotic, and Table~\ref{tab:app_cfg_structured} reports separate configurations
for $\nu=10^{-3}$, $10^{-4}$, and $10^{-5}$. Trajectories are recorded every
unit of time up to a final time $T=20$.
The data are simulated on a $256\times256$ grid and \emph{downsampled to
$64\times64$} for training. The task forecasts ten future vorticity frames from
ten past frames, using $1000$ training and $100$ test trajectories. For zero-shot
super-resolution we additionally simulate a $\nu{=}10^{-4}$ set, training at
$64\times64$ and transferring to the $256\times256$ grid.
\begin{table}[t]
\centering
{\small
\begin{tabular}{clccccc}
\toprule
 & & \multicolumn{3}{c}{\textbf{Regular Grid}} &
 \multicolumn{2}{c}{\textbf{Unstructured Mesh}} \\
\cmidrule(lr){3-5}\cmidrule(lr){6-7}
 & \textbf{Configuration} & \textbf{Allen--Cahn} & \textbf{Darcy} &
 \textbf{Navier--Stokes} & \textbf{Airfoil} & \textbf{Cylinder} \\
\midrule
\multirow{3}{*}{\rotatebox[origin=c]{90}{\textsc{Physics}}}
 & Task   & Final state   & Pressure     & Flow forecast & Flow forecast & Flow forecast \\
 & Input  & Initial field & Permeability & Past states   & Past fields   & Past fields \\
 & Output & Final field   & Pressure     & Future states & Future fields & Future fields \\
\midrule
\multirow{6}{*}{\rotatebox[origin=c]{90}{\textsc{Data}}}
 & Geometry      & Periodic grid & Regular grid & Periodic grid & Shared mesh          & Variable mesh \\
 & Dimension     & 2D            & 2D           & 2D + time     & 2D + time            & 2D + time \\
 & Train / test  & 1000 / 100    & 1000 / 200   & 1000 / 100    & 800 / 100            & 1000 / 100 \\
 & Validation    & --            & --            & --            & 100                   & 100 \\
 & Input tensor  & $(-,64^2,1)$ & $(-,85^2,1)$ & $(10,64^2,1)$ & $(15,3982,4)$ & $(15,N_{\rm mesh},3)$ \\
 & Output tensor & $(-,64^2,1)$ & $(-,85^2,1)$ & $(10,64^2,1)$ & $(15,3982,4)$ & $(25,N_{\rm mesh},3)$ \\
\bottomrule
\end{tabular}
}
\caption{Benchmark dataset details and operator-learning tasks for the five 2D
structured-grid and unstructured-mesh settings.}
\label{tab:app_data_2d}
\end{table}

\subsubsection*{Airfoil}
The Airfoil dataset is a 2D compressible flow over an airfoil, taken from the
MeshGraphNet benchmark~\citep{pfaff2021learning}, where it is generated with the
SU2 solver on a triangular mesh around the wing. It is a time-dependent problem on
an unstructured mesh. Each sample stores a fixed mesh---the node positions, the
triangle connectivity, and a node-type label---together with the flow fields
$(u,v,p,\rho)$, i.e.\ the two velocity components, the pressure, and the density,
which change over time. The original data has $601$ time steps on a mesh of $5233$
nodes and $10216$ cells. We keep the mesh as it is, instead of turning it into a
grid, and reduce the data in two ways so that it is cheaper to train on. In space,
we crop the region around the airfoil and its wake: from half a chord ahead of the
leading edge to two chords behind the trailing edge, and $1.2$ chords above and
below. We keep every triangle that has at least one node inside this box, which
retains the important part of the flow---the airfoil surface, the leading edge, and
the near wake---and brings the mesh down to $3982$ nodes and $7685$ cells, with a
$200$-point airfoil boundary. In time, we keep every $12$th frame ($0,12,\dots,588$),
which takes each trajectory from $601$ down to $50$ frames with a time step of
$0.0024$. We only drop frames and do not interpolate, so the frames we keep are
exact solver outputs. We use $800$ trajectories for training, $100$ for validation,
and $100$ for testing, chosen so that the three splits cover a similar range of
Mach number (about $0.25$ to $0.86$) and angle of attack (about $-25^{\circ}$ to
$+25^{\circ}$). The task is to predict $(u,v,p,\rho)$ for $15$ steps ahead on this
cropped mesh.

\subsubsection*{Cylinder}
The Cylinder dataset is a 2D incompressible flow past a cylinder, also from the
MeshGraphNet benchmark~\citep{pfaff2021learning}, generated with the COMSOL solver
on a triangular mesh. Like Airfoil, it is a time-dependent problem on an
unstructured mesh, and each sample stores a fixed mesh (node positions, triangle
connectivity, node type) and the fields $(u,v,p)$, the two velocity components and
the pressure. The difference is that here the mesh is not the same across samples:
the number of nodes changes from one trajectory to another (about $1732$ to $2059$
nodes in the training set). The original data has $600$ time steps with a time step
of $0.01$. We keep every $6$th frame ($0,6,\dots,594$), which takes each trajectory
from $600$ down to $100$ frames (time step $0.06$) and still covers almost the whole
simulation (from $0$ to about $5.94$). As with Airfoil, we only drop frames and do
not interpolate. From the inflow speed, the cylinder diameter, and the viscosity
$\nu=10^{-3}$, the Reynolds number of the trajectories ranges from about $56.7$ to
$466.6$, with a median near $237.7$. We use $1000$ trajectories for training, $100$
for validation, and $100$ for testing. The task is to predict $(u,v,p)$ for $25$
steps ahead on each sample's own mesh, with the node type and Reynolds number given
as extra inputs. Because the mesh changes between samples, this dataset tests
whether the model can handle different meshes without needing a shared set of nodes.

\begin{table}[t]
\centering
{\small
\setlength{\tabcolsep}{2.8mm}
\begin{tabular}{clccc}
\toprule
 & & \textbf{Structured Grid} & \multicolumn{2}{c}{\textbf{Unstructured Grid}} \\
\cmidrule(lr){3-3}\cmidrule(lr){4-5}
 & \textbf{Configuration} & \textbf{Supernova} & \textbf{ShapeNet-Car} & \textbf{Aircraft} \\
\midrule
\multirow{3}{*}{\rotatebox[origin=c]{90}{\textsc{Phys.}}}
 & Task   & 3D forecast & Steady flow & Steady flow \\
 & Input  & Past fields & Geometry & Geometry \\
 & Output & Future fields & $u,v,w,p$ & $C_p,\rho,u,v,w,p$ \\
\midrule
\multirow{7}{*}{\rotatebox[origin=c]{90}{\textsc{Data}}}
 & Geometry          & Cubic simulation domain & Car + wind-tunnel domain & Aircraft exterior \\
 & Prediction support & Volume & Volume + surface & Surface \\
 & Dimension         & 3D + time & 3D & 3D \\
 & Train / test      & 592 / 20 & 789 / 100 & 140 / 10 \\
 & Validation        & 74 & -- & -- \\
 & Input tensor      & $(10,32^3,6)$ & $(-,32186,7)$ & $(-,331971,6)$ \\
 & Output tensor     & $(10,32^3,6)$ & $(-,32186,4)$ & $(-,331971,6)$ \\
\bottomrule
\end{tabular}
}
\caption{Dataset details for the three 3D benchmarks, grouped by structured and
unstructured discretization. Prediction support states whether the target fields
occupy the volume, the surface, or both.}
\label{tab:app_data_3d}
\end{table}

\begin{table}[t]
\centering
{\fontsize{9.5}{11}\selectfont
\setlength{\tabcolsep}{1mm}
\begin{tabular}{@{}clcccccc@{}}
\toprule
 & & \multicolumn{5}{c}{\textbf{Regular Grid}} & \textbf{3D Volume} \\
\cmidrule(lr){3-7}\cmidrule(lr){8-8}
 & \textbf{Configuration} & \textbf{Allen--Cahn} & \textbf{Darcy} &
 \multicolumn{3}{c}{\textbf{Navier--Stokes}} & \textbf{Supernova} \\
\cmidrule(lr){5-7}
 & & & & $\nu{=}10^{-3}$ & $\nu{=}10^{-4}$ & $\nu{=}10^{-5}$ & \\
\midrule
\multirow{12}{*}{\rotatebox[origin=c]{90}{\textsc{Training}}}
 & Normalization & \multicolumn{5}{c}{Pointwise Gaussian} & $\log(1{+}x)$ \\
\cmidrule(lr){2-8}
 & Loss & $L_2{+}.1\nabla{+}.02\mathrm{TF}$ & $L_2{+}.1\nabla{+}.1\mathrm{TF}$ & $L_2{+}.1\mathrm{TF}$ & $L_2{+}.1(\nabla{+}\mathrm{TF})$ & $L_2{+}.1(\nabla{+}\mathrm{TF})$ & $L_2{+}.05\nabla{+}.1\mathrm{TF}$ \\
 & Epochs & 2000 & 500 & 600 & 700 & 500 & 300 \\
 & LR / WD & $10^{-3}/10^{-6}$ & $5{\times}10^{-4}/10^{-5}$ & $5{\times}10^{-4}/10^{-4}$ & $7{\times}10^{-4}/10^{-4}$ & $10^{-3}/10^{-4}$ & $10^{-3}/10^{-6}$ \\
 & Scale LR/WD & $10^{-4}/0$ & $5{\times}10^{-5}/0$ & $5{\times}10^{-5}/0$ & $10^{-4}/0$ & $10^{-4}/0$ & $10^{-4}/0$ \\
 & Optimizer & AdamW & AdamW & AdamW & AdamW & AdamW & AdamW \\
 & Batch & 20 & 4 & 8 & 8 & 20 & 20 \\
 & Scheduler & OneCycleLR & OneCycleLR & OneCycleLR & OneCycleLR & OneCycleLR & OneCycleLR \\
 & History $T_{\mathrm{in}}$ & -- & -- & 10 & 10 & 10 & 10 \\
 & Rollout $T_{\mathrm{roll}}$ & -- & -- & 10 & 10 & 10 & 10 \\
 & Step weight $w_{\mathrm{end}}$ & -- & -- & 4 & 5 & 5 & 1 \\
 & Noise Inj. & -- & -- & 0.005 & 0.01 & 0.05 & 0.01 \\
\midrule
\multirow{7}{*}{\rotatebox[origin=c]{90}{\textsc{Architecture}}}
 & Layers $L$ & 4 & 6 & 4 & 5 & 5 & 4 \\
 & $d_h$ & 128 & 128 & 192 & 256 & 192 & 64 \\
 & $J/M$ & 3/4 & 5/8 & 5/4 & 5/4 & 5/8 & 5/6 \\
 & $d_c/Q$ & 64/64 & 64/64 & 64/64 & 96/64 & 64/64 & 32/64 \\
 & $k$ & 8 & 24 & 8 & 8 & 8 & 8 \\
 & Hypergraph & Periodic & Euclidean & Periodic & Periodic & Periodic & Euclidean \\
 & Parameters & 575,629 & 1,797,023 & 1,397,653 & 3,371,290 & 2,155,994 & 267,610 \\
\bottomrule
\end{tabular}
}
\caption{HALO configurations for 2D and 3D regular-grid benchmarks. ``--''
denotes no temporal rollout; $w_{\mathrm{end}}$ is the endpoint of the linear
per-step ramp of Eq.~\eqref{eq:rollout_loss}, so $w_{\mathrm{end}}{=}1$ means
uniform weighting; Supernova is normalized by a $\log(1{+}x)$ transform
followed by a channel-wise $z$-score computed over $128$ trajectories.}
\label{tab:app_cfg_structured}
\end{table}

\begin{table}[t]
\centering
{\small
\begin{tabular}{@{}clcccc@{}}
\toprule
 & \textbf{Configuration} & \textbf{Airfoil} & \textbf{Cylinder} & \textbf{ShapeNet-Car} & \textbf{Aircraft} \\
\midrule
\multirow{12}{*}{\rotatebox[origin=c]{90}{\textsc{Training}}}
 & Normalization & \multicolumn{2}{c}{Nondim.\ + ch.~$z$-score} & \multicolumn{2}{c}{ch.~$z$-score} \\
\cmidrule(lr){2-6}
 & Loss & $L_2{+}.1\mathrm{TF}$ & $L_{\rm data}{+}\nabla{+}.1\mathrm{TF}$ & $L_{\mathrm{MSE}}{+}.1\mathrm{TF}$ & $L_{\rm bal}{+}.1\mathrm{TF}$ \\
 & Epochs & 500 & 600 & 500 & 200 \\
 & LR / WD & $10^{-3}/10^{-5}$ & $10^{-3}/10^{-6}$ & $10^{-3}/10^{-5}$ & $10^{-3}/10^{-5}$ \\
 & Scale LR/WD & $10^{-4}/0$ & $10^{-4}/0$ & $10^{-3}/0$ & $10^{-3}/0$ \\
 & Optimizer & AdamW & AdamW & AdamW & AdamW \\
 & Batch & 4 & 10 & 1 & 1 \\
 & Scheduler & OneCycleLR & OneCycleLR & OneCycleLR & OneCycleLR \\
 & History $T_{\mathrm{in}}$ & 15 & 15 & -- & -- \\
 & Rollout $T_{\mathrm{roll}}$ & 15 & 25 & -- & -- \\
 & Step weight $w_{\mathrm{end}}$ & 1 & 3 & -- & -- \\
 & Noise Inj. & 0 & 0.02 & -- & -- \\
\midrule
\multirow{6}{*}{\rotatebox[origin=c]{90}{\textsc{Architecture}}}
 & Layers $L$ & 4 & 4 & 6 & 6 \\
 & $d_h$ & 192 & 192 & 192 & 128 \\
 & $J/M$ & 5/4 & 5/4 & 5/8 & 5/8 \\
 & $d_c/Q$ & 64/64 & 64/64 & 96/64 & 64/64 \\
 & Hypergraph & ring-2 & ring-2 & ring-1 & ring-1 \\
 & Parameters & 1,449,560 & 1,457,367 & 4,116,130 & 1,847,716 \\
\bottomrule
\end{tabular}}
\caption{HALO training and architecture configurations for the unstructured 2D and 3D mesh
benchmarks. ``--'' denotes no temporal rollout; $w_{\mathrm{end}}$ is the
endpoint of the linear per-step ramp of Eq.~\eqref{eq:rollout_loss}, so
$w_{\mathrm{end}}{=}1$ means uniform weighting.
$L_{\rm data}=L_v+1.25L_p$ weights Cylinder's velocity and pressure
channels directly, in place of a single relative-$L_2$ term.
$L_{\rm bal}$ is the channel-balanced relative $L_2$ of
Eq.~\eqref{eq:app_aircraft_bal}, which normalizes each Aircraft field separately.}
\label{tab:app_cfg_unstructured}
\end{table}

\subsubsection*{Supernova}
Supernova is a 3D compressible-gas benchmark from the Well
collection~\citep{ohana2024well}. It models an explosion in a monatomic ideal
gas, governed by the conservation laws for mass, momentum, and energy with
self-gravity and radiative heating and cooling, closed by an ideal-gas equation
of state. We use the released data unchanged. It is
generated with the $N$-body/SPH code ASURA-FDPS, using a density-independent SPH
(DISPH) scheme so that the shock front of the supernova shell is resolved. The gas
is set to one solar metallicity to match the environment around the solar system,
which gives strong radiative cooling. Each sample stores pressure, density,
temperature, and the three velocity components---six fields in all---and we use
$592$ training, $74$ validation, and $20$ test samples. The dataset is released at
a resolution of $64^3$, which is the only size available. Because of compute and
memory limits, we subsample it to $32^3$ for training and keep the full $64^3$ grid
for zero-shot super-resolution. The task forecasts the six fields for ten steps
ahead.

\subsubsection*{ShapeNet-Car}
ShapeNet-Car is a 3D steady-state aerodynamics benchmark whose end goal is to
estimate the drag coefficient of a driving car, the quantity that governs
automotive shape design~\citep{wu2024transolver}. It contains $889$ distinct car
geometries drawn from the ``car'' category of ShapeNet, each placed in a
simulated wind tunnel at a constant inlet speed of $72\,\mathrm{km/h}$
(Figure~\ref{fig:app_shapenet_car}a); the
surrounding space is discretized into a single unstructured mesh of $32{,}186$
points that carries both the air velocity in the volume and the pressure over the
car surface. We adopt the standard split of $789$ meshes for training and $100$
for testing, and follow the released preprocessing in which each node is described
by its position, the signed distance to the surface, and the outward normal
vector. The operator maps this static geometry to the per-node velocity and
pressure fields, from which the drag coefficient is recovered by the surface
integral
\begin{equation}
C_D=\frac{2}{v^2 A}\!\left(
\int_{\partial\Omega} p(\xi)\,\bigl(\hat{n}(\xi)\!\cdot\!\hat{\imath}\bigr)\,d\xi
+\int_{\partial\Omega}\tau(\xi)\!\cdot\!\hat{\imath}\,d\xi\right),
\label{eq:app_shapenet_drag}
\end{equation}
where $v$ is the inlet speed, $\partial\Omega$ the car surface, $p$ the
pressure, $\hat{n}$ the outward unit normal, $\tau$ the wall shear stress,
$\hat{\imath}=(-1,0,0)$ the inlet-flow direction, and $A$ the reference area,
taken as the area of the smallest rectangle enclosing the front of the car; the
fluid is treated as unit density. Unlike the other
benchmarks this task has no temporal axis; performance is reported as the relative
$L_2$ error of the volume and surface fields, the relative error of the drag
coefficient $C_D$, and the Spearman rank correlation $\rho_D$ between the
predicted and reference drag rankings. The rank correlation measures how
faithfully a model orders competing designs. Following~\citet{wu2024transolver},
we define it as the Pearson correlation between the rank variables,
\begin{equation}
\rho_D=\frac{\operatorname{cov}\!\bigl(R(C_D),\,R(\widehat{C}_D)\bigr)}
            {\sigma_{R(C_D)}\,\sigma_{R(\widehat{C}_D)}},
\label{eq:app_shapenet_rho}
\end{equation}
where $C_D=\{C_D^{1},\dots,C_D^{N}\}$ and
$\widehat{C}_D=\{\widehat{C}_D^{1},\dots,\widehat{C}_D^{N}\}$ are the
ground-truth and predicted drag coefficients over the $N$ test samples, $R$ is
the ranking function, $\operatorname{cov}$ the covariance, and $\sigma$ the
standard deviation of the rank variables. We use this form rather than the
$1-6\sum_i d_i^2/\bigl(N(N^2-1)\bigr)$ shortcut, which is only equivalent when
no two coefficients share a rank, so that our numbers are computed exactly as in
the benchmarks we compare against (Table~\ref{tab:large_geom_results}).

\begin{figure}[H]
    \centering
    \includegraphics[width=0.7\linewidth,height=0.4\textheight,keepaspectratio]{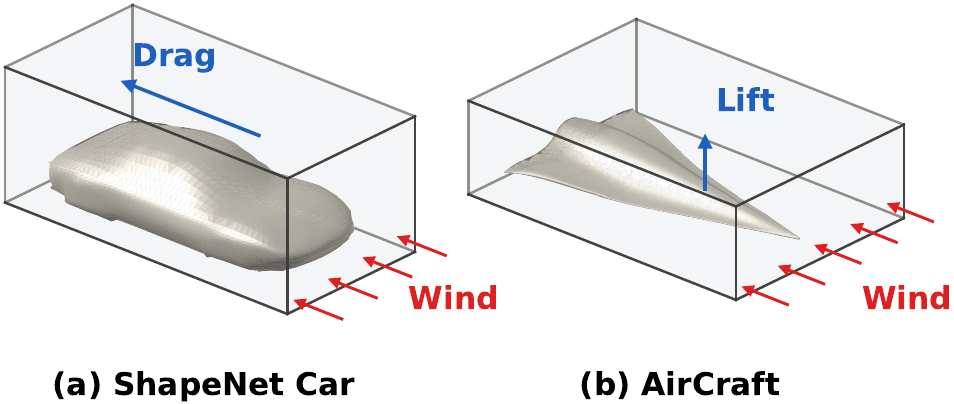}
    \caption{The two large 3D unstructured aerodynamic benchmarks, each immersed
    in an incoming wind: \textbf{(a)} ShapeNet-Car (drag) and \textbf{(b)}
    Aircraft (lift).}
    \label{fig:app_shapenet_car}
\end{figure}

\subsubsection*{Aircraft}
The Aircraft dataset~\citep{luo2025transolverpp} is a high-fidelity, large-scale
aerodynamics benchmark introduced alongside Transolver++, in which each geometry
is discretized into $331{,}971$ surface mesh points, roughly an order of
magnitude denser than ShapeNet-Car. It comprises over $30$ aircraft designs
simulated under $5$ incoming flow conditions that vary the Mach number, angle of
attack, and sideslip angle (Figure~\ref{fig:app_shapenet_car}b), with the CFD solutions produced by aerodynamicists at
an aircraft-design institution. We follow the released split of $140$ training and
$10$ test cases. Each node is described by its coordinates $(x,y,z)$ and outward
surface normal, and the reference solution provides six physical quantities: the
pressure coefficient $C_p$, density $\rho$, the velocity components $(u,v,w)$, and
pressure $p$. Performance is reported as the relative $L_2$ error over all six
quantities (Field) and over the surface pressure alone (Surf.), together with the
relative error of the lift coefficient $C_l$ and the coefficient of determination
$R_l^2$ between the predicted and reference lift across the test cases,
\begin{equation}
R_l^2
= 1 - \frac{\sum_{i=1}^{N}\bigl(C_l^{i}-\hat{C}_l^{i}\bigr)^2}
            {\sum_{i=1}^{N}\bigl(C_l^{i}-\bar{C}_l\bigr)^2},
\label{eq:app_aircraft_r2}
\end{equation}
where $\hat{C}_l^{i}$ is the predicted lift coefficient of case $i$ and
$\bar{C}_l$ the mean reference lift. As with drag in
Eq.~\eqref{eq:app_shapenet_drag}, $C_l$ follows from integrating the predicted
surface pressure and shear stress, here along the lift direction, so it tests
whether a model resolves the surface field accurately enough to recover the
design quantity itself (Table~\ref{tab:large_geom_results}).
At this resolution the benchmark bridges standard academic tasks and
industrial million-scale geometries, stressing a model's ability to build and
process very large meshes.


\section{Training Details}
\label{sec:training}
Throughout, we report the mean relative $L_2$ error~\cite{li2021fourier},
$\|\mathcal{N}_\theta(a^{(i)})-u^{(i)}\|_2/\|u^{(i)}\|_2$, averaged over the test
set. All models minimize Eq.~\eqref{eq:total_loss}, abbreviated in the
configuration tables as
$\mathcal{L}=\mathcal{L}_{\rm data}+\alpha_{\rm grad}\mathcal{L}_{\rm grad}
+\beta_{\rm TF}\mathcal{L}_{\rm TF}$, with $\nabla$ denoting the gradient penalty
and TF the tight-frame regularizer. Time-dependent tasks use full autoregressive
rollout training: each prediction is fed back into the sliding window of
Eq.~\eqref{eq:window_shift} and the loss is accumulated over the complete
horizon $T_{\mathrm{roll}}$ listed in Tables~\ref{tab:app_cfg_structured}
and~\ref{tab:app_cfg_unstructured}. The per-step weights $w_r$ of
Eq.~\eqref{eq:rollout_loss} ramp linearly along that horizon, from $w_1 = 1$ at
the first step to the endpoint $w_{\mathrm{end}}$ tabulated for each benchmark;
they are fixed for the whole of training rather than annealed over epochs.
Navier--Stokes uses $w_{\mathrm{end}}=4$ at $\nu=10^{-3}$ and
$w_{\mathrm{end}}=5$ at $\nu=10^{-4}$ and $10^{-5}$, and Cylinder uses
$w_{\mathrm{end}}=3$ across its $25$ steps, while Airfoil and Supernova weight
their steps uniformly ($w_{\mathrm{end}}=1$). Cylinder is also the one benchmark
whose horizon exceeds its history, $T_{\mathrm{roll}}=25$ against
$T_{\mathrm{in}}=15$, so by construction its final ten predictions are made from
a window containing no observed snapshot at all. We additionally inject Gaussian
noise into the fed-back history at the standard deviation listed in
Tables~\ref{tab:app_cfg_structured} and~\ref{tab:app_cfg_unstructured}
(pushforward regularization~\cite{brandstetter2022message}), which exposes the
model to the off-manifold distribution it meets at inference time. The
perturbation applies only from the second rollout step onward, so the first
prediction of every trajectory is made from clean observed data, and the loss is
always evaluated against the clean target throughout. Airfoil is the one rollout benchmark trained without
it, at the noise level of zero listed in
Table~\ref{tab:app_cfg_unstructured}. Trainable dyadic scales use a separate parameter
group with one-tenth of the main learning rate and zero weight decay. Every run
fixes the random seed for weight initialization, data shuffling, and any
stochastic regularization. Experiments use two NVIDIA A40 48\,GB GPUs.
Algorithm~\ref{alg:halo_full} gives the complete HALO training and inference
procedure; the per-layer AWB computation it invokes is specified in
Eqs.~\eqref{eq:cheby_recurrence}--\eqref{eq:wib_output}.

\begin{algorithm}[htbp]
\caption{HALO: Training and Inference}
\label{alg:halo_full}
\begin{algorithmic}[1]
\fontsize{10}{11.5}\selectfont
\STATE \textbf{--- Preprocessing (once) ---}
\REQUIRE Training set $\{(a^{(i)}, u^{(i)})\}_{i=1}^{N}$, coordinates $\{x_i\}$, neighborhood size $k$, Chebyshev order $M$, scales $J$, AWB layers $L$
\STATE Build hypergraph: $e_i \leftarrow \{i\} \cup \mathcal{N}_k(i)$ for each node \hfill $\triangleright$ $k$-NN or $k$-ring
\STATE Compute $\Delta = I - D_v^{-1/2}HW_eD_e^{-1}H^{\top}D_v^{-1/2}$
\STATE Estimate $\lambda_{\max}$ by sparse power iteration \hfill Eq.~\eqref{eq:app_power_iter}
\STATE $\tilde{\Delta} \leftarrow 2\Delta/\lambda_{\max} - I$ \hfill $\triangleright$ rescale spectrum to $[-1,1]$
\STATE Fix the surrogate grid $\{\lambda_r\}_{r=1}^{R}$, uniform on $[10^{-3}\lambda_{\max},\lambda_{\max}]$, $R{=}64$ \hfill $\triangleright$ no eigendecomposition
\STATE Initialize scales: $s_j^{\mathrm{init}} \leftarrow x^{*}\,2^{j}/\lambda_{\max}$, \quad $\rho_j^{(\ell)} \leftarrow \mathrm{softplus}^{-1}(s_j^{\mathrm{init}})$ for all $\ell,j$
\STATE
\STATE \textbf{--- Training loop ---}
\FOR{each epoch}
  \FOR{each mini-batch $\{(a, u)\}$}
    \STATE \textbf{Uplift:} $\mathbf{h}^{(0)} \leftarrow \mathrm{GELU}([a(x_i);\, x_i;\, z_i]^{\top} W_{\mathrm{in}} + b_{\mathrm{in}})$ \hfill $\mathbb{R}^{d_{\mathrm{in}}} \!\to\! \mathbb{R}^{d_h}$
    \FOR{$\ell = 1$ \TO $L$}
      \STATE $s_j^{(\ell)} \leftarrow \mathrm{softplus}(\rho_j^{(\ell)})$ for $j=0,\dots,J{-}1$
      \STATE Shared basis: $\mathbf{T}_0\!=\!\mathbf{h}^{(\ell-1)}$, $\mathbf{T}_1\!=\!\tilde{\Delta}\mathbf{h}^{(\ell-1)}$, $\mathbf{T}_m\!=\!2\tilde{\Delta}\mathbf{T}_{m-1}-\mathbf{T}_{m-2}$ \hfill $m=2,\dots,M$
      \FOR{$j = 0$ \TO $J{-}1$}
        \STATE Coefficients $\{c_{s_j,m}\}$ via Chebyshev--Gauss quadrature \hfill Eq.~\eqref{eq:quadrature_coeff}
        \STATE $\Psi_j^{\mathrm{base}} \leftarrow \tfrac{1}{2}c_{s_j,0}\mathbf{T}_0 + \sum_{m=1}^{M} c_{s_j,m}\mathbf{T}_m$
        \STATE $\Psi_j^{\mathrm{delta}} \leftarrow \tfrac{1}{2}c_{s_j,0}(\mathbf{T}_0 \mathbf{P}_{\downarrow})\Theta_{j,0}^{(\ell)}\mathbf{P}_{\uparrow} + \sum_{m=1}^{M} c_{s_j,m}(\mathbf{T}_m \mathbf{P}_{\downarrow})\Theta_{j,m}^{(\ell)}\mathbf{P}_{\uparrow}$
        \STATE $\Psi_j \leftarrow \Psi_j^{\mathrm{base}} + \Psi_j^{\mathrm{delta}}$
      \ENDFOR
      \STATE $\mathbf{V}^{(\ell)} \leftarrow [\Psi_0\!\parallel\!\cdots\!\parallel\!\Psi_{J-1}]\,\mathbf{W}_{\mathrm{mix}}^{(\ell)}$ \hfill cross-scale mixing
      \STATE $\mathbf{h}^{(\ell)} \leftarrow \mathrm{GELU}\!\bigl(\mathrm{LN}(\mathbf{V}^{(\ell)} + \mathbf{h}^{(\ell-1)} W_{\mathrm{skip}})\bigr)$
    \ENDFOR
    \STATE \textbf{Output:} $\hat{u} \leftarrow (\mathbf{h}^{(L)} + \mathbf{h}^{(0)})\,W_{\mathrm{out}}$ \hfill global residual
    \STATE
    \STATE \textit{--- Autoregressive rollout (time-dependent only) ---}
    \STATE $X^{(1)} \leftarrow [\,u^{(1-T_{\mathrm{in}})};\dots;u^{(0)}\,]$ \hfill ground-truth seed window, width $T_{\mathrm{in}}$
    \FOR{$r = 1$ \TO $T_{\mathrm{roll}}$}
      \STATE $\hat{u}^{(r)} \leftarrow$ uplift, $L$ AWB layers, and output projection applied to $[X^{(r)};\, x_i;\, z_i]$
      \STATE $\hat{u}^{(r)}_{\mathrm{noisy}} \leftarrow \hat{u}^{(r)} + \epsilon$, \quad $\epsilon \sim \mathcal{N}(0,\, \sigma_{\mathrm{noise}}^2 I)$, \quad $\epsilon\!=\!0$ for $r\!=\!1$ \hfill pushforward noise; first step is noise-free
      \STATE $X^{(r+1)} \leftarrow [\,X^{(r)}_{2:T_{\mathrm{in}}};\, \hat{u}^{(r)}_{\mathrm{noisy}}\,]$ \hfill drop oldest, append prediction \quad Eq.~\eqref{eq:window_shift}
    \ENDFOR
    \STATE
    \STATE \textit{--- Loss computation ---}
    \STATE $\mathcal{L}_{\mathrm{data}} \leftarrow \sum_{r} w_r\,\| \hat{u}^{(r)} - u^{(r)} \|_2 / \| u^{(r)} \|_2$ \hfill step-weighted relative $L_2$; single term with $w\!=\!1$ if steady \quad Eq.~\eqref{eq:rollout_loss}
    \STATE $\mathcal{L}_{\mathrm{grad}} \leftarrow$ finite-difference gradient penalty
    \STATE $\mathcal{L}_{\mathrm{TF}} \leftarrow \frac{1}{L}\sum_{\ell=1}^{L} \mathrm{Var}_{r=1,\dots,R}\bigl[\sum_j g(s_j^{(\ell)}\lambda_r)^2\bigr]$ \hfill tight-frame loss on the fixed grid
    \STATE $\mathcal{L} \leftarrow \mathcal{L}_{\mathrm{data}} + \alpha_{\mathrm{grad}}\,\mathcal{L}_{\mathrm{grad}} + \beta_{\mathrm{TF}}\,\mathcal{L}_{\mathrm{TF}}$
    \STATE Backpropagate $\nabla_{\theta,\,\{\rho_j\}}\,\mathcal{L}$ and update
  \ENDFOR
\ENDFOR
\STATE
\STATE \textbf{--- Inference ---}
\STATE Given $a$, run the uplift, the $L$ AWB layers, and the output projection; for time-dependent tasks iterate the rollout noise-free ($\epsilon\!=\!0$) for $T_{\mathrm{roll}}$ steps. No loss is evaluated.
\end{algorithmic}
\end{algorithm}

\subsubsection*{Hypergraph Construction}
Structured domains use $k$-nearest-neighbor hyperedges, made periodic where the
boundary conditions are, while the mesh benchmarks use ring-based mesh
neighborhoods: two-ring for Airfoil and Cylinder, and one-ring for the denser
ShapeNet-Car and Aircraft meshes. Algorithms~\ref{alg:knn}--\ref{alg:edgering}
give the three neighbor rules and the weighting step they share, which together
produce the operator of Eq.~\eqref{eq:hg_laplacian}. Allen--Cahn, Darcy,
Navier--Stokes, and Supernova use the $k$-NN rule ($k=8$, or $24$ for Darcy);
Airfoil and Cylinder use the cell-based two-ring rule; ShapeNet-Car and Aircraft
receive their mesh as an edge list and use the one-ring rule of
Algorithm~\ref{alg:edgering}.
\begin{algorithm}[htbp]
\caption{$k$-nearest-neighbor hyperedges. One hyperedge per node,
$e_i=\{i\}\cup\mathcal{N}_k(i)$. Squared distances are accumulated one axis at a
time: on a regular lattice many neighbors are exactly equidistant, and a fused
reduction breaks those ties differently, which changes the graph and hence
$\lambda_{\max}$.}
\label{alg:knn}
\begin{algorithmic}[1]
\fontsize{10}{11.5}\selectfont
\REQUIRE Coordinates $X\in\mathbb{R}^{N\times d}$, neighbors $k$, period
$P\in\mathbb{R}^{d}$ or $\varnothing$, block size $B$
\ENSURE Hyperedges $\mathcal{E}=(e_1,\dots,e_N)$, each anchor-first
\STATE $k \leftarrow \min(k,\,N-1)$; \quad $\mathcal{E} \leftarrow ()$
\FOR{$s = 0$ \TO $N-1$ in steps of $B$}
  \STATE $t \leftarrow \min(s+B,\,N)$; \quad
         $D \leftarrow \mathbf{0}\in\mathbb{R}^{(t-s)\times N}$
         \hfill $\triangleright$ row block, memory $\mathcal{O}(BN)$
  \FOR{$a = 1$ \TO $d$}
    \STATE $\delta \leftarrow \lvert X_{s:t,a}\otimes\mathbf{1}
           - \mathbf{1}\otimes X_{:,a}\rvert$
    \IF{$P \neq \varnothing$}
      \STATE $\delta \leftarrow \min(\delta,\,P_a-\delta)$
             \hfill $\triangleright$ periodic wrap
    \ENDIF
    \STATE $D \leftarrow D + \delta^{2}$
  \ENDFOR
  \STATE $D_{i-s,\,i} \leftarrow \infty$ for $i=s,\dots,t-1$
         \hfill $\triangleright$ exclude the anchor itself
  \FOR{$i = s$ \TO $t-1$}
    \STATE $\mathcal{N}_i \leftarrow$ indices of the $k$ smallest entries of
           $D_{i-s,\,:}$
    \STATE Append $e_i \leftarrow (i,\,\mathcal{N}_i)$ to $\mathcal{E}$
  \ENDFOR
\ENDFOR
\RETURN $\mathcal{E}$
\end{algorithmic}
\end{algorithm}

\begin{algorithm}[htbp]
\caption{Mesh $k$-ring hyperedges. Every cell is a hyperedge, and around every
node the patch reachable within $r$ mesh hops. Rings follow mesh connectivity,
so they never join nodes that are spatially close but geodesically distant, such
as the two sides of a thin trailing edge.}
\label{alg:kring}
\begin{algorithmic}[1]
\fontsize{10}{11.5}\selectfont
\REQUIRE Cells $\mathcal{C}$, coordinates $X$, ring radius $r$, flag
\textsc{IncludeCells}
\ENSURE Hyperedges $\mathcal{E}$ with centers $c$
\STATE $\mathrm{nbr}[v] \leftarrow \varnothing$ for all $v$
\FORALL{$\mathrm{cell}\in\mathcal{C}$}
  \FORALL{ordered pairs $(a,b)$ of distinct vertices of $\mathrm{cell}$}
    \STATE $\mathrm{nbr}[a] \leftarrow \mathrm{nbr}[a]\cup\{b\}$
           \hfill $\triangleright$ mesh adjacency from cell incidence
  \ENDFOR
\ENDFOR
\STATE $\mathcal{E} \leftarrow ()$; \quad $c \leftarrow ()$
\IF{\textsc{IncludeCells}}
  \FORALL{$\mathrm{cell}\in\mathcal{C}$}
    \STATE Append $\mathrm{vertices}(\mathrm{cell})$ to $\mathcal{E}$; \quad
           append $\mathrm{centroid}(X_{\mathrm{cell}})$ to $c$
  \ENDFOR
\ENDIF
\FOR{$j = 0$ \TO $N-1$}
  \STATE $S \leftarrow \{j\}$; \quad $F \leftarrow \{j\}$
  \FOR{$1$ \TO $r$}
    \STATE $F \leftarrow \bigl(\bigcup_{u\in F}\mathrm{nbr}[u]\bigr)\setminus S$;
           \quad $S \leftarrow S\cup F$
  \ENDFOR
  \STATE Append $\mathrm{sorted}(S)$ to $\mathcal{E}$; \quad append $X_j$ to $c$
         \hfill $\triangleright$ center is the seed $j$
\ENDFOR
\RETURN $\mathcal{E},\,c$
\end{algorithmic}
\end{algorithm}

\begin{algorithm}[htbp]
\caption{Edge-ring hyperedges, for meshes supplied as an edge list rather than
as cells (ShapeNet-Car, Aircraft). Rings grow by sparse boolean adjacency
products; only the sparsity pattern is kept, since membership is reachability
within $r$ hops and not the number of walks that reach a node.}
\label{alg:edgering}
\begin{algorithmic}[1]
\fontsize{10}{11.5}\selectfont
\REQUIRE Mesh edges $(\mathrm{src},\mathrm{dst})$, node count $N$, ring radius
$r$
\ENSURE Membership pairs $(\mathrm{row},\mathrm{col})$, one hyperedge per anchor
\STATE Drop self-loops from $(\mathrm{src},\mathrm{dst})$
\STATE $\mathrm{row} \leftarrow [0..N{-}1]\,\Vert\,\mathrm{src}$; \quad
       $\mathrm{col} \leftarrow [0..N{-}1]\,\Vert\,\mathrm{dst}$
       \hfill $\triangleright$ one-ring membership, anchors included
\STATE $(\mathrm{row},\mathrm{col}) \leftarrow
       \mathrm{unique}(\mathrm{row},\mathrm{col})$
\IF{$r > 1$}
  \STATE $A \leftarrow$ sparse adjacency of $(\mathrm{src},\mathrm{dst})$ with
         self-loops
  \STATE $\mathcal{M} \leftarrow$ sparse indicator of
         $(\mathrm{row},\mathrm{col})$
  \FOR{$1$ \TO $r-1$}
    \STATE $\mathcal{M} \leftarrow \mathrm{pattern}(A\,\mathcal{M})$
           \hfill $\triangleright$ one more hop, values discarded
  \ENDFOR
  \STATE $(\mathrm{row},\mathrm{col}) \leftarrow
         \mathrm{nonzeros}(\mathcal{M})$
\ENDIF
\RETURN $(\mathrm{row},\mathrm{col})$
\end{algorithmic}
\end{algorithm}

\begin{algorithm}[htbp]
\caption{Hyperedge weighting and the normalized hypergraph Laplacian of
Eq.~\eqref{eq:hg_laplacian}, shared by all three neighbor rules and run on every
benchmark. A hyperedge weighs near one when its members cluster tightly about
its center and less when they are diffuse. The final \textsc{return} is the only
assembly step: ShapeNet-Car and Aircraft stop before it and keep $S$, $H$, $Q$
factored, for the memory reason given in
Appendix~\ref{app:factored_laplacian}.}
\label{alg:laplacian}
\begin{algorithmic}[1]
\fontsize{10}{11.5}\selectfont
\REQUIRE Hyperedges $\mathcal{E}$, centers $c$, coordinates $X$, period $P$,
floor $\varepsilon$
\ENSURE $\Delta = I - D_v^{-1/2}HW_eD_e^{-1}H^{\top}D_v^{-1/2}$
\STATE $H\in\{0,1\}^{N\times M}$ with $H_{v,e}=1 \iff v\in e$
       \hfill $\triangleright$ incidence
\FORALL{$e\in\mathcal{E}$}
  \STATE $D_e \leftarrow \max(\lvert e\rvert,\,1)$
  \STATE $d_v \leftarrow \mathrm{dist}_P(x_v,\,c_e)$ for $v\in e$
         \hfill $\triangleright$ wraps if $P\neq\varnothing$
  \STATE $\sigma_e \leftarrow \tfrac{1}{D_e}\sum_{v\in e} d_v$
  \STATE $w_e \leftarrow \tfrac{1}{D_e}\sum_{v\in e}
         \exp\bigl(-d_v^{2}/(\sigma_e^{2}+\varepsilon)\bigr)$
\ENDFOR
\STATE $D_v(i) \leftarrow \sum_{e\ni i} w_e$
       \hfill $\triangleright$ weighted node degree
\STATE $S \leftarrow \mathrm{diag}\bigl(1/\sqrt{D_v+\varepsilon}\bigr)$; \quad
       $Q \leftarrow \mathrm{diag}(w_e/D_e)$
\RETURN $\Delta \leftarrow I - SHQH^{\top}S$
\end{algorithmic}
\end{algorithm}

Each ring is stored sorted, so its first member is its lowest-numbered node
rather than the node it was grown from. Passing the centers $c$ explicitly is
therefore what keeps every ring anchored on its seed; the weighting step falls
back to treating the first member as the anchor only when no centers are
supplied. A cell's center is the centroid of its vertices and is not a member
index at all. The two families give $M=\lvert\mathcal{C}\rvert+N$ hyperedges:
on the Airfoil mesh, $7{,}685$ cells plus $3{,}982$ rings, or $11{,}667$ in
total.

The axis-by-axis accumulation in Algorithm~\ref{alg:knn} is a correctness
requirement rather than an optimization. On a regular lattice many candidates
are exactly equidistant---on a $32^3$ grid with $k=8$, six face neighbors are
joined by twelve equidistant edge neighbors competing for the last two
slots---and the reduction order fixes which ties win. Blocking the rows bounds
memory to $\mathcal{O}(BN)$: the full $N\times N$ distance matrix is never
formed, which at $N=331{,}971$ would require $441$\,GB.

The final line of Algorithm~\ref{alg:laplacian} is the only place the operator
is assembled, and on ShapeNet-Car and Aircraft it is not executed: $H$, $S$ and
$Q$ are kept and applied through Eq.~\eqref{eq:app_factored_apply}, for the
reasons set out in Appendix~\ref{app:factored_laplacian}. Algorithm~\ref{alg:knn}
costs $\Theta(N^2 d)$ time, which is why the $k$-NN rule is used only on the
grids and the $32^3$ Supernova cube; the two mesh rules cost
$\mathcal{O}(N\bar{d}^{\,r})$ and
$\mathcal{O}(r\,\mathrm{nnz}(A)\,\bar{d})$ respectively, with $\bar{d}$ the mean
mesh degree, and are used on everything larger.

\FloatBarrier

\subsubsection*{Input Normalization}
Every field is standardized before it enters the network, and the prediction is
mapped back to physical units at the output; the ``Normalization'' rows of
Tables~\ref{tab:app_cfg_structured} and~\ref{tab:app_cfg_unstructured} name the
scheme used for each benchmark. On the regular grids (Allen--Cahn, Darcy,
Navier--Stokes) we use a pointwise Gaussian normalization: each grid location $i$
is centered and scaled by its own mean $\mu_i$ and standard deviation $\sigma_i$
estimated on the training set,
\begin{equation}
\hat{x}_i=\frac{x_i-\mu_i}{\sigma_i+\varepsilon},
\label{eq:app_norm_zscore}
\end{equation}
with a small $\varepsilon$ for numerical stability. Allen--Cahn and Darcy map
between different physical quantities, so their input and output fields use
separate statistics, whereas Navier--Stokes maps the vorticity field to itself and
shares a single joint statistic. Airfoil and Cylinder apply the same per-channel
$z$-score after a physical nondimensionalization, with $(\mu,\sigma)$ pooled over
all nodes and time frames of the training set.

Supernova needs one extra step, which is what the ``$\log(1{+}x)$'' entry in
Table~\ref{tab:app_cfg_structured} stands for. Its six fields (density, pressure,
temperature, and the three velocity components) span several orders of magnitude
and are strongly heavy-tailed at the shock front, so a plain $z$-score is swamped
by the extremes. We therefore first compress each field with a signed logarithm
and then $z$-score it per channel,
\begin{equation}
\tilde{x}=\operatorname{sign}(x)\,\log\!\bigl(1+|x|\bigr),
\qquad
\hat{x}=\frac{\tilde{x}-\mu}{\sigma+\varepsilon},
\label{eq:app_norm_log}
\end{equation}
where $\mu$ and $\sigma$ are per-channel statistics of the transformed field,
estimated over $128$ training trajectories. This map is almost linear for small
$|x|$ and logarithmic in the tails, which keeps the quiet ambient gas and the
violent shock on a comparable numeric scale.

\subsubsection*{Field-Aware Targets}
For the multi-field datasets, HALO uses a field-aware design that preserves the
identity of each physical quantity while modelling the fields jointly. On the
time-dependent mesh tasks (Airfoil, Cylinder) it applies this on both sides of the
network: separate per-field input projections map each quantity's history channels
into the shared trunk, and separate per-field output heads read the fields back
out. Concretely, the $c$ channels of the observation in
Eq.~\eqref{eq:input_window} are partitioned into $F$ groups of $c_f$ channels
each, with $\sum_{f=1}^{F} c_f = c$, and group $f$ is uplifted by its own
projection over $d_{\text{in}}^{(f)} = T_{\text{in}}\,c_f + d + d_{\text{cond}}$
inputs before the groups meet in the shared trunk. The coordinate and
conditioning channels are replicated across groups, so every projection sees its
own field history in full geometric context; $F = 1$ recovers the single shared
uplift used in the main text. Airfoil forecasts $(u,v,p,\rho)$ over 15 steps; Cylinder forecasts
$(u,v,p)$ over 25 steps; and Supernova forecasts density, pressure, temperature,
and three velocity components. For Airfoil and Supernova,
$\mathcal{L}_{\rm data}$ is field-wise relative $L_2$ error aggregated over
fields and rollout steps. Cylinder uses
$\mathcal{L}_{\rm data}=L_v+1.25L_p$, where $L_v$ is the velocity-pair error
and $L_p$ is the pressure error. ShapeNet-Car applies the field-aware split on
the decoder only. Because it is a
boundary-value problem (geometry $\mapsto$ field) with no velocity or pressure
history to consume, its encoder is a single shared input projection over the seven
geometry channels (node position, signed distance, and outward normal), while its
decoder keeps separate heads for the three velocity components and the pressure.
Its data term $\mathcal{L}_{\rm data}$ is a pointwise mean-squared-error (MSE) loss
rather than a relative-$L_2$ term. Its reported errors
(Table~\ref{tab:large_geom_results}) nonetheless use the relative-$L_2$ metric, so the
training objective differs from the evaluation metric but the numbers stay
comparable with the baselines.

Aircraft closes that gap: its objective \emph{is} the evaluation metric. With
$y,\hat{y}\in\mathbb{R}^{N\times 6}$ the reference and predicted surface fields
over the $N$ mesh points, let
$r_c = \lVert\hat{y}_c-y_c\rVert_2 / \lVert y_c\rVert_2$ be the per-channel
relative $L_2$ error in physical units. We train on
\begin{equation}
\mathcal{L}_{\rm bal}
= \frac{1}{6}\!\!\sum_{c\in\{C_p,\rho,u,v,w,p\}}\!\!r_c
\label{eq:app_aircraft_bal}
\end{equation}
plus $0.1\,\mathcal{L}_{\rm TF}$, with pressure from an independent head so the
$C_p$ identity $p=\tfrac{1}{2}M^2C_p+\tfrac{5}{7}$ is not enforced. The
per-channel denominators are what make this work. Sharing a single denominator
across the five auxiliary channels instead lets the streamwise velocity dominate
in physical units, at roughly $87\%$ of the auxiliary energy, which leaves $v$
and $w$ under $1\%$ of the auxiliary gradient each. In our runs that variant
ended with per-channel errors of $0.222$ and $0.183$ on those two channels,
against $0.106$ and $0.117$ under Eq.~\eqref{eq:app_aircraft_bal}. Because
$\mathcal{L}_{\rm bal}$ is evaluated in physical space, the reported six-field
error is exactly the quantity optimized.

\subsubsection*{Conditioning Variables}
Conditioning enters through the channels $z_i$ of
Eq.~\eqref{eq:input_window} wherever a benchmark supplies it. Airfoil appends the
freestream Mach number and angle of attack ($d_{\text{cond}}=2$), both computed
from the inflow nodes of the initial state and broadcast unchanged to every
node; Cylinder appends a seven-way node-type one-hot together with a $z$-scored
Reynolds number ($d_{\text{cond}}=8$), the first per-node and the second global
to the case. At $T_{\text{in}}=15$ this puts the per-group encoder widths at
$34/19/19$ for Airfoil's $(u,v)$, $p$, and $\rho$ groups and $40/25$ for
Cylinder's $(u,v)$ and $p$ groups. Aircraft conditions on its flight state---Mach
number, angle of attack, and sideslip angle---supplied alongside its six
geometry channels. Darcy, Allen--Cahn, Navier--Stokes, and ShapeNet-Car use no
conditioning.

\subsection{Sliced-ELLPACK Sparse Backend}
\label{app:sell_backend}

The shared Chebyshev recurrence in Eq.~\eqref{eq:cheby_recurrence} applies the
sparse operator $M$ times per AWB in the forward pass and again in the backward
pass. With $L$ blocks and an $R$-step rollout, this gives $LMR$ sparse
matrix--dense matrix products (SpMMs) in each direction. We fold batch and
feature dimensions into a single dense width $K=Bd_h$, rather than issuing one
product per sample, so the training operands are wide ($K=768$ for Airfoil and
$K=3{,}840$ for the longest Navier--Stokes configuration). The sparse storage
format therefore affects wall-clock time even though it does not change the
operator or its asymptotic cost.

We store $\widetilde{\Delta}$ in Sliced-ELLPACK without row sorting
(SELL-C-1). Consecutive rows are partitioned into slices of height
$C\in\{16,32\}$, selected from the operand width, and each slice $s$ is padded
only to its local maximum row length $W_s$. If $o_s$ is the slice offset, slot
$q$ of row $r$ is stored at
\begin{equation}
o_s+qC+r,
\qquad 0\leq r<C,\quad 0\leq q<W_s,
\label{eq:app_sell_layout}
\end{equation}
so adjacent GPU lanes read adjacent values and 32-bit column indices. The
hypergraphs used here are well suited to this organization: the fine Airfoil
operator has $57.8\pm8.8$ nonzeros per row and a $1.10\times$ padding ratio,
while the structured-grid operators pad by only $0.5$--$2.2\%$. We deliberately
retain mesh order ($\sigma=1$): sorting the Cylinder rows reduces padding but
degrades the gathered feature locality, increasing kernel time from
$0.467$\,ms to $0.554$\,ms. Highly irregular, small operators can be less
favorable to SELL; they are handled by the width-aware CSR fallback.

The custom Triton kernel assigns one program to a slice and a block of dense
columns, accumulates each output tile in registers, and writes it once, avoiding
the atomics of COO and the variable-length per-row reduction of CSR. Conversion
is performed lazily and cached. Since $\widetilde{\Delta}$ is symmetric, its
backward product is the same SpMM and requires neither a transposed copy nor a
second sparse layout. Automatic dispatch selects SELL on CUDA for $K\geq192$
and CSR otherwise. Every training configuration here is at or above that
crossover except Aircraft, whose $K=Bd_h=128$ falls below it; its row in
Table~\ref{tab:app_sell_epochs} therefore requests the SELL backend explicitly,
and the two backends give identical training curves to the printed digits. PyTorch and cuSPARSE do not provide a SELL SpMM path for this
workload, which is why the layout is paired with a custom kernel.

On the real Airfoil operator ($N=3{,}982$,
$\mathrm{nnz}=230{,}300$, dense operand $3982\times768$), back-to-back A40
kernel times are $1.799$\,ms for COO, $0.344$\,ms for cuSPARSE CSR, and
$0.171$\,ms for SELL-C-1. The corresponding operator storage is $4.39$,
$2.67$, and $1.94$\,MiB: narrowing the surviving column-index stream to 32 bits
more than offsets the slice padding. Table~\ref{tab:app_sell_epochs} reports the
end-to-end effect under the best-model configurations.

\begin{table}[H]
\centering
{\small
\renewcommand{\arraystretch}{1.08}
\begin{tabular}{@{}lrrr@{}}
\toprule
\textbf{Dataset} & \textbf{COO (s)} & \textbf{SELL-C-1 (s)} & \textbf{Speedup} \\
\midrule
Allen--Cahn ($64^2$) & 15.11 & 14.36 & $1.05\times$ \\
Darcy ($85^2$) & 135.71 & 112.86 & $1.20\times$ \\
Navier--Stokes ($\nu=10^{-5}$, $64^2$) & 459.12 & 434.88 & $1.06\times$ \\
Airfoil & 444.78 & \textbf{243.08} & $\mathbf{1.83\times}$ \\
Cylinder & 485.55 & \textbf{247.05} & $\mathbf{1.97\times}$ \\
Aircraft & 222.60 & 192.13 & $1.16\times$ \\
ShapeNet-Car ($32$k nodes) & 308.50 & 253.46 & $1.22\times$ \\
\midrule
\textbf{Suite total} & \textbf{2{,}071.4} & \textbf{1{,}497.8} & $\mathbf{1.38\times}$ \\
\bottomrule
\end{tabular}
}
\caption{Mean epoch time for the coordinate and SELL-C-1 backends. Each entry
is a two-epoch A/B measurement on an NVIDIA A40 at the dataset's best-model
configuration, with identical seed, batch order, and learning-rate schedule;
the sparse format is the only change. The gain is largest on the denser
unstructured-mesh operators.}
\label{tab:app_sell_epochs}
\end{table}

This substitution is an execution optimization, not a model change. On a
converged Cylinder checkpoint, COO, CSR, and SELL give raw-state relative
$L_2$ errors of $0.0250747$, $0.0250765$, and $0.0250790$, respectively, and
the physical velocity and pressure errors agree to five significant figures.
Direct checks on the real Laplacian give maximum forward and backward absolute
differences of $3.3\times10^{-6}$ and $2.4\times10^{-6}$ from cuSPARSE. Thus,
SELL-C-1 removes $28\%$ of the summed one-epoch suite time while preserving the
reported numerical accuracy.

\subsection{Factored Hypergraph Laplacian}
\label{app:factored_laplacian}

The Laplacian of Eq.~\eqref{eq:hg_laplacian} is a triple product,
$\Delta = I - SHQH^\top S$, in which the only non-diagonal factor is the
incidence matrix $H$; the node scaling $S=D_v^{-1/2}$ and the hyperedge scaling
$Q=W_eD_e^{-1}$ are diagonal. Assembling $\Delta$ evaluates $HQH^\top$, whose
sparsity pattern couples every pair of nodes that share a hyperedge, so a
one-hop incidence structure is replaced by a two-hop one: a hyperedge with $m$
members contributes $m$ entries to $H$ but $m^2$ to $\Delta$.
Table~\ref{tab:app_factored_density} measures the resulting densification
across the suite, at between $2.4\times$ and $4.4\times$ the nonzeros of the
incidence matrix.

The shared Chebyshev recurrence of Eq.~\eqref{eq:cheby_recurrence} never reads
the entries of $\widetilde{\Delta}$; it needs only the products
$\widetilde{\Delta}\,\mathbf{T}_{m-1}$. The operator can therefore be applied
through its factors,
\begin{equation}
\Delta x = x - s\odot\bigl(H\,(q\odot(H^\top(s\odot x)))\bigr),
\label{eq:app_factored_apply}
\end{equation}
with $s=\operatorname{diag}(S)$ and $q=\operatorname{diag}(Q)$ precomputed: two
sparse products against $H$ and four elementwise scalings, with nothing
allocated beyond the incidence factors themselves. This is the same operator
evaluated in a different association order, so the two forms differ only in
floating-point summation order, and both feed the power iteration of
Eq.~\eqref{eq:app_power_iter} unchanged.

\begin{table}[htbp]
\centering
{\small
\setlength{\tabcolsep}{8.5pt}
\setlength{\aboverulesep}{0pt}
\setlength{\belowrulesep}{0pt}
\setlength{\extrarowheight}{0.65ex}
\renewcommand{\arraystretch}{1.05}
\begin{tabular}{@{}lrrrrrrrl@{}}
\toprule
\textbf{Benchmark} & \textbf{Ops} & $n$ & $\mathrm{nnz}(H)$ &
$\mathrm{nnz}(\Delta)$ & \textbf{Density} & \textbf{MACs} &
\textbf{Asm./run} & \textbf{Form} \\
\midrule
\multicolumn{9}{@{}l}{\textit{Shared topology: one operator serves the whole run}} \\
Allen--Cahn, NS$\,{\times}3$ & $1$ & $4{,}096$ & $36{,}864$ & $102{,}400$ & $2.78\times$ & $0.72\times$ & $1.20$\,MiB & assembled \\
Darcy & $1$ & $7{,}225$ & $180{,}625$ & $568{,}139$ & $3.15\times$ & $0.64\times$ & $6.56$\,MiB & assembled \\
Airfoil & $1$ & $3{,}982$ & $96{,}321$ & $230{,}300$ & $2.39\times$ & $0.84\times$ & $2.67$\,MiB & assembled \\
Supernova & $1$ & $32{,}768$ & $294{,}912$ & $1{,}200{,}286$ & \flagcell{$\mathbf{4.07\times}$} & $0.49\times$ & $13.99$\,MiB & assembled \\
\midrule
\multicolumn{9}{@{}l}{\textit{Per-sample topology: one operator per mesh, all cached}} \\
Cylinder & \flagcell{$\mathbf{1{,}140}$} & $1{,}876$ & $44{,}176$ & $103{,}666$ & $2.35\times$ & $0.85\times$ & $1.34$\,GiB & assembled \\
\rowcolor{TableTint}
\textbf{ShapeNet-Car} & $889$ & $32{,}186$ & $766{,}506$ & $3{,}394{,}454$ & $\mathbf{4.43\times}$ & $0.45\times$ & $\mathbf{33.94}$\,\textbf{GiB} & \textbf{factored} \\
\rowcolor{TableTint}
\textbf{Aircraft} & $30$ & $331{,}971$ & $2{,}323{,}785$ & $6{,}349{,}663$ & $2.73\times$ & $0.73\times$ & $\mathbf{2.20}$\,\textbf{GiB} & \textbf{factored} \\
\bottomrule
\end{tabular}
}
\caption{The factored-versus-assembled decision, measured on the operators each
benchmark actually builds. \textbf{Ops} is the number of distinct operators a run
holds resident; \textbf{Density} is $\mathrm{nnz}(\Delta)/\mathrm{nnz}(H)$, how
much fatter assembly makes the operator; \textbf{MACs} is
$2\,\mathrm{nnz}(H)/\mathrm{nnz}(\Delta)$, the multiply--accumulates of one
factored apply relative to one assembled apply, the factor of two being the
second sparse product the factored form pays. MACs counts arithmetic only, and so
excludes the intermediate the factored path writes and the second kernel launch.
Assembly densifies on every benchmark, so the factored form is arithmetically
cheaper on every benchmark, yet neither column tracks the form used. Green marks
the two benchmarks kept factored; gray marks the two cells a reader would expect
to force that form and which do not---Supernova's densification, the second
highest in the suite, and Cylinder's operator count, the largest. What tracks the
form is \textbf{Asm./run}, the operator count times the assembled size of one.
Cylinder's meshes vary in size; its row reports the representative operator and
scales it by the cached count. Sizes are binary.}
\label{tab:app_factored_density}
\end{table}

Read Table~\ref{tab:app_factored_density} from the density column first and the
form column last. Assembly densifies on \emph{every} benchmark, and because the
factored form issues two sparse products where the assembled form issues one, it
is arithmetically cheaper wherever the densification exceeds $2\times$, which is
everywhere in the suite and most strongly on ShapeNet-Car at $0.45\times$ the
multiply--accumulates. The two forms are also numerically interchangeable.
Neither fact selects the form. Supernova is the counter-example that settles it:
it carries the second-highest densification in the suite at $4.07\times$, it is
three-dimensional, and it is still assembled, because a run holds exactly one of
it. What tracks the form is the resident operator set, and that is a product of
two quantities which must both be large before the choice matters.

\paragraph{How many operators a run holds.} The suite splits three ways rather
than two. The structured grids and the airfoil share a single operator across
every sample: it is built once, occupies at most $14$\,MiB, and its build cost
amortizes to nothing over training, so its size is irrelevant. Supernova belongs
to this group---its $32^3$ grid is identical for all $592$ trajectories---which
is why dimensionality is the wrong reading of the criterion. The remaining three
benchmarks carry a per-sample mesh, and all three memoize it: each loader caches
an operator on first use, so after one epoch the entire set is resident and
scales with the dataset.

\paragraph{How large each operator is.} This is what separates Cylinder from the
two 3D geometry benchmarks. Cylinder's meshes hold on the order of $1{,}900$
nodes, giving $1.20$\,MiB per assembled operator, whereas ShapeNet-Car and
Aircraft carry the largest meshes in the suite at $32{,}186$ and $331{,}971$
nodes, giving $39.1$ and $75.2$\,MiB---a gap of $33$ to $63\times$ on a single
operator.

\begin{table}[htbp]
\centering
{\small
\setlength{\tabcolsep}{6pt}
\setlength{\aboverulesep}{0pt}
\setlength{\belowrulesep}{0pt}
\setlength{\extrarowheight}{0.65ex}
\renewcommand{\arraystretch}{1.05}
\begin{tabular}{@{}l r >{\columncolor{TableTint}[\tabcolsep][\tabcolsep]}r >{\columncolor{TableTint}[\tabcolsep][\tabcolsep]}r}
\toprule
& \textbf{Cylinder} & \textbf{ShapeNet-Car} & \textbf{Aircraft} \\
\midrule
Distinct operators & $1{,}140$ & $889$ & $30$ \\
Factored, per operator & $0.69$\,MiB & $11.94$\,MiB & $37.99$\,MiB \\
Assembled, per operator & $1.20$\,MiB & $39.09$\,MiB & $75.20$\,MiB \\
Factored, whole run & $0.77$\,GiB & $\mathbf{10.37}$\,\textbf{GiB} & $\mathbf{1.11}$\,\textbf{GiB} \\
Assembled, whole run & $\mathbf{1.34}$\,\textbf{GiB} & $\mathbf{33.94}$\,\textbf{GiB} & $\mathbf{2.20}$\,\textbf{GiB} \\
\midrule
\textbf{Form used} & assembled & \textbf{factored} & \textbf{factored} \\
\bottomrule
\end{tabular}
}
\caption{Why Cylinder assembles and the two geometry benchmarks do not, with the
factored benchmarks shaded. All three carry a per-sample mesh and memoize one
operator per sample, so all three hold a resident set that scales with the
dataset; they differ by up to $63\times$ in the size of a single operator, and it
is the product that decides. Factored size is the cached blob---integer incidence
pairs and the two scaling vectors; assembled size is the corresponding CSR
operator.}
\label{tab:app_factored_memory}
\end{table}

Table~\ref{tab:app_factored_memory} gives that product. Cylinder has the
per-sample topology but not the size: its $1{,}140$ cached operators total
$1.34$\,GiB, which fits alongside the run, so it is assembled and keeps the
fused SELL path of Appendix~\ref{app:sell_backend} on one operator per apply.
ShapeNet-Car combines per-sample topology with large meshes, and assembling its
$889$ operators would raise the resident set from $10.4$ to $33.9$\,GiB,
alongside the raw geometry and field arrays the loader already holds, which is
the difference between a run that fits on the training machine and one that does
not. Aircraft's ratio is smaller at $1.98\times$, but its individual operators
are the largest in the suite, and assembling them means $30$ sparse triple
products over meshes of $331{,}971$ nodes.

The boundary between Cylinder at $1.34$\,GiB and Aircraft at $2.20$\,GiB is
practical rather than principled. Stated as a rule: factor where the resident
operator set would otherwise threaten the memory budget, and assemble elsewhere,
because assembling is the simpler path and the arithmetic margin alone does not
pay for the extra intermediate. Only the two 3D geometry benchmarks fall on the
first side.

Two secondary effects reinforce the same split. Assembly is a sparse triple
product per sample, so on ShapeNet-Car it is $889$ of them, each producing
$3.4$M nonzeros, whereas the factored build is a scatter-add over incidence
pairs that skips the product entirely; with one shared operator this cost is
paid once and is invisible, and the same ratios apply to the on-disk operator
cache both loaders write between runs. ShapeNet-Car and Aircraft are also the
only two benchmarks trained at batch size $1$, so their per-step footprint is
already set by a single sample's mesh. Gradient checkpointing does not separate
them, being enabled on eight of the ten configurations.

The factored form is not free. Each apply writes and re-reads an
$|\mathcal{E}|\times K$ intermediate between the two products, issues two kernel
launches instead of one, and keeps both $H$ and $H^\top$ in CSR---registering
the transposed factor for the backward pass is what holds this at two copies
rather than four---so on device it saves only $1.3\times$ against the assembled
operator on a car mesh, against $3.3\times$ in host memory and on disk. The
saving is in how many operators are held, not in the size of any one of them on
the GPU.

The two choices compose with the sparse layout rather than competing with it.
The factored apply dispatches each incidence product through the same backend
selection an assembled operator receives, so the SELL kernel is applied to each
factor in turn; this is what the ShapeNet-Car and Aircraft rows of
Table~\ref{tab:app_sell_epochs} measure, at $1.22\times$ and $1.16\times$ with
the operator never assembled. The two benchmarks fall on opposite sides of the
$K\!\geq\!192$ crossover of Appendix~\ref{app:sell_backend}: ShapeNet-Car trains
at $K=Bd_h=192$ and reaches SELL under automatic dispatch, whereas Aircraft
trains at $K=128$ and requires the backend to be requested explicitly. The two
backends give identical Aircraft training curves to the printed digits.

Finally, the two forms are the same operator. The factored build applies the
Gaussian hyperedge weighting of Algorithm~\ref{alg:laplacian} unchanged---per
hyperedge, $\sigma_e$ is the mean member distance to the center and $w_e$ the
mean incidence weight---so the only numerical difference is fp32 summation order
across two products rather than one, and both forms feed the same power
iteration and satisfy $\lambda_{\max}\leq 1$. Routing the factored apply through
the shared sparse backend leaves it bit-identical on the CSR path, agreeing with
the direct formulation to $0$ in both forward and backward on a $500$-node
operator; against the SELL kernel the agreement is fp32 roundoff, at relative
$1.6\times10^{-8}$ forward and $2.3\times10^{-8}$ backward on a $32{,}000$-node
operator at $K=192$.
\subsection{Learned Wavelet Scales}

With the number of scales $J$ fixed, we now look at where training moves them and
how to read the filter bank they form (Figure~\ref{fig:scale_bank}). Every colored
curve in that figure is a single scale acting as a band-pass filter,
$g(s_j\lambda)^2$, that responds only to a limited range of graph frequencies
$\lambda$. The dyadic initialization $s_j^{\mathrm{init}}=x^{*}2^{j}/\lambda_{\max}$
spaces these bands one octave apart, placing the peak of band $j$ at
$\lambda_j^{*}=x^{*}/s_j=\lambda_{\max}/2^{j}$, where
$x^{*}=2-1/\sqrt{3}\approx1.423$ is the frequency at which the Meyer-like kernel
$g$ is largest (Eq.~\eqref{eq:meyer_kernel}). Scales are indexed
$j=0,\dots,J{-}1$ throughout, as in Definition~\ref{def:frame} and
Table~\ref{tab:learned_scales}. The finest scale $s_0$ therefore sits at the top
of the spectrum and picks up the sharpest detail, while the coarsest scale
$s_{J-1}$ sits near the bottom and picks up the smooth, slowly varying part of
the field.

The vertical axis deserves a word, since it is easy to misread. A scale enters the
kernel only through the product $s_j\lambda$, so retuning $s_j$ slides a band left
or right along the spectrum but never changes how tall it is: whatever the scale,
each band rises to the same peak value $g(x^{*})^2\approx1.918$. So that the
figure speaks to \emph{where} the bands sit rather than how tall they are, we
divide every band by this common peak, and each one then tops out at $1$. The
heights in Figure~\ref{fig:scale_bank} are thus normalized, and the comparison
across panels is one of position, not amplitude.

The bold black curve is what these bands add up to: the total spectral coverage
$G_{\mathbf{s}}(\lambda)=\sum_j g(s_j\lambda)^2$ (Definition~\ref{def:frame}), which
records how much filtering weight the bank places at each frequency. Two features
of it carry the message. First, it runs \emph{above} the individual bands: the
Meyer-like kernel never quite reaches zero (in the argument $\xi=s_j\lambda$ it
tapers as $\xi^2$ near the origin and as $4/\xi^2$ for large $\xi$;
Eq.~\eqref{eq:meyer_kernel}), so neighboring bands always
overlap, and at any peak the coverage adds the on-band response to the tails
leaking in from either side. Second, it is nearly, but not exactly, flat: a finite
ladder of equally spaced bands tiles the spectrum only approximately and leaves a
gentle ripple of about one hump per octave. That ripple is precisely the frame
defect $\mathcal{D}_{\mathrm{TF}}=\mathrm{Var}_{\lambda}[G_{\mathbf{s}}]$
(Eq.~\eqref{eq:tf_defect}), the quantity the tight-frame regularizer
$\mathcal{L}_{\mathrm{TF}}$ (Eq.~\eqref{eq:tight_frame_loss}) drives toward zero so
that the bank passes every frequency with almost equal gain and conserves signal
energy (Proposition~\ref{prop:frame}). Panel~(a) of Figure~\ref{fig:scale_bank}
shows this near-flat coverage at initialization.

\begin{figure}[htbp]
\centering
\begin{figpanel}
\centering
\begin{minipage}[t]{0.39\textwidth}
\centering
\includegraphics[width=\linewidth]{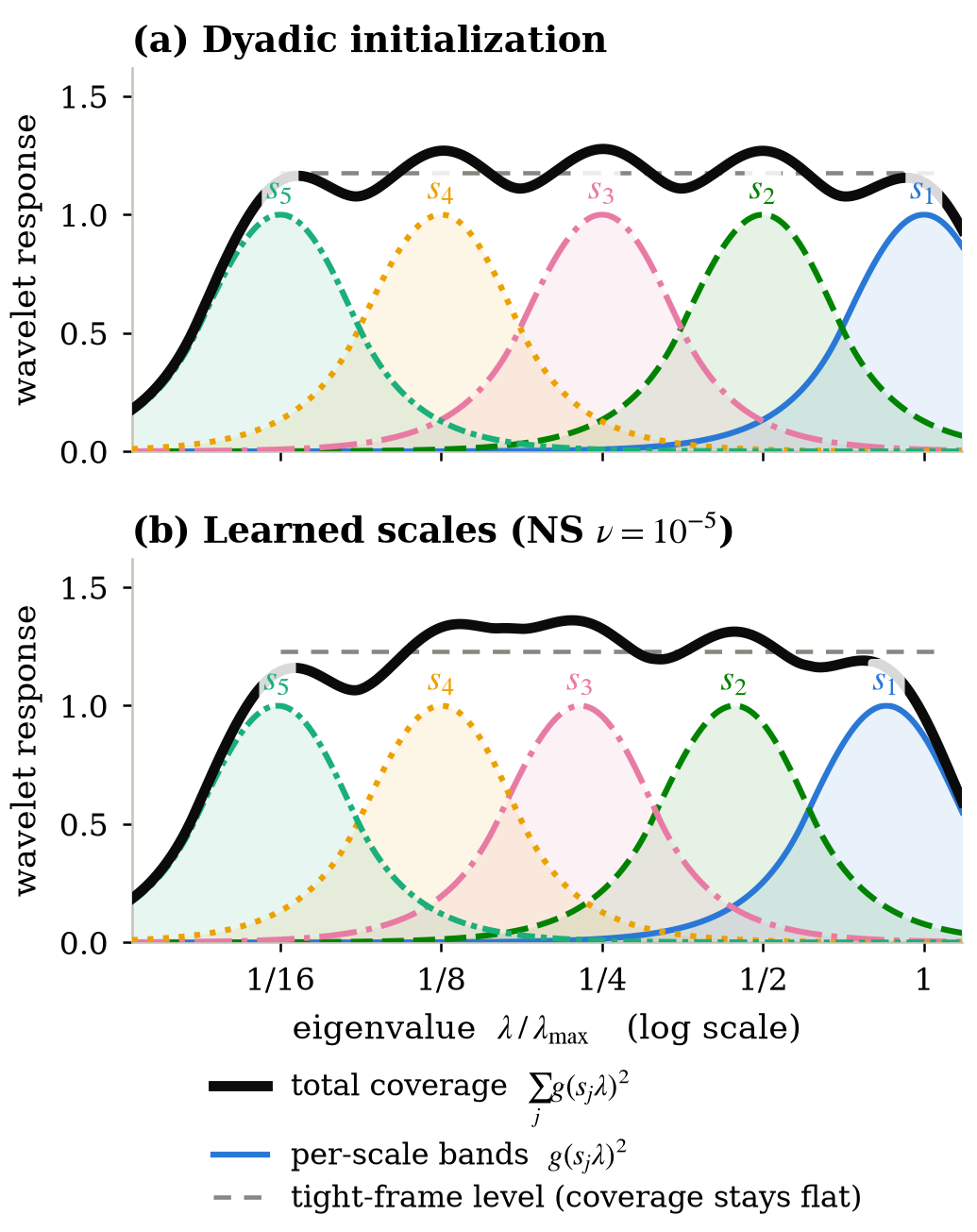}
\caption{Wavelet filter bank on the normalized spectrum (log axis; bands
normalized to unit peak): per-scale bands $g(s_j\lambda)^2$ (colored), their total
coverage $G_{\mathbf{s}}=\sum_j g(s_j\lambda)^2$ (black), and the tight-frame
level (dashed). (a) Dyadic initialization; (b) learned scales on Navier--Stokes
$\nu{=}10^{-5}$.}
\label{fig:scale_bank}
\end{minipage}\hfill
\begin{minipage}[t]{0.59\textwidth}
\centering
\includegraphics[width=\linewidth]{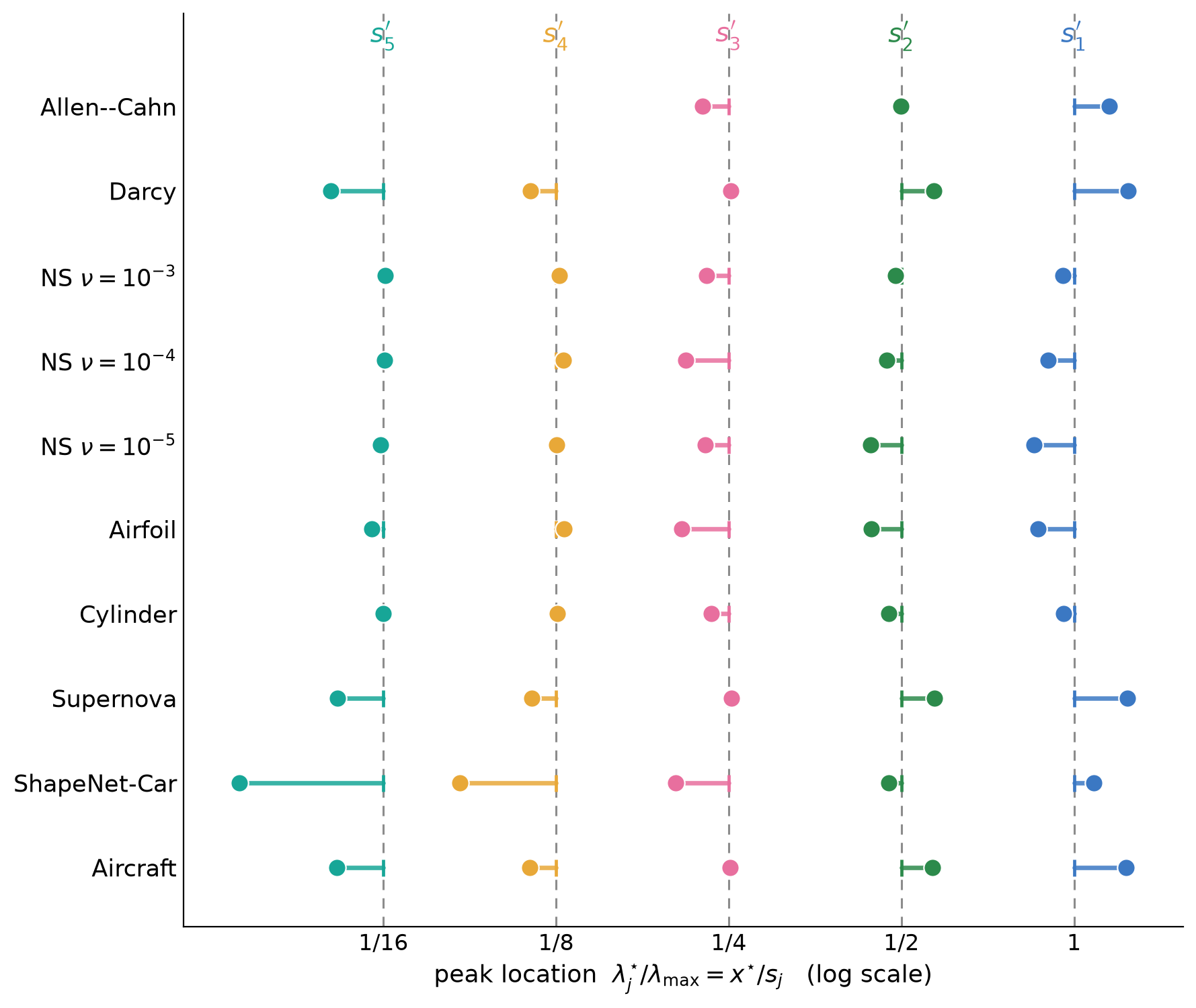}
\caption{Peak location $\lambda_j^{*}/\lambda_{\max}=x^{*}/s_j$ of each
learned scale per benchmark (log axis; colors index $s_0$--$s_4$): each dot is a
learned peak, joined by a segment to its dyadic initialization
$\lambda_{\max}/2^{j}$ (dashed vertical lines). Dots right of the dashed line mark
sharpened (higher-frequency) scales and dots left of it stretched
(lower-frequency) ones. Allen--Cahn uses only three scales.}
\label{fig:scale_shift}
\end{minipage}
\end{figpanel}
\end{figure}

\subsubsection*{Why This Many Scales}

We use $J{=}5$ for every benchmark except Allen--Cahn, which uses $J{=}3$. The
count fixes how many octaves the ladder tiles: $J{=}5$ spans four octaves, from
$\lambda_{\max}$ down to $\approx\lambda_{\max}/16$, enough to separate the sharp,
high-$\lambda$ structures (shocks, wakes, interfaces, and multi-field couplings)
from the smooth low-$\lambda$ bulk in the turbulent and multi-physics problems.
Allen--Cahn's phase field is governed by a single diffusion constant and a few
interface scales, so three bands already cover its informative spectrum; adding
lower-frequency bands only increased the parameter count without improving
accuracy. Since the shared Chebyshev basis makes the per-layer cost independent
of $J$ (Remark~\ref{rem:cost}), $J$ trades spectral resolution against
parameters, not against runtime.

\begin{table}[htbp]
\centering
{\small
\begin{tabular}{@{}lcccccc@{}}
\toprule
\textbf{Dataset} & $\lambda_{\max}$ & $s_0$ & $s_1$ & $s_2$ & $s_3$ & $s_4$ \\
\midrule
Dyadic initialization ($\lambda_{\max}{=}1$) & -- & 1.423 & 2.845 & 5.691 & 11.381 & 22.762 \\
\midrule
Allen--Cahn ($J{=}3$) & 0.997 & 1.240 & 2.858 & 6.337 & -- & -- \\
Darcy & 0.998 & 1.149 & 2.504 & 5.661 & 12.645 & 28.198 \\
Navier--Stokes $\nu{=}10^{-3}$ & 0.997 & 1.492 & 2.926 & 6.241 & 11.257 & 22.636 \\
Navier--Stokes $\nu{=}10^{-4}$ & 0.997 & 1.584 & 3.023 & 6.788 & 11.092 & 22.713 \\
Navier--Stokes $\nu{=}10^{-5}$ & 0.997 & 1.677 & 3.225 & 6.268 & 11.397 & 23.090 \\
Airfoil & 0.987 & 1.664 & 3.251 & 6.972 & 11.160 & 24.139 \\
Cylinder & 0.989 & 1.500 & 3.029 & 6.175 & 11.452 & 23.029 \\
Supernova & 0.996 & 1.154 & 2.504 & 5.645 & 12.600 & 27.475 \\
ShapeNet-Car & 0.998 & 1.318 & 3.001 & 7.062 & 16.761 & 40.677 \\
Aircraft & 0.998 & 1.156 & 2.515 & 5.673 & 12.675 & 27.463 \\
\bottomrule
\end{tabular}
}
\caption{Final learned wavelet scales of the best HALO model on each benchmark,
averaged over the $L$ Adaptive Wavelet Blocks (per-block values vary by at most a
few percent).}
\label{tab:learned_scales}
\end{table}

Table~\ref{tab:learned_scales} lists the learned scales and
Figure~\ref{fig:scale_shift} the peak location of each band relative to its dyadic
start. Although initialized identically across layers, the scales leave the grid
in a dataset-dependent way: the Navier--Stokes and Airfoil models enlarge the
fine and middle scales, pulling their peaks toward the mid-spectrum
(Figures~\ref{fig:scale_shift} and~\ref{fig:scale_bank}b), and the shift grows as the flow becomes more
chaotic (mean $s_0$ rises from $1.49$ at $\nu{=}10^{-3}$ to $1.68$ at
$\nu{=}10^{-5}$). Darcy, Supernova, ShapeNet-Car, and Aircraft move the opposite way
(Figure~\ref{fig:scale_shift}), shrinking $s_0$ to keep more high-$\lambda$ response
while stretching the coarsest scale toward lower frequencies (ShapeNet-Car's $s_4$
by $78\%$ and Aircraft's by $21\%$). On the large $\sim$330k-point Aircraft mesh this outward spread is the most
symmetric: the two finest scales sharpen while the two coarsest stretch,
widening coverage toward both the fine surface features and the smooth far-field. Across all of
them the tight-frame penalty holds the coverage flat despite the drift
(Figure~\ref{fig:scale_bank}b), so the scales specialize to each operator without
giving up the frame property.

\subsection{Zero-Shot Evaluation Details}

Table~\ref{tab:zeroshot_structured_full} gives the per-baseline comparison
behind the compact transfer ledger in Figure~\ref{fig:superresolution_collage}.
All models are trained
on the listed base grid and evaluated on the target grid without fine-tuning.

\begin{table}[htbp]
\centering
{\small
\setlength{\tabcolsep}{1mm}
\begin{tabular}{@{}llcccc@{}}
\toprule
\multicolumn{6}{c}{\textsc{Structured-Grid Zero-Shot Super-Resolution}} \\
\midrule
\textbf{Dataset} & \textbf{Grid} & \textbf{FNO} & \textbf{WNO} & \textbf{GNO} & \textbf{HALO} \\
\midrule
\multirow{2}{*}{Allen--Cahn}
 & Base\, $64^2$  & 0.0239 & 0.0643 & 0.0215 & \textbf{0.0046} \\
 & ZS\, $128^2$   & 0.0244 & 0.2801 & 0.1870 & \textbf{0.0115} \\
\midrule
\multirow{2}{*}{Darcy}
 & Base\, $85^2$  & 0.0108 & 0.0084 & 0.0346 & \textbf{0.0041} \\
 & ZS\, $421^2$   & 0.0121 & 0.0109 & 0.0469 & \textbf{0.0082} \\
\midrule
\multirow{2}{*}{NS $\nu{=}10^{-4}$}
 & Base\, $64^2$  & \textbf{0.0408} & 0.1275 & 0.4046 & 0.0442 \\
 & ZS\, $256^2$   & \textbf{0.0481} & 0.4799 & 0.5590 & 0.0556 \\
\midrule
\multirow{2}{*}{Supernova}
 & Base\, $32^3$  & 0.5576 & 0.5531 & 0.6524 & \textbf{0.4004} \\
 & ZS\, $64^3$    & 0.6746 & 0.7292 & 0.8297 & \textbf{0.5849} \\
\bottomrule
\end{tabular}
}
\caption{Mean relative $L_2$ error; lower is better. Bold marks the best result
in each row; ZS denotes evaluation on the zero-shot target grid.}
\label{tab:zeroshot_structured_full}
\end{table}

HALO has the lowest target-grid error on three of the four structured
benchmarks; FNO remains ahead on Navier--Stokes. Its advantage is primarily
absolute accuracy rather than uniformly graceful degradation: HALO transfers
from a substantially lower base error on Allen--Cahn, Darcy, and Supernova. On
the unstructured domains, training uses $25\%$ of the cells and evaluation uses
the complete mesh ($4.0\times$ as many cells), with error increasing from
$0.0113$ to $0.0332$ on Airfoil and from $0.0664$ to $0.0902$ on Cylinder.


\section{Held-Out Qualitative Predictions}
\label{sec:qualitative_predictions}

Figures~\ref{fig:app_structured_pair}--\ref{fig:app_aircraft_pred} are held-out
diagnostics, not aggregate results; the test-set means remain those in
Table~\ref{tab:benchmark} and Table~\ref{tab:large_geom_results}. Throughout,
reference and prediction share a color scale and the remaining panel of each
group is the pointwise absolute error, so no per-figure caption repeats the
convention.

\subsection{Structured-Grid Benchmarks}
Figure~\ref{fig:app_structured_pair} places the Allen--Cahn and Darcy examples
side by side. Within each dataset panel, the two columns are separate held-out
inputs and the rows are input field, reference, prediction, and absolute error.
The examples expose both interface preservation in Allen--Cahn and the localized
Darcy residuals that track the permeability geometry. Figure~\ref{fig:app_ns_all}
shows $\nu=10^{-3}$, $10^{-4}$, and $10^{-5}$ from left to right. Within each
viscosity group the columns are forecasts at $t=10$ and $13\,\mathrm{s}$, with
row labels shown only in the leftmost group. The lower-viscosity predictions
make the late-rollout difficulty visible as the dynamics grow more chaotic.

\begin{figure}[htbp]
\centering
\begin{minipage}[t]{0.485\linewidth}
\centering
\colorbox{BurntOrange}{\parbox{\dimexpr\linewidth-2\fboxsep\relax}{\centering\color{white}\bfseries Allen--Cahn}}
\vspace{2pt}
\includegraphics[width=\linewidth,height=0.66\textheight,keepaspectratio]{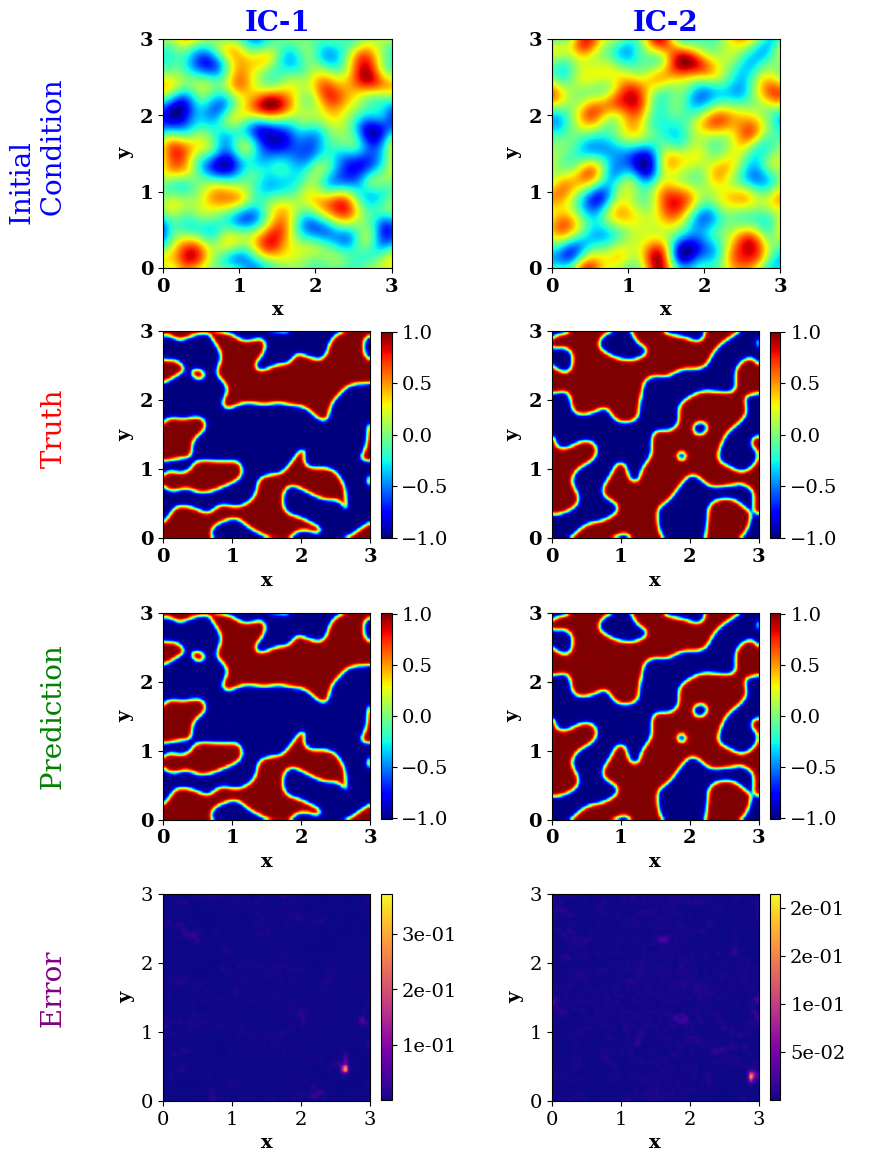}
\end{minipage}
\hfill
\begin{minipage}[t]{0.485\linewidth}
\centering
\colorbox{RoyalPurple}{\parbox{\dimexpr\linewidth-2\fboxsep\relax}{\centering\color{white}\bfseries Darcy}}
\vspace{2pt}
\includegraphics[width=\linewidth,height=0.66\textheight,keepaspectratio]{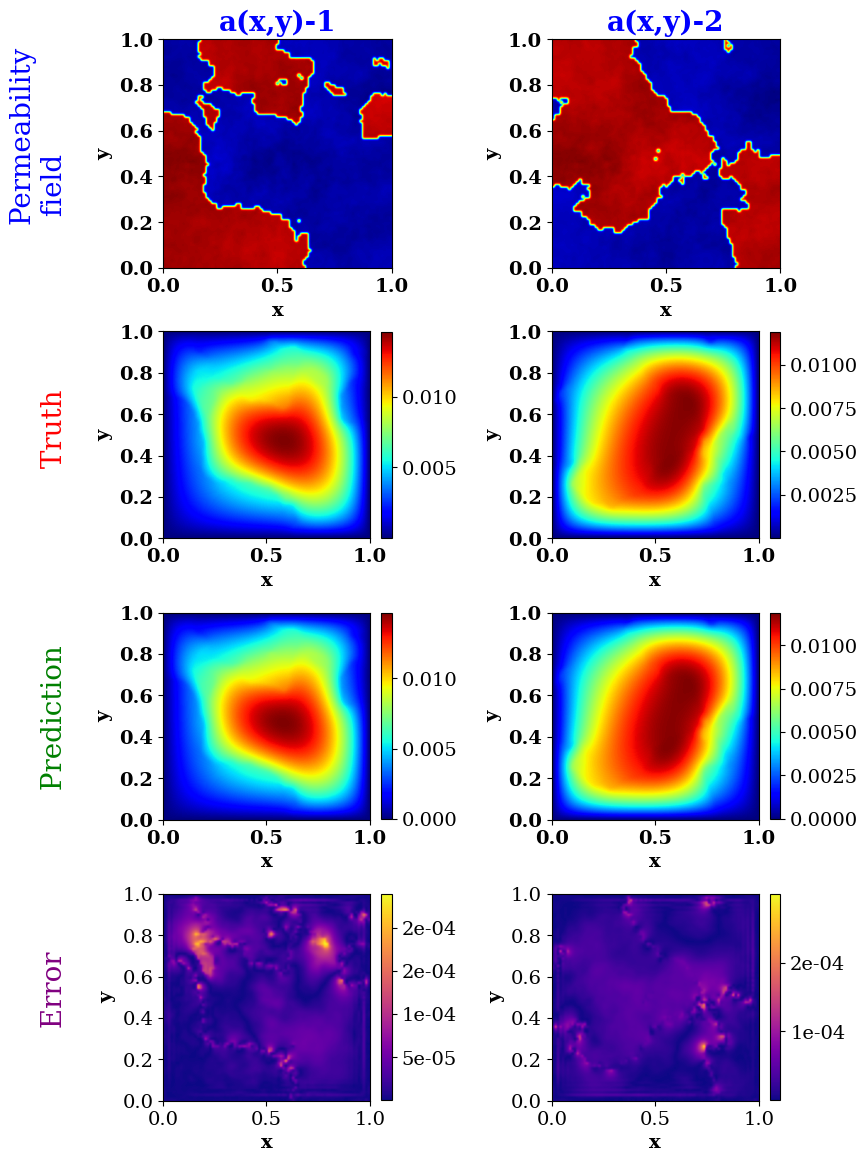}
\end{minipage}
\caption{Representative held-out Allen--Cahn and Darcy predictions. Each
dataset contains two unseen inputs; rows show the input, ground truth, HALO
prediction, and pointwise absolute error.}
\label{fig:app_structured_pair}
\end{figure}
\newpage

\begin{figure}[htbp]
\centering
\begin{minipage}[t]{0.352\linewidth}
\vspace{0pt}\centering
\colorbox{Cerulean}{\parbox{\dimexpr\linewidth-2\fboxsep\relax}{\centering\color{white}\bfseries $\nu=10^{-3}$}}
\vspace{2pt}
\includegraphics[width=\linewidth]{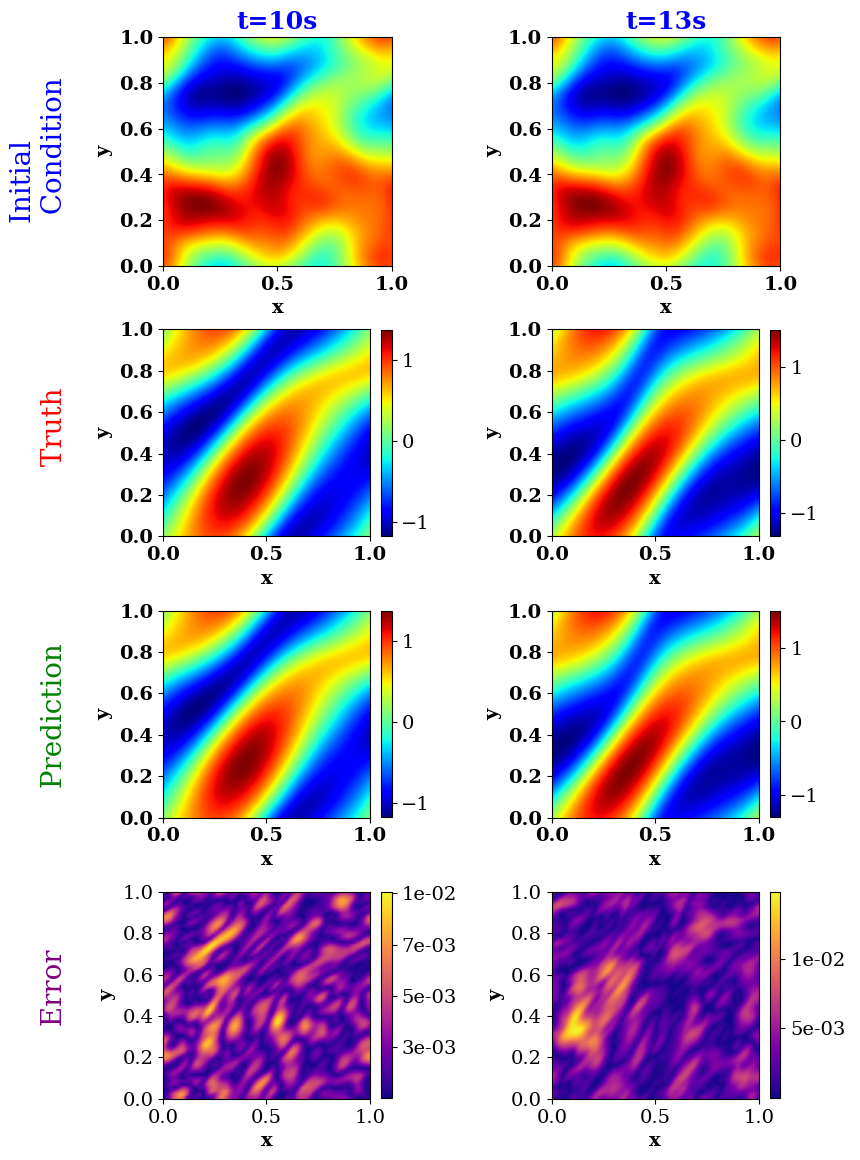}
\end{minipage}
\hfill
\begin{minipage}[t]{0.310\linewidth}
\vspace{0pt}\centering
\colorbox{Cerulean}{\parbox{\dimexpr\linewidth-2\fboxsep\relax}{\centering\color{white}\bfseries $\nu=10^{-4}$}}
\vspace{2pt}
\includegraphics[width=\linewidth]{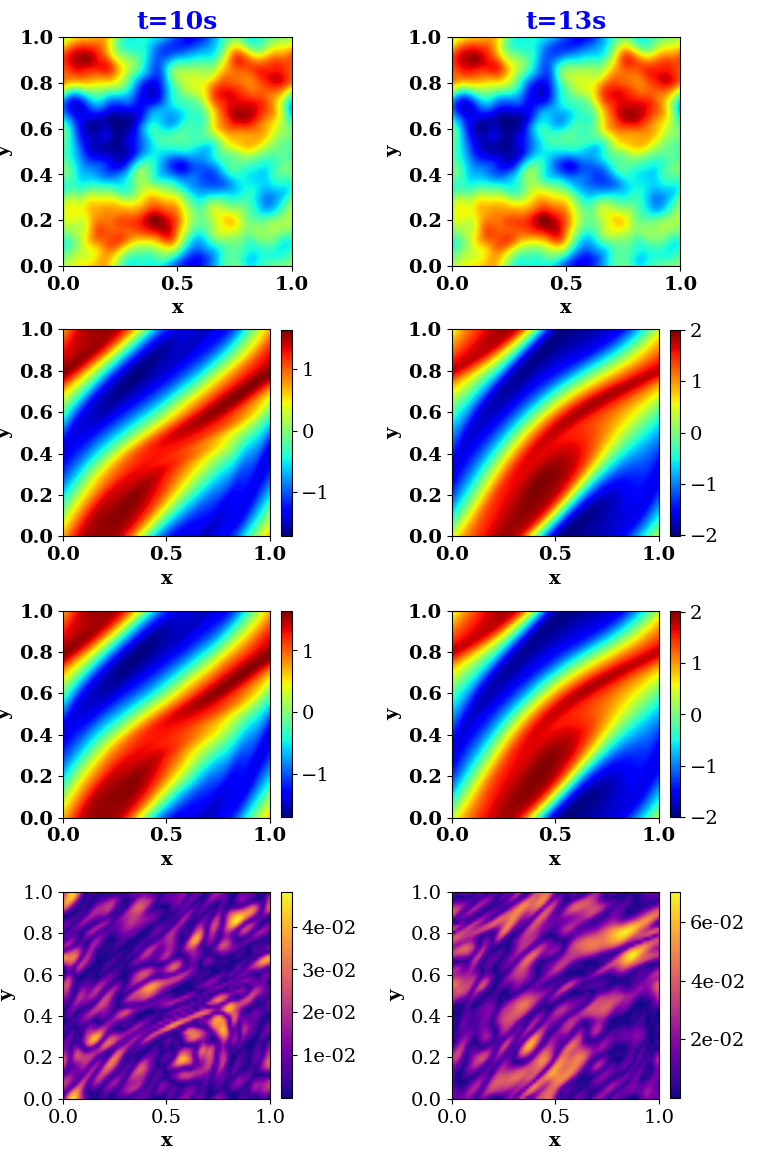}
\end{minipage}
\hfill
\begin{minipage}[t]{0.310\linewidth}
\vspace{0pt}\centering
\colorbox{Cerulean}{\parbox{\dimexpr\linewidth-2\fboxsep\relax}{\centering\color{white}\bfseries $\nu=10^{-5}$}}
\vspace{2pt}
\includegraphics[width=\linewidth]{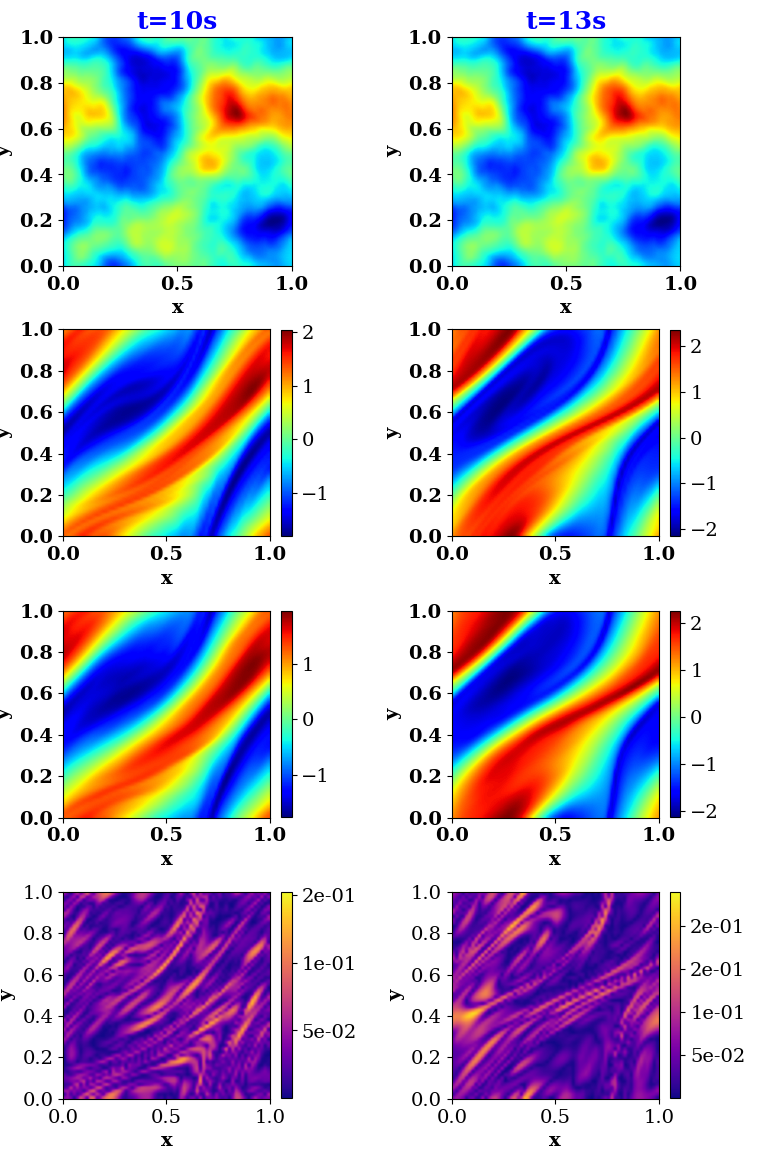}
\end{minipage}
\caption{Held-out Navier--Stokes autoregressive rollouts at three viscosities.
For each viscosity, columns show the $t=10$ and $13\,\mathrm{s}$ forecasts and
rows show the input, ground truth, HALO prediction, and absolute error.}
\label{fig:app_ns_all}
\end{figure}

\subsection{Unstructured-Mesh Flow Benchmarks}
\begin{figure}[H]
\centering
\includegraphics[width=\textwidth,height=0.33\textheight,keepaspectratio]{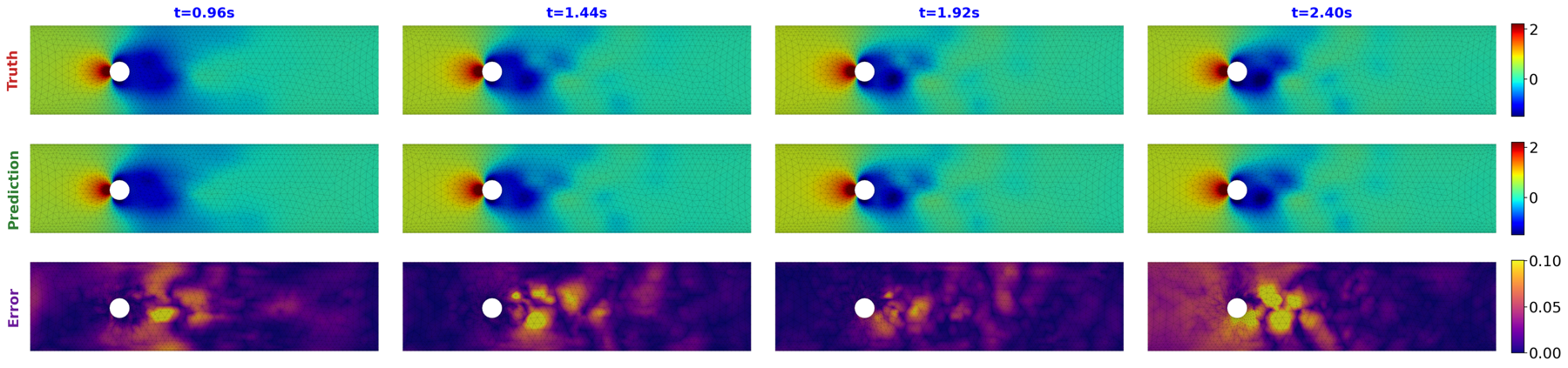}
\par\textbf{(a)} Pressure\par
\includegraphics[width=\textwidth,height=0.33\textheight,keepaspectratio]{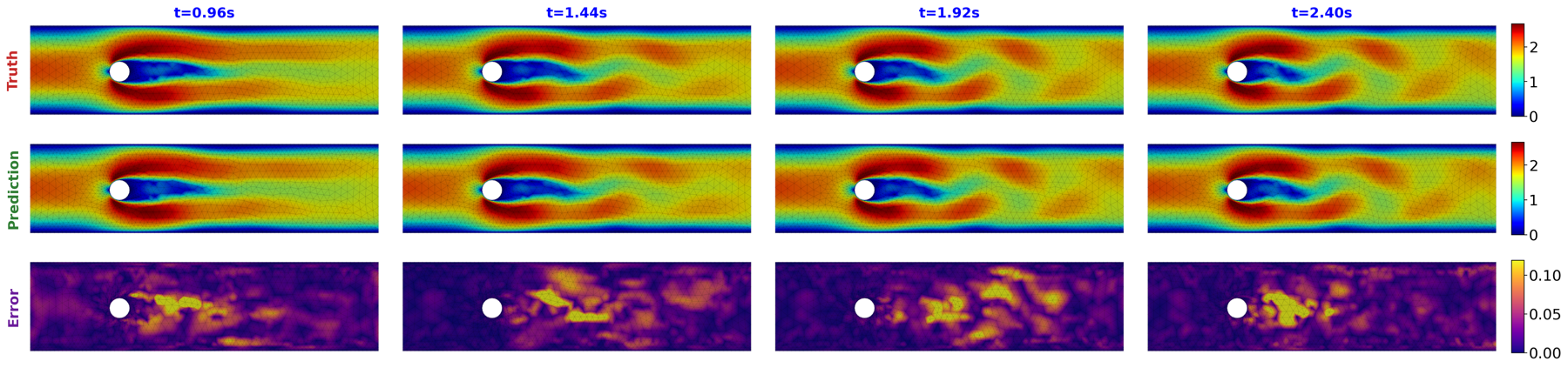}
\par\textbf{(b)} Velocity magnitude
\caption{Held-out Cylinder rollout on the variable native mesh.}
\label{fig:app_cylinder_pressure_speed}
\end{figure}

Figures~\ref{fig:app_cylinder_pressure_speed} and \ref{fig:app_airfoil} show
rollouts directly on the native unstructured discretizations, with the mesh
overlaid to make the spatial support of the forecast explicit. The
reference/prediction pairs track the body-adjacent and downstream flow, and the
error maps locate the harder wake structures behind the airfoil and cylinder.
Cylinder is the variable-mesh case, so its panels also show that the forecast is
produced on that trajectory's own node set rather than a shared one.

\begin{figure}[H]
\centering
\includegraphics[width=\textwidth,height=0.33\textheight,keepaspectratio]{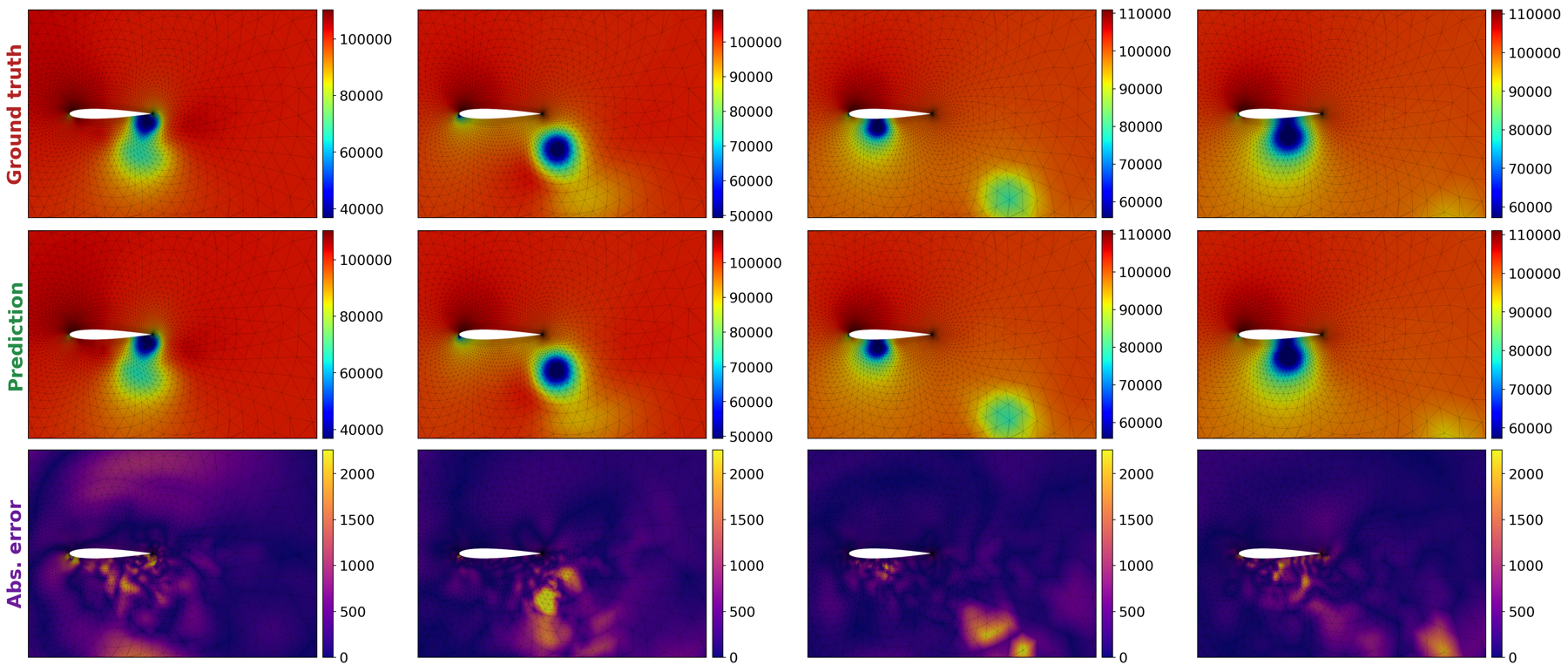}
\par\textbf{(a)} Pressure\par
\includegraphics[width=\textwidth,height=0.33\textheight,keepaspectratio]{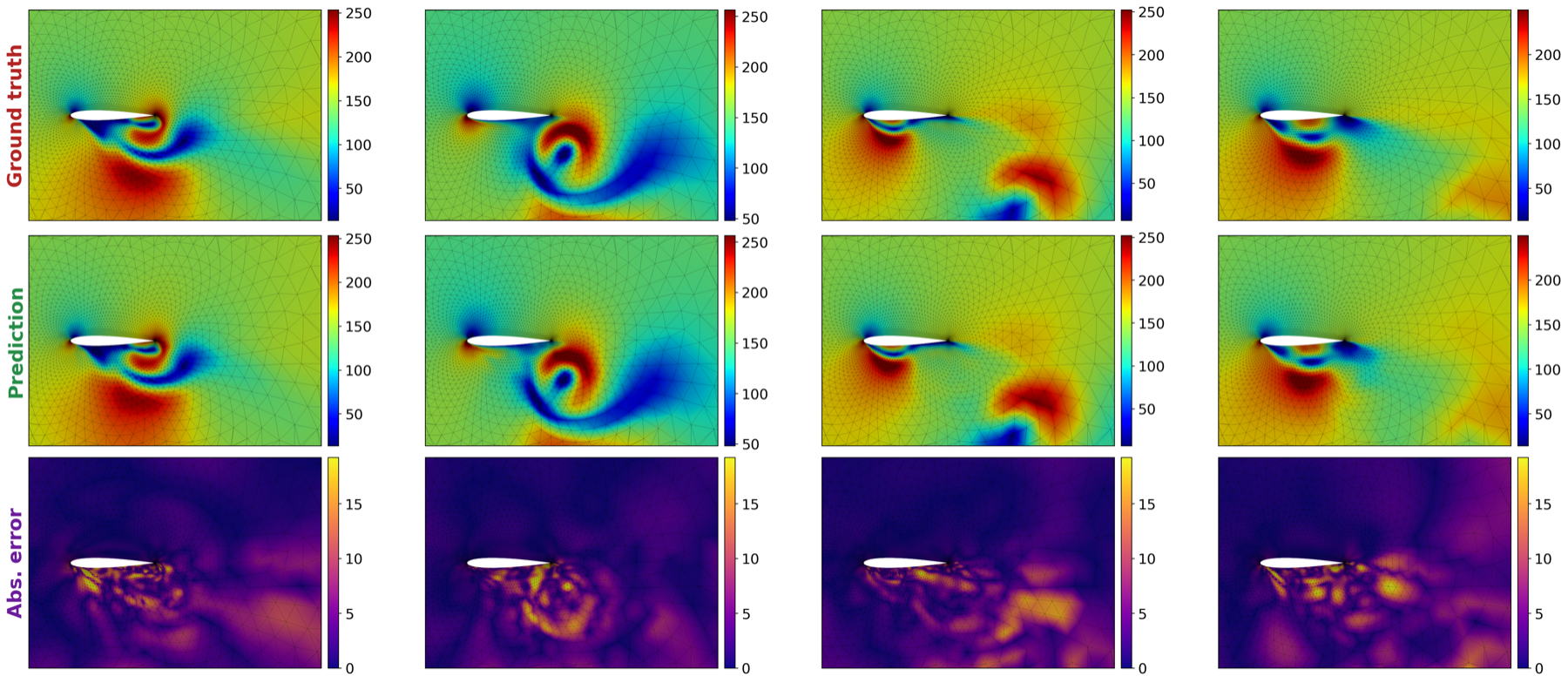}
\par\textbf{(b)} Velocity magnitude
\caption{Held-out Airfoil rollout for the same test trajectory.}
\label{fig:app_airfoil}
\end{figure}

\vspace{-0.5cm}
\subsection{3D Unstructured-Mesh Benchmarks}
Figure~\ref{fig:app_shapenet_pred} shows held-out ShapeNet-Car predictions on the
native 3D unstructured mesh; being steady-state, each panel is a single car
rather than a rollout. The surface-pressure panel (a) resolves the high-pressure
stagnation region at the front of the vehicle and the low-pressure zones over the
hood and cabin, while the volume-velocity panel (b) recovers the surrounding
streamlines and the momentum deficit in the wake. In both, the error maps stay
small and largely structureless, which is what one expects if the input-specific
hypergraph transfers cleanly to an unseen car mesh.

Figure~\ref{fig:app_aircraft_pred} gives the corresponding surface-pressure
comparisons on two held-out Aircraft meshes, each roughly ten times larger than a
ShapeNet-Car mesh. For both airframes, HALO recovers the smooth pressure
distribution over the full surface, while the absolute-error maps localize the
remaining discrepancies along sharp geometric features and surface boundaries,
the same error structure seen in Figure~\ref{fig:large_geom}.

The complete quantitative comparison and its metric-by-metric discussion are
reported with Table~\ref{tab:large_geom_results} in the main Experiments section;
metric definitions and dataset construction remain in Section~\ref{sec:data}.

\begin{figure}[H]
\centering
\includegraphics[width=\textwidth,height=0.30\textheight,keepaspectratio]{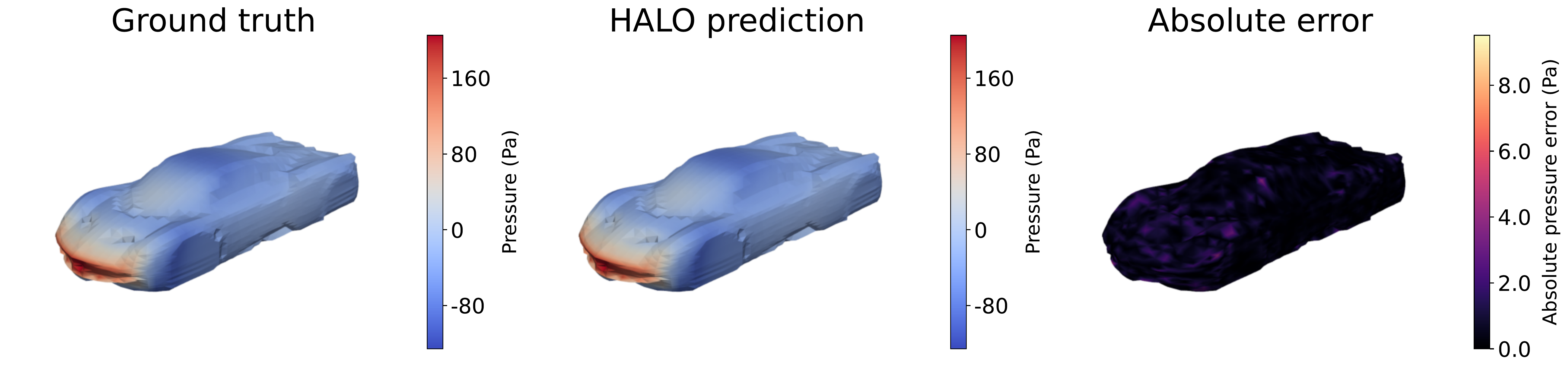}\par
\textbf{(a)} Surface pressure\par
\par
\includegraphics[width=\textwidth,height=0.30\textheight,keepaspectratio]{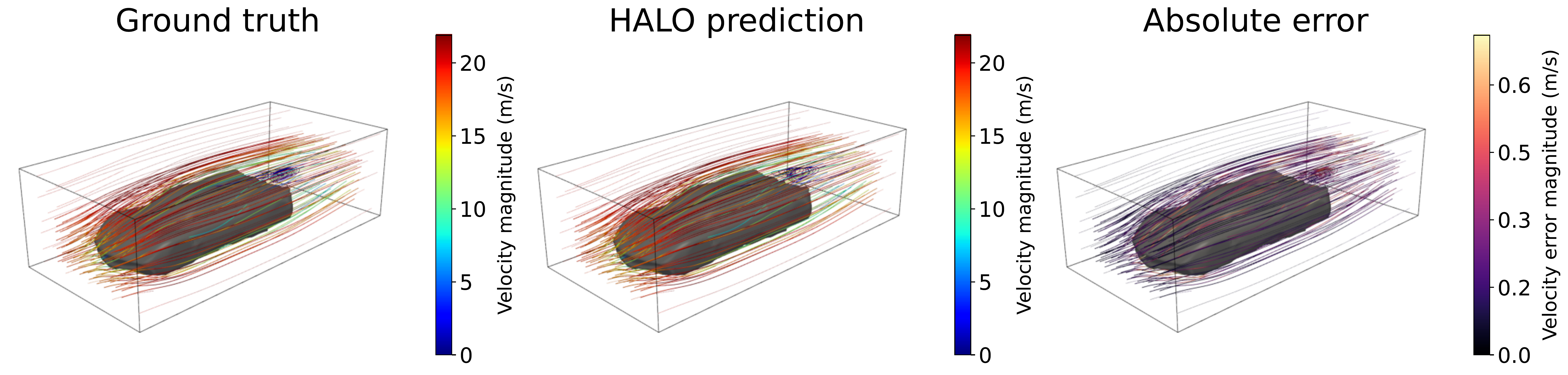}\par
\textbf{(b)} Volume velocity\par
\caption{Held-out ShapeNet-Car predictions on the native 3D unstructured mesh.
\textbf{(a)} surface pressure. \textbf{(b)} volume velocity magnitude with
streamlines. Columns show the ground-truth field, the HALO prediction, and the
pointwise absolute error.}
\label{fig:app_shapenet_pred}
\end{figure}

\begin{figure}[H]
\centering
\includegraphics[width=\textwidth,height=0.30\textheight,keepaspectratio]{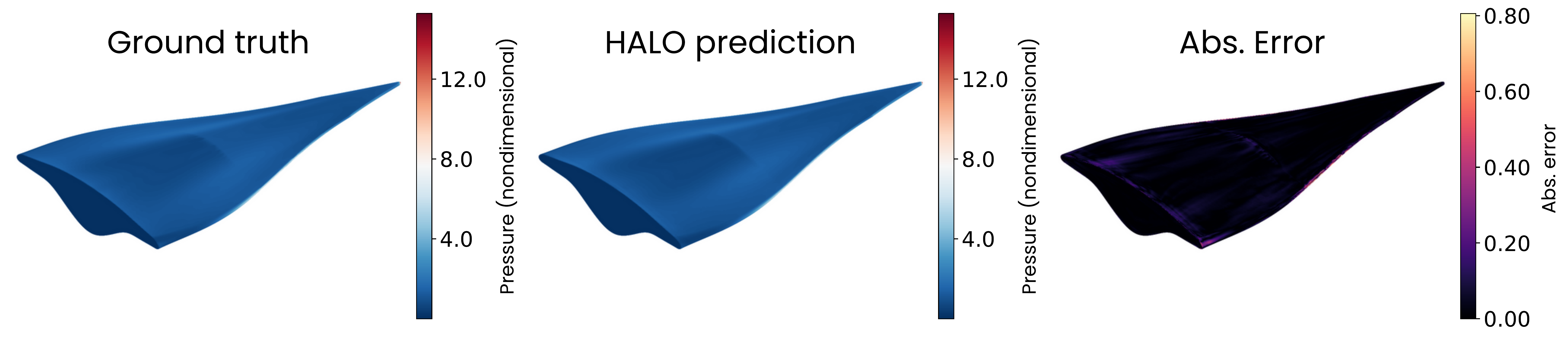}
\par\textbf{(a)} Aircraft 92\par
\includegraphics[width=\textwidth,height=0.30\textheight,keepaspectratio]{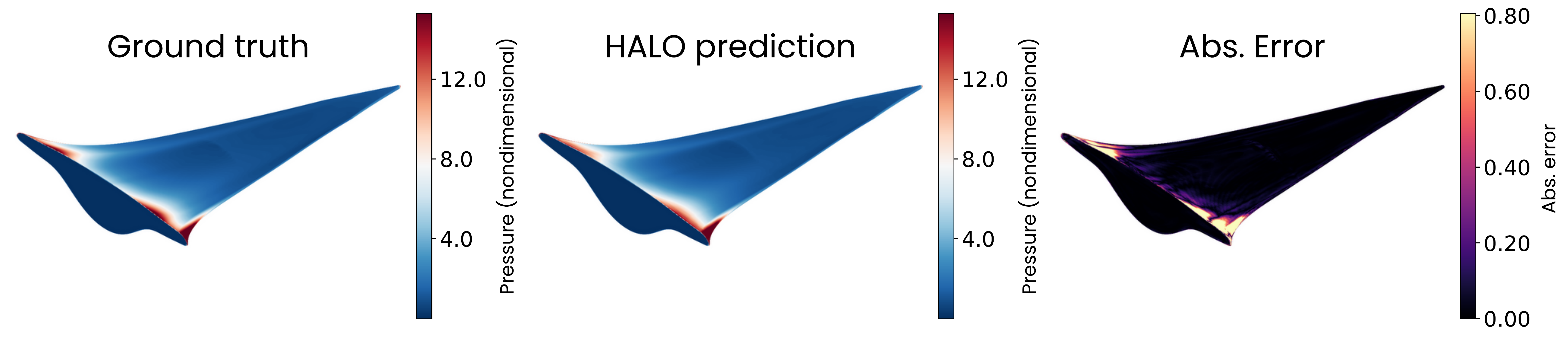}
\par\textbf{(b)} Aircraft 93
\caption{Held-out Aircraft surface-pressure predictions on native 3D
unstructured meshes at Mach $7.0$, angle of attack $0^\circ$, and sideslip
angle $2^\circ$: (a) aircraft 92, with $4.42\%$ relative-$L_2$ error; and
(b) aircraft 93, with $5.83\%$ relative-$L_2$ error. Within each panel, columns
show the ground truth, HALO prediction, and pointwise absolute error.}
\label{fig:app_aircraft_pred}
\end{figure}

\end{document}